\documentclass{article}
\usepackage{ijcai24}

\usepackage{times}
\usepackage{soul}
\usepackage{url}
\usepackage[hidelinks]{hyperref}
\usepackage[utf8]{inputenc}
\usepackage[small]{caption}
\usepackage{graphicx}
\usepackage{amsmath}
\usepackage{amsthm}
\usepackage{booktabs}
\usepackage{algorithm}
\usepackage{algorithmic}
\usepackage[switch]{lineno}

\usepackage{threeparttable}
\usepackage{subfigure}
\usepackage{amssymb}
\usepackage{color}
\usepackage{multirow}

\newcommand{\ALG}{PrivSGP-VR}

\newtheorem{Thm}{Theorem}
\newtheorem{Def}{Definition}
\newtheorem{Lem}{Lemma}
\newtheorem{Ass}{Assumption}
\newtheorem{Rem}{Remark}
\newtheorem{Cor}{Corollary}

\title{{\ALG}: Differentially Private Variance-Reduced Stochastic Gradient Push with Tight Utility Bounds}

\author{
Zehan Zhu$^1$
\and
Yan Huang$^1$\and
Xin Wang$^2$\And
Jinming Xu$^{1}$\thanks{Corresponding author.}
\affiliations
$^1$Zhejiang University, Hangzhou, China\\
$^2$Qilu University of Technology, Jinan, China
\emails
\{12032045, huangyan5616\}@zju.edu.cn,
xinwang@qlu.edu.cn,
jimmyxu@zju.edu.cn
}

\begin{document}
\maketitle

\begin{abstract}
In this paper, we propose a differentially private decentralized learning method (termed PrivSGP-VR) which employs stochastic gradient push with variance reduction and guarantees $(\epsilon, \delta)$-differential privacy (DP) for each node. Our theoretical analysis shows that, under DP Gaussian noise with constant variance, PrivSGP-VR achieves a sub-linear convergence rate of $\mathcal{O}(1/\sqrt{nK})$, where $n$ and $K$ are the number of nodes and iterations, respectively, which is independent of stochastic gradient variance, and achieves a linear speedup with respect to $n$. Leveraging the moments accountant method, we further derive an optimal $K$ to maximize the model utility under certain privacy budget in decentralized settings. With this optimized $K$, PrivSGP-VR achieves a tight utility bound of $\mathcal{O}\left( \sqrt{d\log \left( \frac{1}{\delta} \right)}/(\sqrt{n}J\epsilon) \right)$, where $J$ and $d$ are the number of local samples and the dimension of decision variable, respectively, which matches that of the server-client distributed counterparts, and exhibits an extra factor of $1/\sqrt{n}$ improvement compared to that of the existing decentralized counterparts, such as A(DP)$^2$SGD. Extensive experiments corroborate our theoretical findings, especially in terms of the maximized utility with optimized $K$, in fully decentralized settings.
\end{abstract}

\section{Introduction}
Distributed learning has been widely adopted in various application domains due to its great potential in improving computing efficiency~\cite{langer2020distributed}.
In particular, we assume that each computing node has $J$ data samples, and we use $f_i(x; j)$ to denote the loss of the $j$-th data sample at node $i$ with respect to the model parameter $x\in \mathbb{R}^d$. We are then interested in solving the following non-convex finite-sum optimization problem via a group of $n$ nodes:
\begin{equation}
\label{global_loss_function}
\underset{x\in \mathbb{R}^d}{\min}f\left( x \right) \triangleq \frac{1}{n}\sum_{i=1}^n{f_i\left( x \right)},
\end{equation}
where $f_i\left( x \right) \triangleq \frac{1}{J}\sum_{j=1}^J{f_i\left( x;j \right)}$ is the loss function of node $i$ and all nodes collaborate to find a common model parameter $x$ minimizing their average loss functions.
We also assume that each node $i$ can only evaluate local stochastic gradient $\nabla f_i\left( x;\xi_i \right)$, $\xi_i \in \{1,2,...,J\}$.

For distributed parallel methods~\cite{li2014scaling,mcmahan2017communication} where there is a center (e.g., parameter server), they suffer from high communication burden and single point failure of the central node~\cite{lian2017can}.
These potential bottlenecks motivate researchers to study fully decentralized methods~\cite{lian2017can,lian2018asynchronous} to solve Problem~\eqref{global_loss_function}, where the central node is not required and each node only communicates with its neighbors. 
The existing decentralized learning algorithms usually employ undirected graphs for communication, which can not be easily implemented due to the existence of deadlocks~\cite{assran2019stochastic}.
It is desirable to consider more practical scenarios where communication graphs may be directed and even time-varying. Stochastic gradient push (SGP) proposed in~\cite{assran2019stochastic}, which builds on push-sum protocol~\cite{kempe2003gossip}, is proven to be very effective in solving Problem~\eqref{global_loss_function} over directed and time-varying communication graphs.

It has been well known that the frequent exchange of model parameters in decentralized learning may lead to severe concern on privacy leakage as the disclose of intermediate parameters could potentially compromise the original data~\cite{wang2019beyond}. For instance, previous studies~\cite{truex2019hybrid,carlini2019secret} have shown that the exposed parameters can be utilized to crack original data samples.
To address the above issue, differential privacy (DP), as a theoretical tool to provide rigorous privacy guarantees and quantify privacy loss, can be incorporated into each node in decentralized learning systems to enhance the privacy protection. DP techniques usually inject certain noises to disturb parameters for privacy preservation, which inevitably degrades the model accuracy. Besides, the variance of the added DP noise needs to be increased with the total number of iterations $K$ to ensure certain privacy guarantee due to the accumulated privacy leakage over communication rounds~\cite{dwork2014algorithmic,abadi2016deep,wu2020value,wei2021user}. 
In this regard, an excessive total number of iterations $K$ may severely deteriorate the model accuracy, and hence prohibits the implementation of DP in real decentralized learning systems. Therefore, given certain privacy budget, it is necessary to optimize $K$ to achieve a useful model with high accuracy. However, the few existing decentralized learning algorithms with DP guarantee for non-convex problems either do not consider optimizing $K$ under certain privacy guarantee~\cite{yu2021decentralized}, or their derived theoretical utility bound under the optimized $K$ cannot match that of the server-client distributed conterparts~\cite{xu2021dp}.



In this paper, we aim to design a differentially private decentralized learning algorithm for non-convex problems, find the optimal $K$ that attains maximized model accuracy given certain privacy budget, and achieve a tight utility bound matching that of the server-client distributed conterparts. We summarize our main contributions as follows:

\begin{itemize}
    \item \textbf{New efficient algorithm with personalized DP guarantee for each node.} 
    Different from the existing works, we propose a differentially private learning method (termed PrivSGP-VR) for non-convex problems, which can work over general time-varying directed communication topologies in \emph{fully decentralized} settings. More importantly, a personalized $(\epsilon_i,\delta_i)$-differential privacy (DP) is guaranteed for each node $i$, and variance-reduced technique is adopted to eliminate the effect of stochastic gradient noise, improving the convergence performance.
    \item \textbf{Linear speedup and tight utility bound.} 
    Under DP Gaussian noise with constant variance for each node, we derive a sub-linear convergence rate of $\mathcal{O}(\frac{1}{\sqrt{nK}})$ for PrivSGP-VR, which is independent of stochastic gradient variance and scales linearly w.r.t. the number of nodes $n$. More importantly, given certain privacy budget $(\epsilon_i, \delta_i)$ for each node $i$, leveraging the moments accountant method, we derive the optimized number of iterations $K$ to obtain a tight convergence error bound maximizing the model utility.
    With this optimized $K$, we achieve the utility bound of $\mathcal{O}\left( \sqrt{d\log \left( \frac{1}{\delta} \right)}/(\sqrt{n}J\epsilon) \right)$\footnote{Here, we set $\epsilon_i=\epsilon$ and $\delta_i=\delta$ for utility bound comparison.} for PrivSGP-VR, which matches that of the server-client distributed counterparts, and enjoys an extra factor of $1/\sqrt{n}$ improvement compared to that of the existing decentralized counterparts (c.f., Table~\ref{table_1}).
    \item \textbf{Extensive experimental evaluations.} 
    Extensive experiments on two training tasks are conducted to validate our theoretical findings. In particular, our experimental results show that properly setting the total number of iterations $K$ will significantly improve the model accuracy for the proposed PrivSGP-VR algorithm under certain privacy budget. To the best of our knowledge, this is the first empirical validation of the existence of an optimal choice of $K$ in the realm of differentially private decentralized learning. Besides, we validate the property of linear speedup for PrivSGP-VR employing DP noise with a constant variance. Moreover, we demonstrate the trade-off between maximizing model utility and ensuring privacy protection by executing PrivSGP-VR with various optimized numbers of iterations that correspond to different privacy budgets. 
\end{itemize}

\begin{table*}[t]
\setlength{\abovecaptionskip}{0.15cm}
\centering
    \begin{threeparttable}
    \begin{tabular}{|c|c|c|c|c|}
    \hline
    \rule{0pt}{10pt}
    \textbf{Algorithm} & \textbf{Privacy} & \textbf{Utility} & \textbf{Communication rounds} & \textbf{Architecture}  \\
    \hline
    \rule{0pt}{9pt}
    DP-SGD & \multirow{2}{*}{$(\epsilon,\delta)$-DP} &\multirow{2}{*}{$\frac{\sqrt{d\log \left( \frac{1}{\delta} \right)}}{J\epsilon}
$}  & \multirow{2}{*}{--} & single node \\
    \rule{0pt}{9pt}
     \cite{abadi2016deep}&&&& centralized\\
    \hline
    \rule{0pt}{9pt}
    Distributed DP-SRM\tnote{1} & $(\epsilon,\delta)$-DP &\multirow{2}{*}{$\frac{\sqrt{d\log \left( \frac{1}{\delta} \right)}}{nJ\epsilon}
$}  & \multirow{2}{*}{$\frac{n^2J\epsilon \sqrt{d}}{\sqrt{\log \left( \frac{1}{\delta} \right)}}
$} & $n$ nodes \\
    \rule{0pt}{9pt}
     \cite{wang2019efficient}& global&&& server-client \\
    \hline
    \rule{0pt}{9pt}
    LDP SVRG/SPIDER & $(\epsilon,\delta)$-DP &\multirow{2}{*}{$\frac{\sqrt{d\log \left( \frac{1}{\delta} \right)}}{\sqrt{n}J\epsilon}
$}  & \multirow{2}{*}{$\frac{n^{\frac{3}{2}}J\epsilon \sqrt{d}}{\sqrt{\log \left( \frac{1}{\delta} \right)}}
$} & $n$ nodes \\
    \rule{0pt}{9pt}
     \cite{lowy2023private}& for each node&&& server-client \\
    \hline
    \rule{0pt}{9pt}
    SoteriaFL-SAGA/SVRG& $(\epsilon,\delta)$-DP &\multirow{2}{*}{$\frac{\sqrt{\left( 1+\omega \right) d\log \left( \frac{1}{\delta} \right)}}{\sqrt{n}J\epsilon}$}  & \multirow{2}{*}{$\frac{\sqrt{n}J\epsilon}{\sqrt{\left( 1+\omega \right) d\log \left( \frac{1}{\delta} \right)}}
$} & $n$ nodes \\
    \rule{0pt}{9pt}
     \cite{li2022soteriafl}& for each node&&& server-client \\
    \hline
    \rule{0pt}{9pt}
    A(DP)$^2$SGD\tnote{2} & $(\epsilon,\delta)$-DP &\multirow{2}{*}{$\frac{\sqrt{d\log \left( \frac{1}{\delta} \right)}}{J\epsilon}$}  & \multirow{2}{*}{$\frac{J^2\epsilon ^2}{d\log \left( \frac{1}{\delta} \right)}$} & $n$ nodes \\
    \rule{0pt}{9pt}
     \cite{xu2021dp}& for each node&&& decentralized \\
    \hline
    \rule{0pt}{9pt}
    PrivSGP-VR & $(\epsilon,\delta)$-DP &\multirow{2}{*}{$\frac{\sqrt{d\log \left( \frac{1}{\delta} \right)}}{\sqrt{n}J\epsilon}
$}  & \multirow{2}{*}{$\frac{J^2\epsilon ^2}{d\log \left( \frac{1}{\delta} \right)}
$} & $n$ nodes \\
    \rule{0pt}{9pt}
     (Our Algorithm~\ref{PrivSGP-VR})& for each node&&& decentralized \\
    \hline
    \end{tabular}
    \begin{tablenotes}
        \item[1] Wang \textit{et al.}~\shortcite{wang2019efficient} only consider global $(\epsilon,\delta)$-DP that merely protects the privacy for the entire dataset while we consider $(\epsilon,\delta)$-DP for each node, which can protect the local dataset at the node's level.
        \item[2] For A(DP)$^2$SGD, the authors only provide the utility bound under global $(\epsilon,\delta)$-DP for the entire dataset. We thus derive their utility bound in the sense of ensuring $(\epsilon,\delta)$-DP for each node for fair comparison.
    \end{tablenotes}
    \end{threeparttable}
\caption{Comparison of existing differentially private stochastic algorithms for the non-convex problem in terms of privacy, utility and communication complexity. 
DP-SGD is the centralized (single-node) stochastic learning algorithm serving as a baseline.
Distributed DP-SRM and LDP SVRG/SPIDER are server-client distributed learning algorithms without communication compression.
SoteriaFL-SAGA/SVRG are server-client distributed learning algorithms with communication compression.
$\omega$ is the parameter for unbiased compression in SoteriaFL-SAGA/SVRG ($\omega=0$ corresponds to no compression).
A(DP)$^2$SGD is decentralized learning algorithm.
For comparison, we set $\epsilon_i=\epsilon$ and $\delta_i=\delta$ for each node $i$ in our PrivSGP-VR. 
The Big $\mathcal{O}$ notation is omitted for simplicity.
}
\label{table_1}
\vspace{-11pt}
\end{table*}

\section{Related Works}
Differential privacy (DP) was first proposed in~\cite{dwork2006our} to protect data privacy for database queries.
A DP mechanism adds randomly generated zero-mean noises to the output of a query function before it is exposed, making it difficult for curious attackers to extract users' private information from the distorted query results. 
The basic composition theorem~\cite{dwork2006our,dwork2009differential} and advanced composition theorem~\cite{dwork2010boosting,bun2016concentrated} are commonly used for computing the overall accumulated privacy loss in iterative training processes. 
However, these theorems can result in loose estimates of privacy loss. 
To address this issue, the moments accountant method proposed in~\cite{abadi2016deep} obtains a much tighter estimate on the overall privacy loss by tracking higher moments and thus provides a more accurate way for calculating the privacy spending.

There has been a recent surge in research efforts towards achieving differential privacy guarantees in large-scale machine learning systems.
Abadi \textit{et al.}~\shortcite{abadi2016deep}; Wang \textit{et al.}~\shortcite{wang2017differentially}; Chen \textit{et al.}~\shortcite{chen2020understanding}; Wang \textit{et al.}~\shortcite{wang2020differentially} design differentially private stochastic learning algorithms in a centralized setting.
For distributed\footnote{Here, by being distributed, we mean sever-client architecture.} settings, Laplace and Gaussian mechanisms have been incorporated into federated learning, with corresponding convergence rates analyzed, respectively~\cite{wu2020value,wei2020federated,wei2021user}.
Truex \textit{et al.}~\shortcite{truex2020ldp} explore differential privacy guarantee for each client in federated personalized model learning.
In~\cite{zhou2023optimizing}, the authors consider optimizing the numbers of queries and replies in federated learning to maximize the model utility given certain privacy budget for strongly convex problems.
Zhang \textit{et al.}~\shortcite{zhang2020private}; Li \textit{et al.}~\shortcite{li2022soteriafl} achieve both differential privacy and communication compression in federated learning for non-convex problems, and provided a characterization of trade-offs in terms of privacy, utility, and communication complexity.
There are also other works dedicated to designing differentially private distributed learning algorithms, such as~\cite{wang2019efficient,li2019asynchronous,zeng2021differentially,liu2022loss,lowy2023private,wang2023efficient}, but all the above-mentioned distributed methods are only applicable to the server-client architecture.

Recently, there have been few works aiming to achieve differential privacy for fully decentralized learning algorithms. 
For example, the works in~\cite{cheng2018leasgd,cheng2019towards} achieve differential privacy in fully decentralized learning systems for strongly convex problems. 
Wang \textit{et al.}~\shortcite{wang2022tailoring} achieve differential privacy in fully decentralized architectures by tailoring gradient methods for deterministic optimization problems. 
Yu \textit{et al.}~\shortcite{yu2021decentralized} present a decentralized stochastic learning method for non-convex problems with differential privacy guarantee (DP$^2$-SGD) based on D-PSGD~\cite{lian2017can}, which relies on a fixed communication topology and uses the basic composition theorem to bound the overall privacy loss.
To have a tight privacy guarantee, Xu \textit{et al.}~\shortcite{xu2021dp} propose a differentially private asynchronous decentralized learning algorithm (A(DP)$^{2}$SGD) for non-convex problems based on AD-PSGD~\cite{lian2018asynchronous}, which provides privacy guarantee in the sense of R\'enyi differential privacy (RDP)~\cite{mironov2017renyi}. However, it should be noted that the above-mentioned two fully decentralized differentially private algorithms~\cite{yu2021decentralized,xu2021dp} work only for undirected communication graphs, which is often not satisfied in practical scenarios, and their convergence performance suffer from the effect of stochastic gradient variance. Moreover, none of them provide experimental evaluation to verify that selecting an appropriate value of $K$ can, indeed, improve the model utility (accuracy) under given certain privacy budget.

On the theoretical level, for general non-convex problems, a utility bound of $\mathcal{O}\left( \sqrt{d\log \left( \frac{1}{\delta} \right)}/\left( J\epsilon \right ) \right)$ is established for centralized learning with DP~\cite{abadi2016deep}, and a utility bound of $\mathcal{O}\left( \sqrt{d\log \left( \frac{1}{\delta} \right)}/(\sqrt{n}J\epsilon) \right)$ is provided for server-client distributed algorithms with DP~\cite{lowy2023private,li2022soteriafl}, which scales as $1/\sqrt{n}$ w.r.t. the number of nodes $n$. For DP-based fully decentralized algorithms as mentioned above, DP$^2$-SGD~\cite{yu2021decentralized} lacks a theoretical utility guarantee under a given privacy budget, while the utility bound of A(DP)$^{2}$SGD~\cite{xu2021dp} can not match that of the server-client distributed counterparts, losing a scaling factor of $1/\sqrt{n}$ (c.f., Table~\ref{table_1}).


\section{Algorithm Development}
We consider solving Problem~\eqref{global_loss_function} over the following general network model.

\textbf{Network Model.}
The communication topology considered in this work is modeled as a sequence of time-varying directed graph $\mathcal{G}^k=\left( \mathcal{V},\mathcal{E}^k \right) $, where $\mathcal{V}=\{1,2,...,n\}$ denotes the set of nodes and $\mathcal{E}^k \subset \mathcal{V} \times \mathcal{V}$ denotes the set of directed edges/links at iteration $k$. 
We associate each graph $\mathcal{G}^k$ with a non-negative mixing matrix $P^k \in \mathbb{R}^{n \times n}$ such that $(i,j) \in \mathcal{E}^k$ if $P_{i,j}^k > 0$, i.e., node $i$ receiving a message from node $j$ at iteration $k$.
Without loss of generality, we assume that each node is an in-neighbor of itself.

The following assumptions are made on the mixing matrix and graph for the above network model to facilitate the subsequent convergence analysis for the proposed algorithm.
\begin{Ass}[Column Stochastic Mixing Matrix]
\label{Ass_weight_matrix}
For any iteration $k$, the non-negative mixing matrix $P^k$ is column-stochastic, i.e., $\mathbf{1}^\top P^k=\mathbf{1}^\top$, where $\mathbf{1}$ is a column vector with all of its elements equal to $1$.
\end{Ass}

\begin{Ass}[$B$-strongly Connected Graph]
\label{assumption_mixing_matrix}
We assume that there exists finite, positive integers $B$ and $\bigtriangleup$, such that the graph with edge set $\bigcup\nolimits_{k=lB}^{\left( l+1 \right) B-1}{\mathcal{E}^k}$ is strongly connected and has diameter at most $\bigtriangleup$ for $\forall 
 l \geqslant 0$.
\end{Ass}

Before developing our proposed algorithm, we briefly introduce the following definition of $(\epsilon, \delta)$-DP~\cite{dwork2014algorithmic}, which is crucial to subsequent analysis.
\begin{Def}[$(\epsilon,\delta)$-DP]
A randomized mechanism $\mathcal{M}$ with domain $\mathcal{D}$ and range $\mathcal{R}$ satisfies $(\epsilon,\delta)$-differential privacy, or $(\epsilon,\delta)$-DP for short, if for any two adjacent inputs $\mathrm{x},\mathrm{x}^{\prime}\in \mathcal{D}$ differing on a single entry and for any subset of outputs $S\subseteq \mathcal{R}$, it holds that
\begin{equation}
Pr\left[ \mathcal{M}\left( \mathrm{x} \right) \in S \right] \leqslant e^{\epsilon}Pr\left[ \mathcal{M}\left( \mathrm{x}^{\prime} \right) \in S \right] +\delta ,
\end{equation}
where the privacy budget $\epsilon$ denotes the privacy lower bound to
measure a randomized query and $\delta$ is the probability of
breaking this bound.
\end{Def}
It can be observed that the smaller the values of $\epsilon$ and $\delta$ are, the higher the level of privacy guarantee will be. In this paper, we allow each node $i$ to tolerate different level of privacy loss, yielding personalized privacy budget $(\epsilon_i,\delta_i$) for each node $i$.
Now, we are ready to present our differentially private decentralized learning algorithm as follows.

\paragraph{Stochastic gradient push over time-varying directed graphs.}
We first introduce decentralized stochastic gradient push method based on Push-Sum protocol~\cite{kempe2003gossip}, which can tackle the unblanceness of directed topologies by asymptotically estimating the Perron–Frobenius eigenvector of transition matrices. 
In particular, each node $i$ maintains three variables during the learning process: i) the model parameter $x_i^k$; ii) the scalar Push-Sum weight $w_i^k$ and iii) the de-biased parameter $z_i^k=x_i^k/w_i^k$, with the same initialization of $x_i^0=z_i^0=x^0 \in \mathbb{R}^d$ and $w_i^0=1$ for all nodes $i \in \{1,2,...,n\}$.
At each iteration $k$, each node $i$ updates as follows:
\begin{equation*}
\begin{aligned}
&\texttt{Local SGD:} \quad x_{i}^{k+\frac{1}{2}}=x_{i}^{k}-\gamma \nabla f_i\left( z_{i}^{k};\xi _{i}^{k} \right) , 
\\
&\texttt{Averaging:} \quad x_{i}^{k+1}=\sum_{j=1}^n{P_{i,j}^{k}x_{j}^{k+\frac{1}{2}}},w_{i}^{k+1}=\sum_{j=1}^n{P_{i,j}^{k}w_{j}^{k}},
\\
&\texttt{De-bias:} \quad  z_{i}^{k+1}=x_{i}^{k+1}/w_{i}^{k+1},
\end{aligned}
\end{equation*}
where $\gamma > 0$ is the step size and $\nabla f_i(z_i^k;\xi_i^k)$ is the stochastic gradient evaluated on the de-biased parameter $z_i^k$. Note that, during the training process, each node exchanges model parameter with its neighbors frequently for averaging, resulting in potential privacy leakage as the original data could be recovered based on the disclosed model parameters.

\paragraph{Ensuring differential privacy guarantee for each node.} 
We apply the differential privacy mechanism to protect the exchanged sensitive model parameters of each node. In particular, for each node $i$, the exchanged model parameter $x_i^{k+\frac{1}{2}}$ is obtained by performing a \texttt{Local SGD} step using the gradient $\nabla f_i(z_i^k;\xi_i^k)$. Since perturbing the gradient is equivalent to perturbing the model parameter, we thus inject randomly generated noise to the gradient $\nabla f_i(z_i^k;\xi_i^k)$ instead of directly adding noise to the exchanged model parameter in the proposed approach as follows:
\begin{equation}
\nabla \tilde{f}_i(z_{i}^{k};\xi _{i}^{k})=\nabla f_i(z_{i}^{k};\xi _{i}^{k})+N_i^k
\end{equation}
where the noise $N_i^k$ is drawn from the Gaussian distribution $\mathcal{N}\left( 0,\sigma_i ^2\mathbb{I}_d \right) $ and $\mathbb{I}_d$ represents the identity matrix with $d$ dimension. 
Then, the \texttt{Local SGD} step becomes:
\begin{equation}
\label{sgd_with_dp_nosie}
x_{i}^{k+\frac{1}{2}}=x_{i}^{k}-\gamma \nabla \tilde{f}_i(z_{i}^{k};\xi _{i}^{k})=x_{i}^{k}-\gamma \left( \nabla f_i(z_{i}^{k};\xi _{i}^{k})+N_{i}^{k} \right) .
\end{equation}
We will refer to the above generated differentially private algorithm as PrivSGP (its pseudo-code can be found in Algorithm~\ref{PrivSGP} in Appendix~\ref{appendix_alg}).

\begin{algorithm}[tb]
\caption{PrivSGP-VR}
\label{PrivSGP-VR}
\textbf{Initialization}: $x_{i}^{0}=z_{i}^{0}=x^0 \in \mathbb{R}^d$, $w_i^0=1$ and privacy budget $(\epsilon_i,\delta_i)$ for all $i \in \mathcal{V}$, step size $\gamma > 0$, and  total number of iterations $K$.
\begin{algorithmic}[1] 
\FOR{$j \in \left\{ 1,2,...,J \right\}$, at node $i$,}
\STATE Initializes \quad $\nabla f_i(\phi_{i,j};j)=\nabla f_i(z_i^0;j)$
\ENDFOR
\FOR{$k=0,1,2,...,K-1$, at node $i$,}
\STATE Randomly samples a local training data $\xi_i^k$ with the sampling probability $\frac{1}{J}$;
\STATE Computes the corrected gradient by
        \begin{equation*}
        g_i^k=\nabla f_i(z_i^k;\xi_i^k)-\nabla f_i(\phi_{i,\xi_i^k};\xi_i^k)+\frac{1}{J}\sum_{j=1}^J \nabla f_i(\phi_{i,j};j);
        \end{equation*}
\STATE Stores gradient:\quad $\nabla f_i(\phi_{i,\xi_i^k};\xi_i^k)=\nabla f_i(z_i^k;\xi_i^k) $;
\STATE Adds noise $\tilde{g}_{i}^{k}=g_{i}^{k}+N_{i}^{k}$, where $N_i^k \in 
    \mathbb{R}^d \thicksim \mathcal{N}\left( 0,\sigma_i^2\mathbb{I}_d \right) 
$ and $\sigma_i$ is defined in Theorem~\ref{Theorem_3};
\STATE Generates intermediate model parameter:  $x_i^{k+\frac{1}{2}}=x_i^k-\gamma \tilde{g}_{i}^{k}$ ;
\STATE Sends $\left( x_i^{k+\frac{1}{2}}, w_i^k \right)$ to out-neighbors;
\STATE Receives $\left( x_j^{k+\frac{1}{2}}, w_j^k \right)$ from in-neighbors;
\STATE Updates $x_i^{k+1}$ by: \quad $x_{i}^{k+1}=\sum_{j=1}^n{P_{i,j}^{k}x_{j}^{k+\frac{1}{2}}}$;
\STATE Updates $w_i^{k+1}$ by: \quad $w_{i}^{k+1}=\sum_{j=1}^n{P_{i,j}^{k}w_{j}^{k}}$;
\STATE Updates $z_i^{k+1}$ by: \quad     
        $z_{i}^{k+1}=x_{i}^{k+1}/w_i^{k+1}$;
\ENDFOR
\end{algorithmic}
\end{algorithm}

\paragraph{Eliminating the stochastic gradient noise.} We now introduce the variance reduction technique~\cite{defazio2014saga} to eliminate the effect of stochastic gradient noise of each node on convergence performance.
Specifically, each node $i$ maintains a stochastic gradient table for all of its own local data samples.
At each iteration $k$, after computing the stochastic gradient $\nabla f_i\left( z_i^k;\xi _{i}^{k} \right)$, node $i$ does not perform a local differentially private SGD step using $\nabla f_i\left( z_i^k;\xi _{i}^{k} \right)$ directly (c.f.,~\eqref{sgd_with_dp_nosie}). 
Instead, $\nabla f_i\left( z_i^k;\xi _{i}^{k} \right)$ is corrected by subtracting the previously stored stochastic gradient corresponding to the $\xi_i^k$-th data sample, and then adding the average of all stored stochastic gradients.
With such a corrected stochastic gradient, node $i$ performs a local differentially private SGD step and replaces the stochastic gradient of the $\xi_i^k$-th data sample in the table with $\nabla f_i\left( z_i^k;\xi _{i}^{k} \right)$.
To better understand this process, let 
\begin{equation}
\label{saga_table}
\phi _{i,j}^{k+1}=\left\{ \begin{array}{c}
\phi _{i,j}^{k}\,\,  ,  j\ne \xi _{i}^{k}\\
z_i^k\,\,   ,  j=\xi _{i}^{k}\\
\end{array} \right. ,
\end{equation}
where $\phi _{i,j}^{k}$ is the most recent model parameter used for computing $\nabla f_i\left( \cdot;j \right)$ prior to iteration $k$.
Thus, $\nabla f_i\left( \phi _{i,j}^{k};j \right)$ represents the previously stored stochastic gradient for the $j$-th data sample of node $i$ prior to iteration $k$, and 
\begin{equation}
\label{corrected_stochastic_gradient}
g_{i}^{k} \triangleq \nabla f_i\left( z_i^k;\xi _{i}^{k} \right) -\nabla f_i\left( \phi _{i,\xi _{i}^{k}}^{k};\xi _{i}^{k} \right) +\frac{1}{J}\sum_{j=1}^J{\nabla f_i\left( \phi _{i,j}^{k};j \right)}
\end{equation}
is the corrected stochastic gradient of node $i$ at iteration $k$.
As a result, we replace the original stochastic gradient $\nabla f_i\left( z_i^k;\xi _{i}^{k} \right)$ in~\eqref{sgd_with_dp_nosie} with $g_{i}^{k}$, leading to the following new local differentially private SGD step, i.e.,
\begin{equation}
x_{i}^{k+\frac{1}{2}}=x_{i}^{k}-\gamma \left( g_{i}^{k}+N_{i}^{k} \right),
\end{equation}
which yields the proposed differentially private decentralized learning method PrivSGP-VR, whose complete pseudocode is summarized in Algorithm~\ref{PrivSGP-VR}.

\section{Theoretical Analysis}
In this section, we provide utility and privacy guarantees for the proposed PrivSGP-VR method.

\subsection{Convergence Guarantee}
To facilitate our convergence analysis, we make the following commonly used assumptions.

\begin{Ass}[Smoothness]
\label{assumption_smooth_saga}
For each node $i$,  $ \forall x\in \mathbb{R}^d$ and sample $\forall \xi_i \in  \{1,2,...,J\}$, the local sample loss function $ f_i(x;\xi_i)$ has $L$-Lipschitz continuous gradients . 
\end{Ass}

\begin{Ass}[Unbiased Gradient]
\label{assumption_unbiased_gradient}
For any model $x\in \mathbb{R}^d$, the stochastic gradient $\nabla f_i\left( x;\xi _i \right), \xi_i \sim  \{1,2,...,J\}$ generated by each node $i$ is unbiased, i.e.,
\begin{equation}
\mathbb{E}\left[ \nabla f_i\left( x;\xi _i \right) \right] =\nabla f_i\left( x \right) .
\end{equation}
\end{Ass}

\begin{Ass}[Bounded Data Heterogeneity]
\label{assumption_bounede_outer_variation}
There exists a finite positive constant $b^2$ such that for any node $ i$ and $\forall x\in \mathbb{R}^d$,
\begin{equation}
\left\| \nabla f_i\left( x \right) -\nabla f\left( x \right) \right\| ^2\leqslant b^2
.
\end{equation}
\end{Ass}

With the above assumptions, we have the following convergence result for PrivSGP-VR (Algorithm~\ref{PrivSGP-VR}).

\begin{Thm}[Convergence Rate]
\label{Theorem_1_saga} 
Suppose Assumptions~\ref{Ass_weight_matrix}-\ref{assumption_bounede_outer_variation} hold. Let $K$ be the total number of iterations and $f^*=\underset{x\in \mathbb{R}^d}{\min}f\left( x \right)$. If the step-size is set as $\gamma=\sqrt{\frac{n}{K}}$, then there exist constants $C$ and $q \in [0,1)$, which depend on the diameter of the network $\bigtriangleup $ and the sequence of mixing matrices $P^k$, such that, for any $K$ satisfying $K \geqslant \hat{K}(C,q)$,
we have
\begin{equation}
\label{mian_inequal_saga}
\begin{aligned}
&\frac{1}{K}\sum_{k=0}^{K-1}{\frac{1}{n}\sum_{i=1}^n{\mathbb{E}\left[ \left\| \nabla f\left( z_{i}^{k} \right) \right\| ^2 \right]}}
\\
\leqslant & \frac{13F^0+6L\left\| x^0 \right\| ^2+18Lb^2+24L\cdot \frac{d}{n}\sum_{i=1}^n{\sigma _{i}^{2}}}{\sqrt{nK}}
,
\end{aligned} 
\end{equation}
where $F^0=f\left( x^0 \right) -f^*$,  $C$ and $q$ can be found in Lemma~\ref{def_of_C_and_q}, and the definition of constant $\hat{K}(C,q)$ can be found at~\eqref{total_total_iteration_K} in the appendix, respectively.
\end{Thm}
\begin{proof}
See Appendix~\ref{proof_of_theorem}. 
\end{proof}

\begin{Rem}
Under DP Gaussian noise with a constant variance, the above result suggests that the convergence rate for PrivSGP-VR is $\mathcal{O}(\frac{1}{\sqrt{nK}})$, which is independent of stochastic gradient variance $\zeta ^2$ with $\mathbb{E}\left[ \left\| \nabla f_i\left( x;\xi _i \right) -\nabla f_i\left( x \right) \right\|^2 \right] \leqslant \zeta ^2$ that appears in the convergence error bound in~\cite{xu2021dp,yu2021decentralized}, and achieves linear speedup with respect to the number of nodes. 
Although it converges to an exact stationary point as the total number of iterations $K$ goes to infinity, the privacy loss will also become infinite according to the composition theorem~\cite{dwork2006our}.
As a result, it is necessary to consider the trade-off between the model utility and privacy guarantee for the proposed PrivSGP-VR algorithm under certain given privacy budget.
\end{Rem}

\subsection{Privacy and Utility Guarantee}
Leveraging the moments accountant method~\cite{abadi2016deep}, we can calculate the variance $\sigma_i^2$ of the DP Gaussian noise needed to be added according to the total number of iterations $K$ and the given privacy budget $(\epsilon_i, \delta_i)$, which is provided in the following theorem.

\begin{Thm}[Privacy Guarantee]
\label{Theorem_3}
Suppose the stochastic gradient of each $f_i$ is uniformly bounded, i.e., $G=\underset{k,i}{\max}\left\| \nabla f_i\left( z_{i}^{k};\xi _{i}^{k} \right) \right\| <\infty 
$. There exist constants $c_1$ and $c_2$ such that, given the total number of iterations $K$ for Algorithm~\ref{PrivSGP-VR}, $(\epsilon_i,\delta_i)$-differential privacy can be guaranteed for each node $i$, for any $\epsilon_i <\frac{c_1K}{J^2}$ and $\delta_i \in (0,1)$, if $N_i^k$ is drawn from the Gaussian distribution $\mathcal{N}\left( 0,\sigma_i ^2\mathbb{I}_d \right) $ with
\begin{equation}
\label{noise_scale_saga}
\sigma _i=3c_2G\frac{\sqrt{K\log \left( \frac{1}{\delta _i} \right)}}{J\epsilon_i}
.
\end{equation}
\end{Thm}
\begin{proof}
See Appendix~\ref{proof_of_moments_accountant}.
\end{proof}

As highlighted in Theorem~\ref{Theorem_3}, it is evident that when a certain privacy budget $(\epsilon_i, \delta_i)$ is given, a larger value of $K$ requires the added DP Gaussian noise with a larger variance $\sigma_i^2$. This can potentially impact the model utility negatively.
Therefore, our objective is to optimize the value of $K$ in order to maximize the final model accuracy under certain privacy budget $(\epsilon_i, \delta_i)$ for each node $i$.

Plugging \eqref{noise_scale_saga} into \eqref{mian_inequal_saga} in Theorem~\ref{Theorem_1_saga}, we obtain
the following utility guarantee.
\begin{Cor}[Maximized Utility Guarantee]
\label{proposition_saga}
Given certain privacy budget $(\epsilon_i,\delta_i)$ for each node $i\in \{1,2,...,n\}$, under the same conditions of Theorem~\ref{Theorem_1_saga} and~\ref{Theorem_3}, if the total number of iterations $K$ further satisfies
\begin{equation}
\label{value_of_T_saga}
K=\frac{\left( 13 F^0 +6L\left\| x^0 \right\| ^2+18Lb^2 \right) J^2n}{216Ldc_{2}^{2}G^2\sum_{i=1}^n{\frac{1}{\epsilon _{i}^{2}}\log \left( \frac{1}{\delta _i} \right)}}
,
\end{equation}
then we have
\begin{equation}
\label{upper_bound_last_saga}
\begin{aligned}
&\frac{1}{K}\sum_{k=0}^{K-1}{\frac{1}{n}\sum_{i=1}^n{\mathbb{E}\left[ \left\| \nabla f\left( z_{i}^{k} \right) \right\| ^2 \right]}} \leqslant
\\
&\frac{12c_2G\sqrt{6Ld\left( 13F^0+6L\left\| x^0 \right\| ^2+18Lb^2 \right) \sum_{i=1}^n{\frac{1}{\epsilon _{i}^{2}}\log \left( \frac{1}{\delta _i} \right)}}}{nJ}.
\end{aligned}
\end{equation}
\end{Cor}
\begin{proof}
Substituting the Gaussian noise level in~\eqref{noise_scale_saga} into~\eqref{mian_inequal_saga}, we have
\begin{equation}
\begin{aligned}
\label{last_inequal_saga}
& \frac{1}{K}\sum_{k=0}^{K-1}{\frac{1}{n}\sum_{i=1}^n{\mathbb{E}\left[ \left\| \nabla f\left( z_{i}^{k} \right) \right\| ^2 \right]}}\leqslant  \frac{13 F^0 +6L\left\| x^0 \right\| ^2+18Lb^2}{\sqrt{nK}}
\\
&+\sqrt{K}\cdot \frac{216Ldc_{2}^{2}G^2}{J^2\sqrt{n}}\cdot \frac{1}{n}\sum_{i=1}^n{\frac{1}{\epsilon _{i}^{2}}\log \left( \frac{1}{\delta _i} \right)}.
\end{aligned}
\end{equation}
Regarding the right hand side of \eqref{last_inequal_saga} as a function of $K$, we can obtain the optimal value of $K$ (c.f.,~\eqref{value_of_T_saga}) 
and error bound (c.f.,~\eqref{upper_bound_last_saga}) by minimizing this function.
\end{proof}

\begin{Rem}
\label{remark_for_trade_off}
Corollary~\ref{proposition_saga} provides a valuable insight regarding the optimization of the total number of iterations $K$ under certain privacy budget. It shows that there exists an optimal value of $K$ that minimizes the error bound and thus maximizes the model accuracy. This suggests that simply increasing the value of $K$ will not necessarily lead to improved results, and it is important to carefully select the appropriate number of iterations. Furthermore, \eqref{upper_bound_last_saga} highlights the trade-off between privacy and maximized model utility. It shows that stronger privacy protection, represented by smaller privacy budget $(\epsilon_i,\delta_i)$, leads to a larger minimum error bound and thus worse maximized model utility. This finding underscores the inherent tension between privacy and model performance in decentralized learning scenarios.
\end{Rem}

\begin{Rem} Note that if one set $\epsilon_i=\epsilon$ and $\delta_i=\delta$ for each node $i$, the utility bound of the proposed PrivSGP-VR~\eqref{upper_bound_last_saga} will be reduced to $\mathcal{O}\left( \sqrt{d\log \left( \frac{1}{\delta} \right)}/(\sqrt{n}J\epsilon) \right)$, achieving the same utility guarantee as differentially private learning algorithms with server-client structure, such as LDP SVRG/SPIDER~\cite{lowy2023private}, and SoteriaFL-SAGA/SVRG~\cite{li2022soteriafl} without communication compression ($\omega=0$), see Table~\ref{table_1}.
Due to the presence of network dynamics in a fully decentralized time-varying setting, it is not surprising that the proposed PrivSGP-VR requires more communication rounds than that of other sever-client distributed counterparts. 
In addition, PrivSGP-VR recovers the same utility $\mathcal{O}\left( \sqrt{d\log \left( \frac{1}{\delta} \right)}/(J\epsilon) \right)$ as the baseline DP-SGD~\cite{abadi2016deep} when $n=1$.
Furthermore, it can be observed that the utility bound of our PrivSGP-VR is tighter than that of the existing decentralized counterpart A(DP)$^2$SGD~\cite{xu2021dp}, exhibiting an extra factor of $1/\sqrt{n}$ improvement. To the best of our knowledge, we are the first to derive such a utility bound scaling as $1/\sqrt{n}$ with respect to the number of nodes in the realm of decentralized learning with DP guarantee for each node, for general non-convex problems.
\end{Rem}

\section{Experiments}
We conduct extensive experiments to validate the theoretical findings for the proposed PrivSGP-VR under various settings. All experiments are deployed in a high performance computer with Intel Xeon E5-2680 v4 CPU @ 2.40GHz and 8 Nvidia RTX 3090 GPUs, and are implemented with distributed communication package \textit{torch.distributed} in PyTorch, where a process serves as a node, and inter-process communication is used to mimic communication among nodes. We consider two non-convex learning tasks (i.e., deep CNN ResNet-18~\cite{he2016deep} training on Cifar-10 dataset~\cite{krizhevsky2009learning} and shallow 2-layer neural network training on Mnist dataset~\cite{deng2012mnist}), in fully decentralized setting. For all experiments, we split shuffled datasets evenly to $n$ nodes. For communication topology, unless otherwise stated, we use time-varying directed exponential graph (refer to Appendix~\ref{missing_definition_of_graph} for its definition) for our PrivSGP-VR. 

\subsection{Deep CNN ResNet-18 training}
We first report the experiments of training CNN model ResNet-18 on Cifar-10 dataset.
Once the dataset and learning model are given, the problem-related parameters such as $L$ and $b^2$ can be estimated by leveraging the method introduced in~\cite{wang2019adaptive,luo2021cost}. The values of these parameters are $L=25, G^2=100, f(\bar{x}^0)-f^*=2.8, b^2=500000$ and $\left\| x^0 \right\| ^2 =780000$, for ResNet-18 training task.

\begin{figure}[!htpb]
 \vspace{-10pt}
\setlength{\abovecaptionskip}{-0cm}
  \centering
  \subfigure[Training loss]{
    \includegraphics[width=0.45\linewidth]{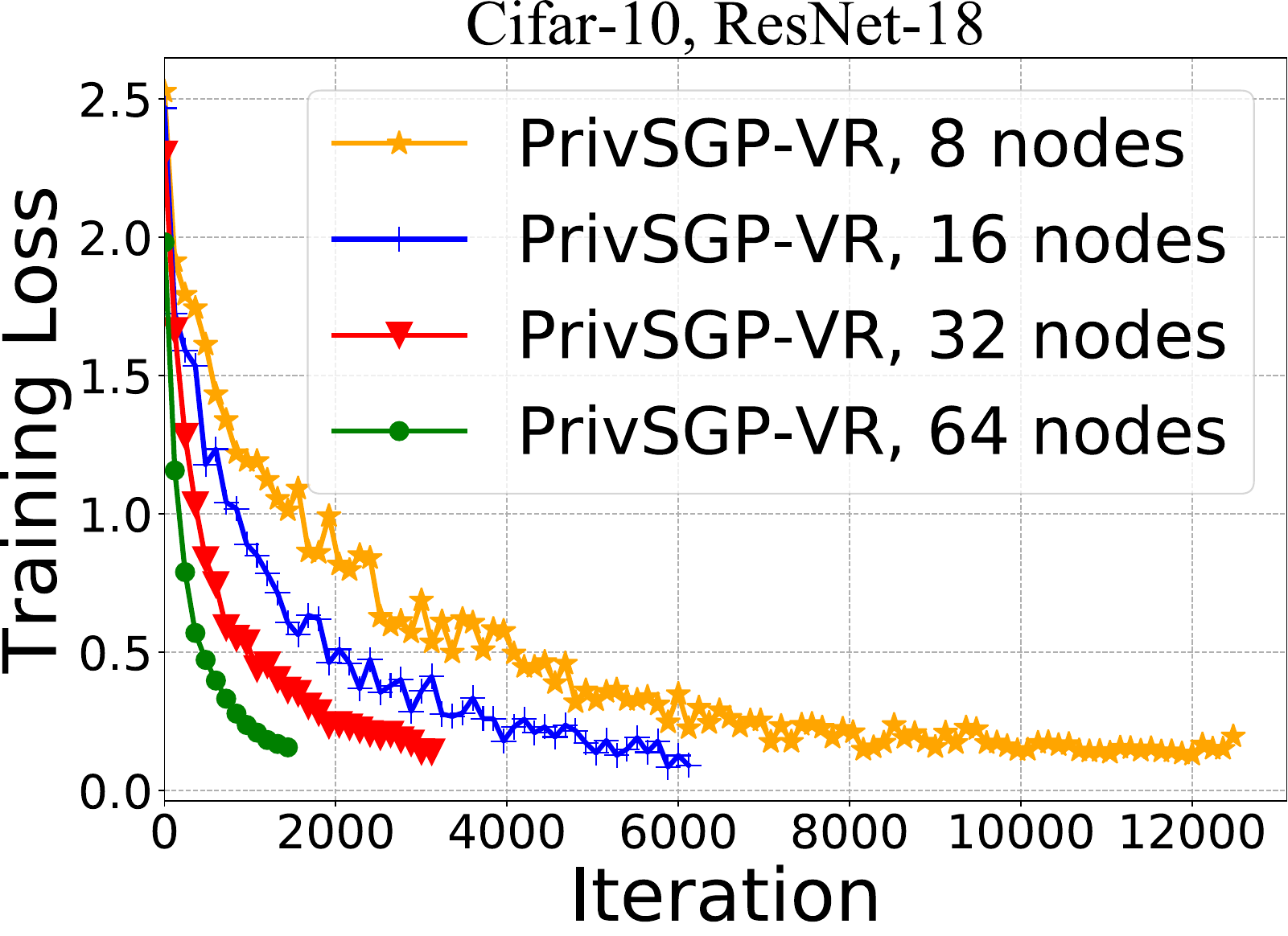}
    \label{loss_speed_up}
  }
  \hfill
  \subfigure[Testing accuracy]{
    \includegraphics[width=0.45\linewidth]{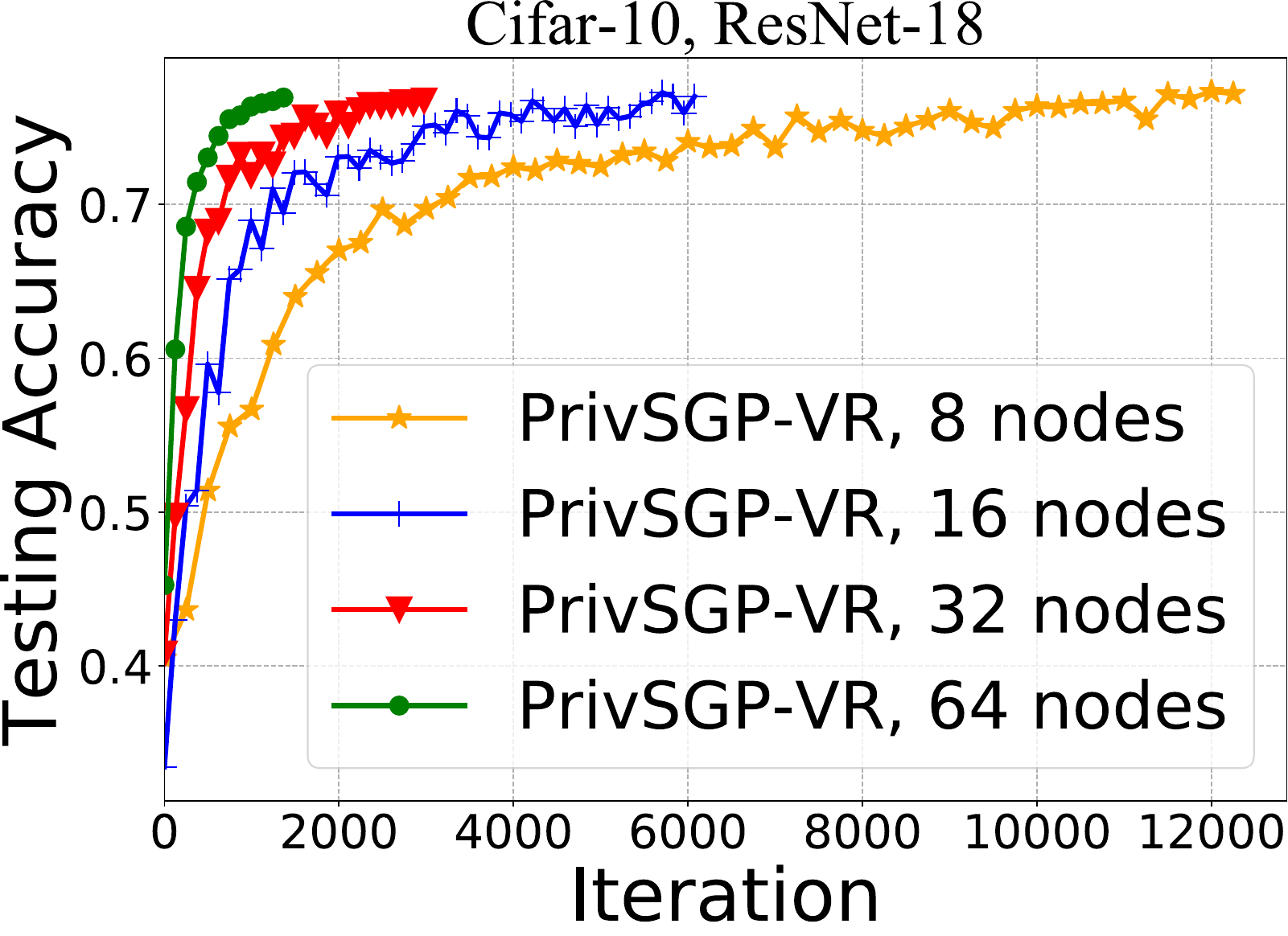}
    \label{acc_speed_up}
  }
  \caption{Comparison of convergence performance for PrivSGP-VR over 8, 16, 32 and 64 nodes under the same DP Gaussian noise variance, when training ResNet-18 on Cifar-10.}
  \label{speed up}
  \vspace{-10pt}
\end{figure}

\paragraph{Linear speedup under constant DP Gaussian noise variance.}
We first illustrate the convergence and scalability in terms of number of nodes $n$ of PrivSGP-VR.
In our experimental setup, we implement PrivSGP-VR on 4 distinct network configurations, comprising 8, 16, 32 and 64 nodes, respectively. All configurations utilize the same DP Gaussian noise variance $\sigma_i^2=0.03$ for each node $i$. 
It can be observed from Figure~\ref{speed up} that, by increasing the number of nodes by a factor of 2, we can achieve comparable final training loss and model testing accuracy by running only half the total number of iterations.
This observation illustrates the linear speedup property exhibited by our PrivSGP-VR algorithm.


\begin{figure}[!htpb]
\vspace{-10pt}
\setlength{\abovecaptionskip}{-0cm}
  \centering
  \subfigure[Training loss]{
    \includegraphics[width=0.45\linewidth]{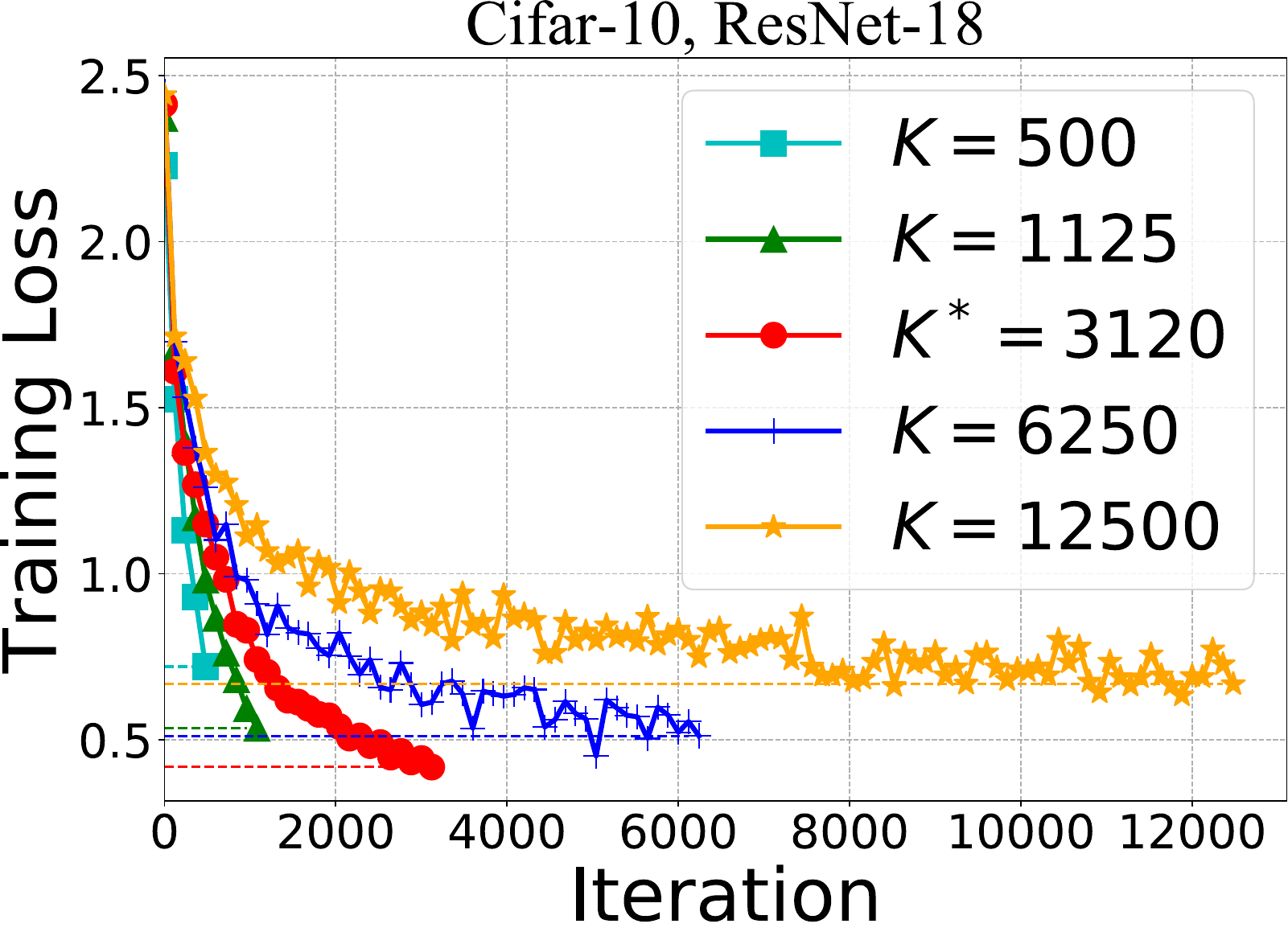}
    \label{loss_different_K}
  }
  \hfill
  \subfigure[Testing accuracy]{
    \includegraphics[width=0.45\linewidth]{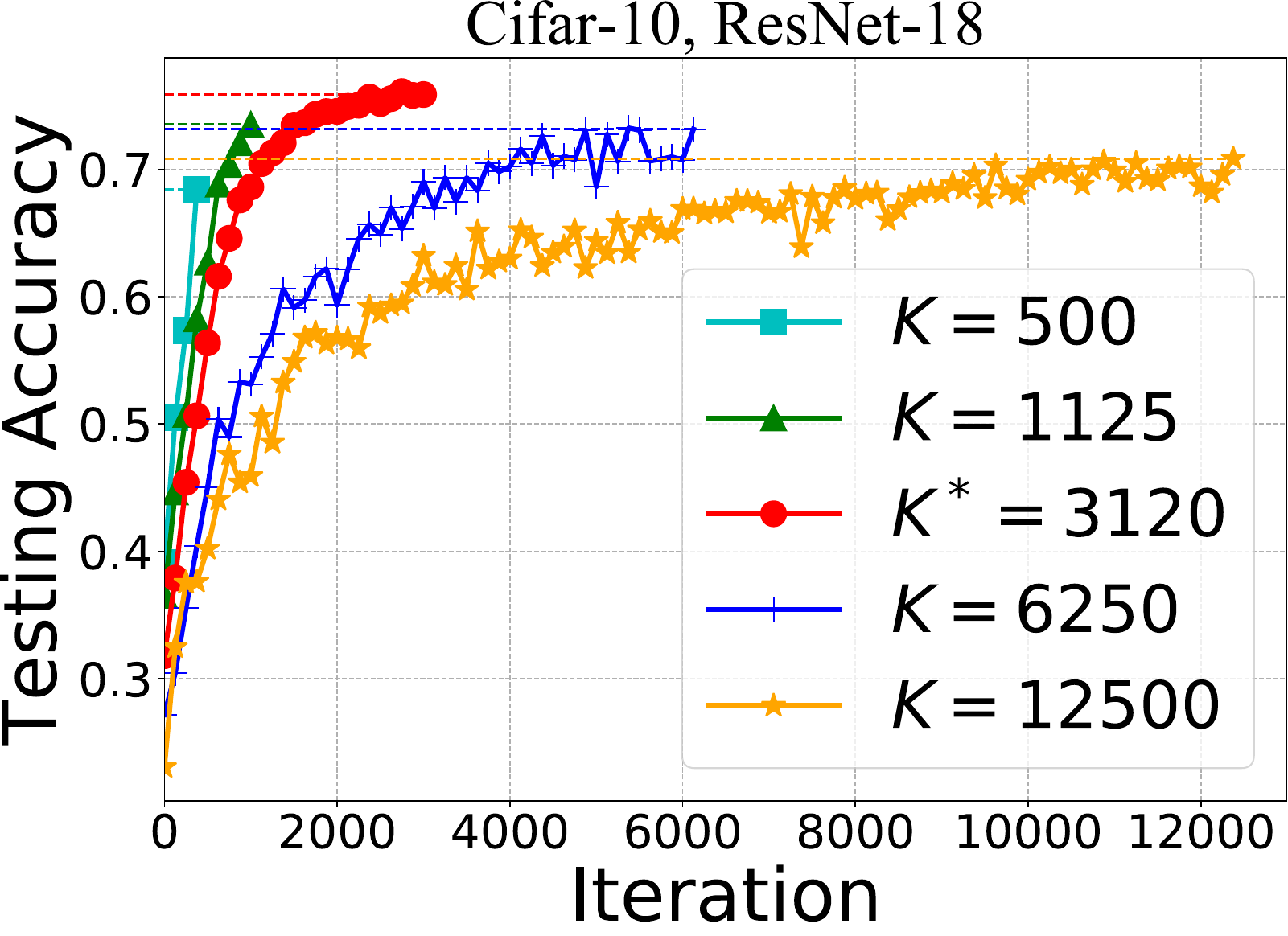}
    \label{acc_different_K}
  }
  \caption{Comparison of convergence performance for PrivSGP-VR over 16 nodes by setting different total number of iterations $K$ under a certain privacy budget, when training ResNet-18 on Cifar-10.}
  \label{different total iteration}
\vspace{-10pt}
\end{figure}

\paragraph{Optimizing number of iterations under certain privacy budget.}
We investigate the significance of selecting an appropriate total number of iterations $K$ for our proposed PrivSGP-VR given a specific privacy budget. To demonstrate this, we conduct experiments using PrivSGP-VR on a network consisting of 16 nodes. For each node $i$, we set the privacy budget to $\epsilon_i=3$ and $\delta_i=10^{-5}$. By varying the value of $K$, we aim to highlight the impact of this parameter on the overall performance of PrivSGP-VR.
Using~\eqref{value_of_T_saga}, we determine the optimal value of $K$ to be approximately 3120. In addition to this optimal choice, we also consider other values of $K$ for comparison: 500, 1125, 6250, and 12500.
For each chosen value of $K$, to guarantee the given privacy budget, we add DP Gaussian noise with variance $\sigma_i^2$ calculated according to~\eqref{noise_scale_saga}.
The results illustrated in Figure~\ref{different total iteration} demonstrate that the total number of iterations $K$ has a substantial impact on both training loss and testing accuracy.
It is evident that selecting the proper value of $K=3120$, as determined by our proposed approach, leads to the minimized loss and maximized accuracy. On the other hand, if a larger value of $K$ (e.g., 12500) or a smaller value (e.g., 500) is chosen, the training loss becomes larger and the model testing accuracy is lower. These findings validate the importance of selecting an appropriate value for $K$ to ensure optimal performance of PrivSGP-VR under a certain privacy budget.

\begin{figure}[!htpb]
\setlength{\abovecaptionskip}{0.2cm}
\centering
\includegraphics[width=0.7\columnwidth]{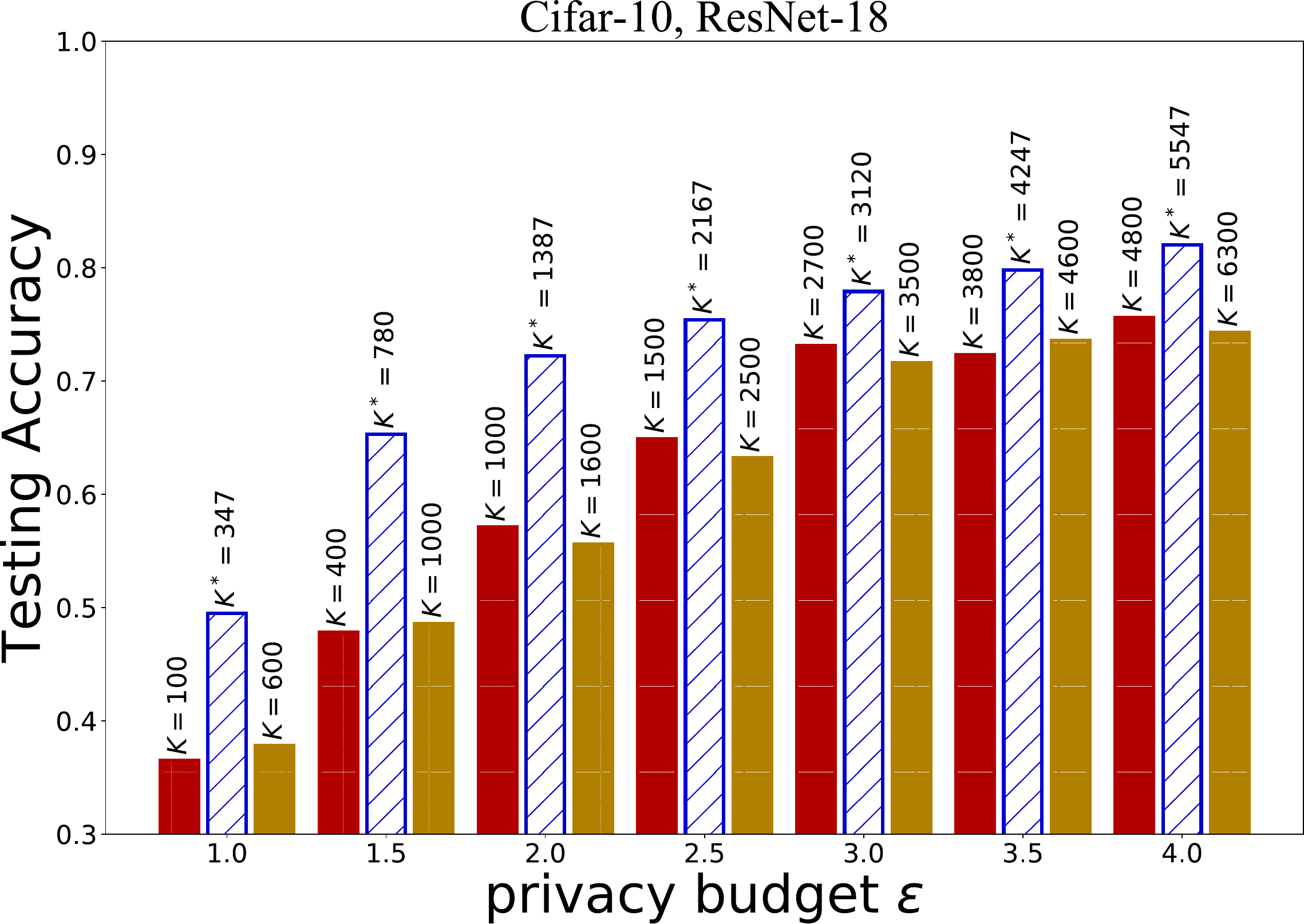} 
\caption{Performance of running PrivSGP-VR for $K^*$ ($K$) iterations under different certain privacy budgets $\epsilon$, when training ResNet-18 on Cifar-10.}
\label{trade_off_utility_privacy}
\end{figure}

\paragraph{Trade off between the maximized model utility and privacy guarantee.}
We conduct experiments by deploying the PrivSGP-VR algorithm on a network consisting of 16 nodes with a fixed value of $\delta=10^{-5}$ for each node. The $\epsilon$ value for each node is varied from the set $\{1, 1.5, 2, 2.5, 3, 3.5, 4\}$.
For each value of $\epsilon$, we determine the optimal total number of iterations $K^*$ using equation~\eqref{value_of_T_saga}. Then, we execute PrivSGP-VR for $K^*$ iterations, along with two other $K$ values for comparative analysis. We incorporate the corresponding DP Gaussian noise with variance calculated according to equation~\eqref{noise_scale_saga}.
Figure~\ref{trade_off_utility_privacy} illustrates the trade-off between model utility (testing accuracy) and privacy under the optimized number of iterations. As the privacy budget $\epsilon$ diminishes (indicating a higher level of privacy protection), the maximized model utility deteriorates. This trade-off between privacy and maximized utility aligns with the theoretical insights outlined in Remark~\ref{remark_for_trade_off}.

\begin{figure}[!htpb]
\vspace{-10pt}
\setlength{\abovecaptionskip}{-0cm}
  \centering
  \subfigure[Training loss]{
    \includegraphics[width=0.45\linewidth]{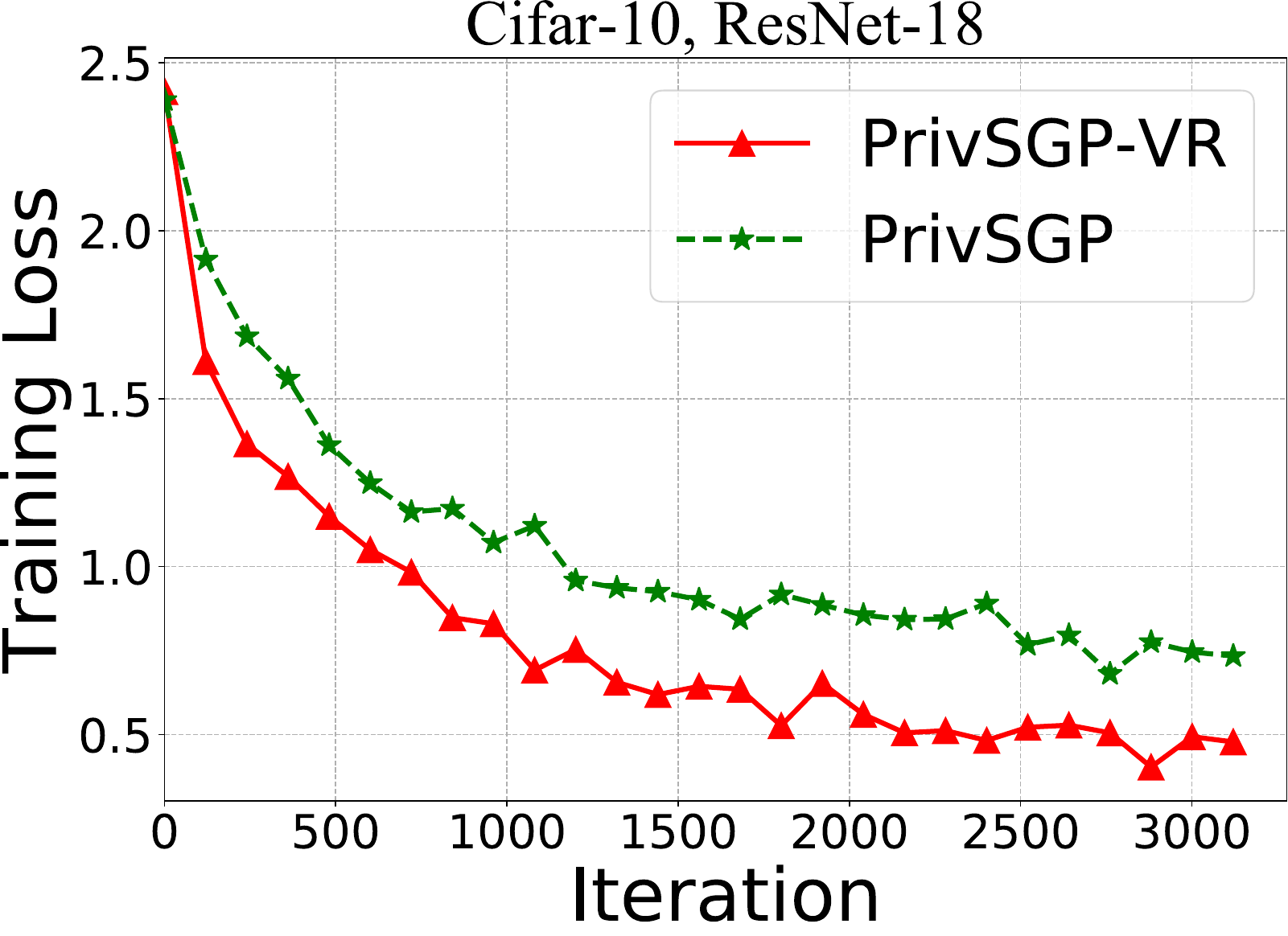}
    \label{loss_VR}
  }
  \hfill
  \subfigure[Testing accuracy]{
    \includegraphics[width=0.45\linewidth]{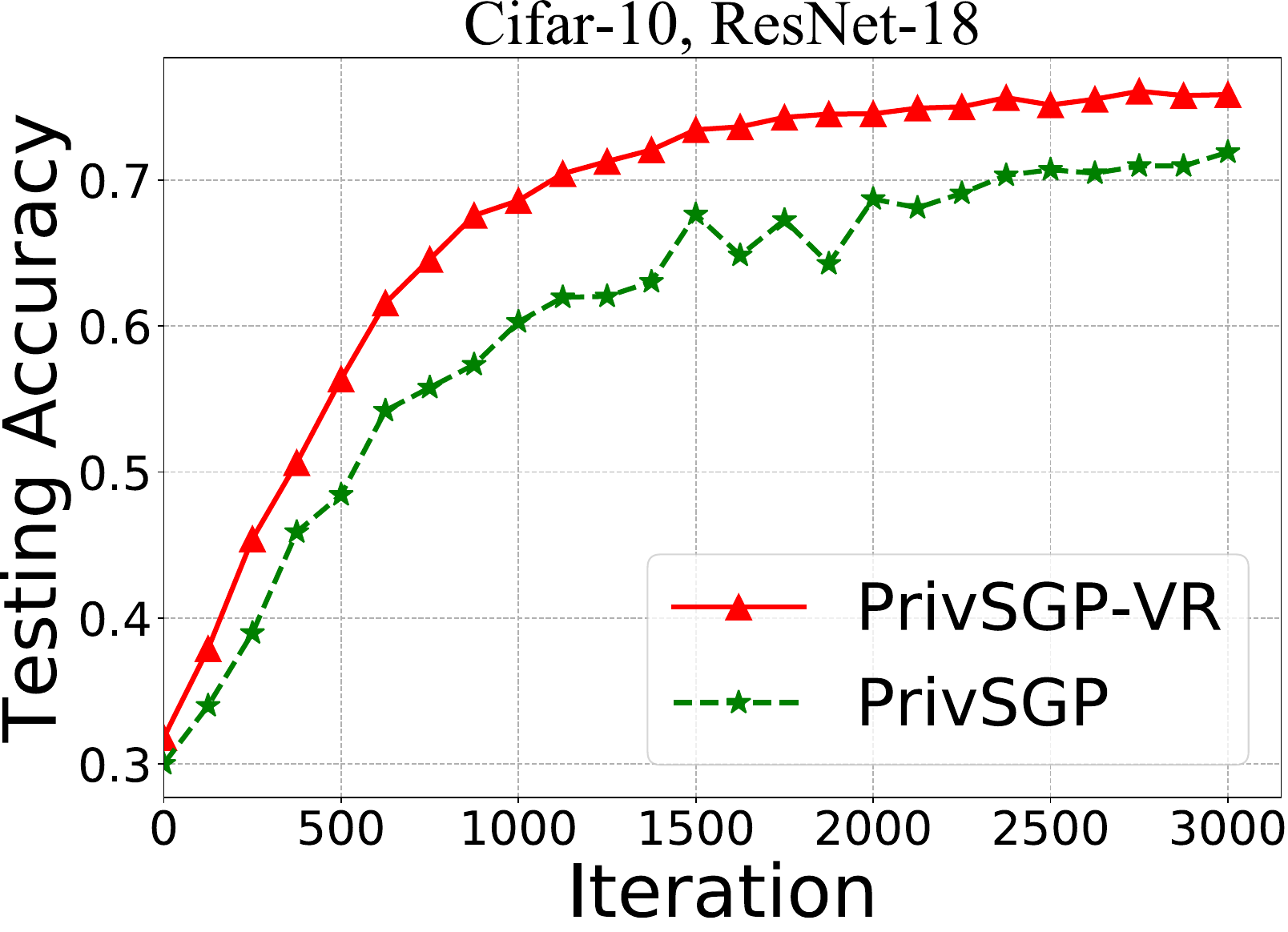}
    \label{acc_VR}
  }
  \caption{Comparison of convergence performance for PrivSGP-VR with PrivSGP over 16 nodes under the same DP Gaussian noise variance, when training ResNet-18 on Cifar-10.}
  \label{variance reduction}
\vspace{-10pt}
\end{figure}

\paragraph{Verifying the effectiveness of variance reduction technique.}
To validate the effectiveness of the variance reduction technique employed by PrivSGP-VR, we conducted experiments to compare PrivSGP-VR with PrivSGP (Algorithm~\ref{PrivSGP}, without the variance reduction technique). For fair comparisons, we applied DP Gaussian noise with an identical variance of $\sigma_i^2=0.03$ for both PrivSGP-VR and PrivSGP. Moreover, both algorithms were executed for a fixed number of 3000 iterations. The results, as depicted in Figure~\ref{variance reduction}, clearly illustrate that PrivSGP-VR outperforms PrivSGP in terms of both training loss and model testing accuracy. 
This validates the effectiveness of the variance reduction technique integrated into PrivSGP-VR.

\begin{figure}[!htpb]
\vspace{-10pt}
\setlength{\abovecaptionskip}{-0cm}
  \centering
  \subfigure[Training loss]{
    \includegraphics[width=0.45\linewidth]{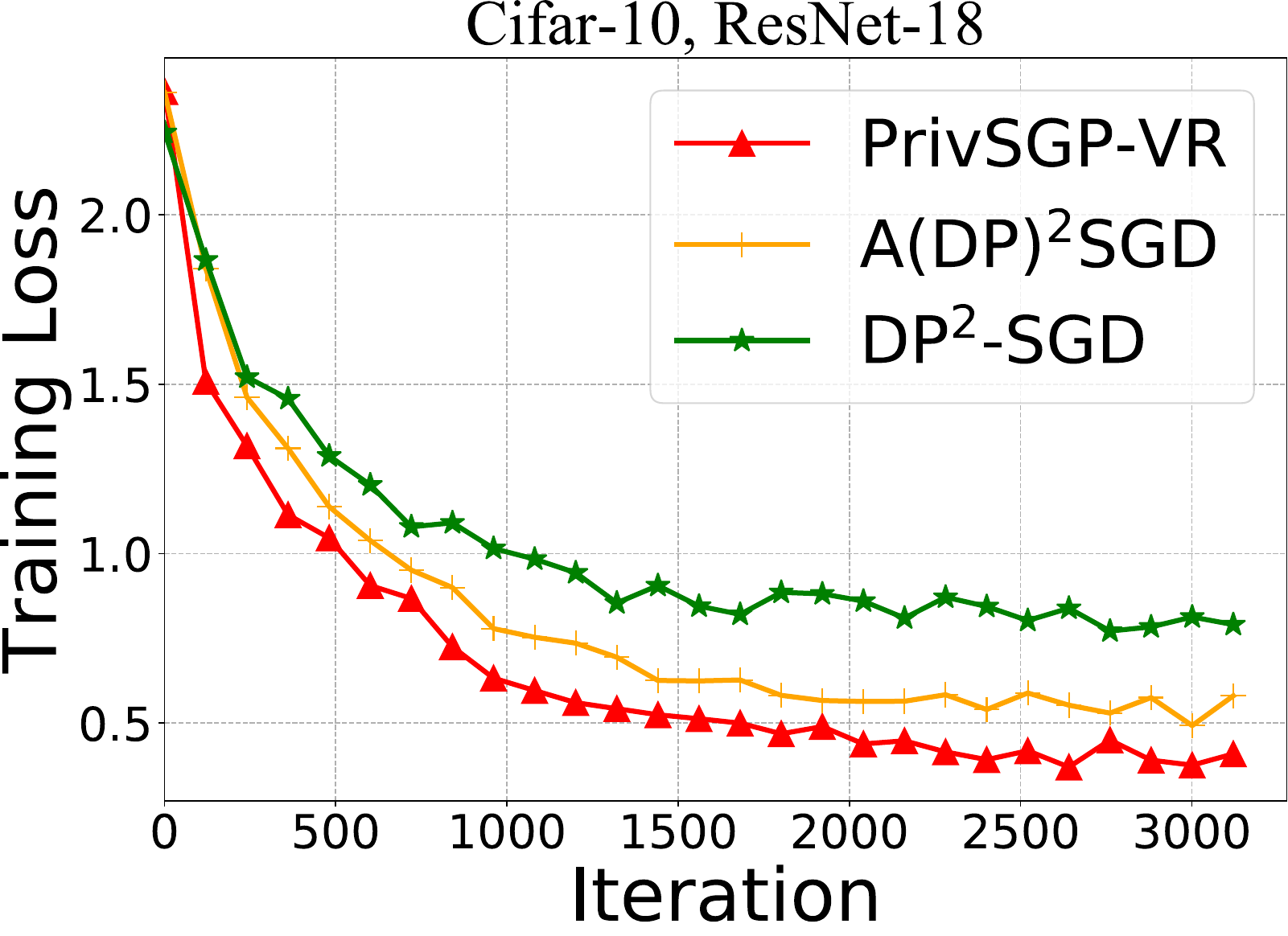}
    \label{loss_compare_alg}
  }
  \hfill
  \subfigure[Testing accuracy]{
    \includegraphics[width=0.45\linewidth]{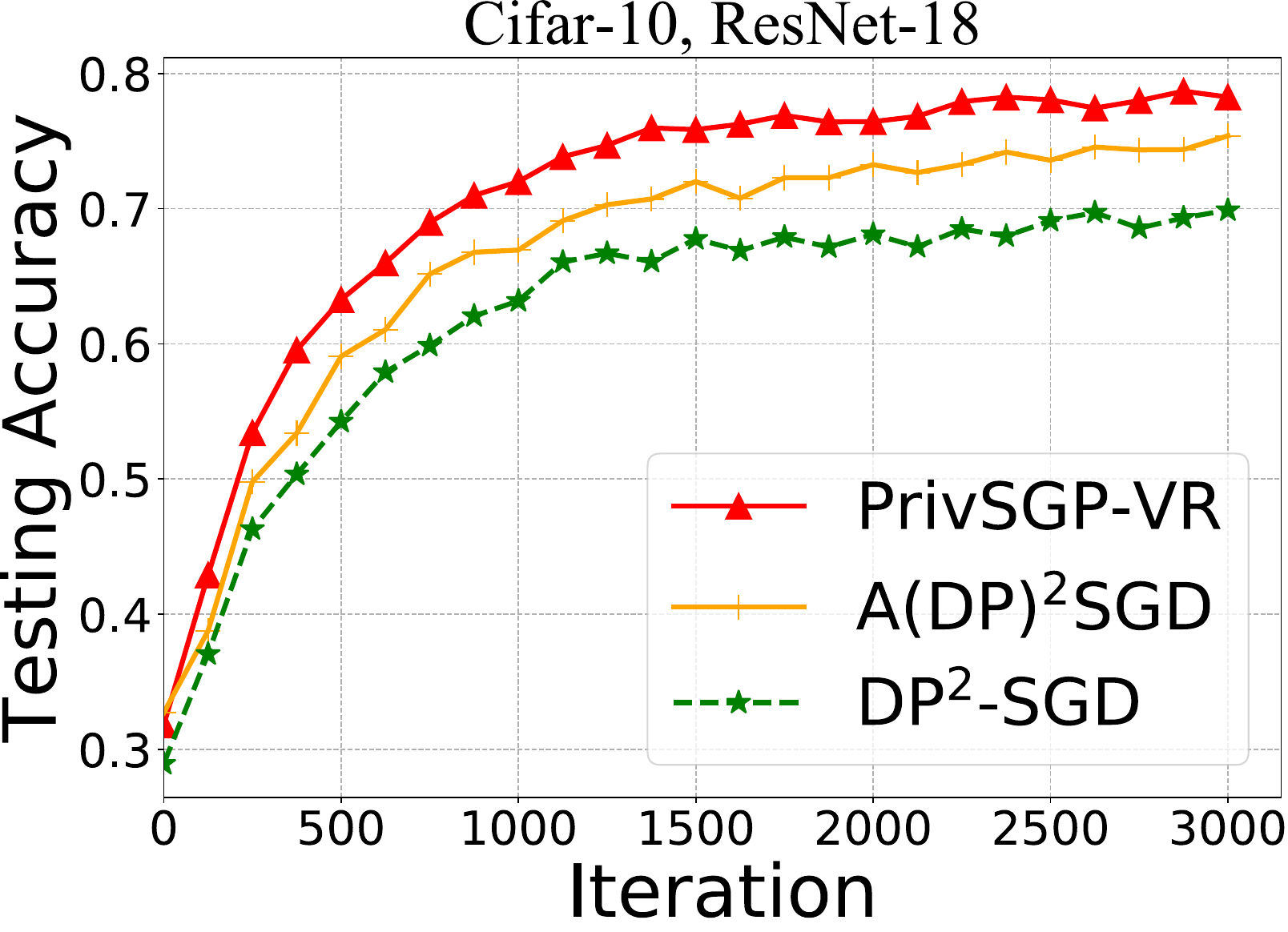}
    \label{acc_compare_alg}
  }
  \caption{Comparison of convergence performance for PrivSGP-VR with DP$^2$-SGD and A(DP)$^2$SGD over 16 nodes with $(3,10^{-5})$-DP guarantee for each node, when training ResNet-18 on Cifar-10.}
  \label{compare_methods}
\vspace{-10pt}
\end{figure}

\paragraph{Comparison with existing decentralized counterparts.}
We present experiments to compare the performance of PrivSGP-VR with other fully decentralized private stochastic algorithms DP$^2$-SGD and A(DP)$^2$SGD.
We implement all three algorithms on an undirected ring graph consisting of 16 nodes. 
The results shown in Figure~\ref{compare_methods} demonstrate that, under $(3,10^{-5})$-DP guarantee for each node, our PrivSGP-VR outperforms DP$^2$-SGD and A(DP)$^2$SGD in that PrivSGP-VR converges faster than the other two algorithms in both training loss and testing accuracy.

\subsection{Shallow 2-layer neural network training}
We also provide additional experimental results for training 2-layer neural network on Mnist dataset which can be found in Appendix~\ref{extra_experiments}, and the experimental results under various settings are aligned with that of training ResNet-18 on Cifar-10 dataset.

\section{Conclusion}
We have proposed a differentially private decentralized learning method over time-varying directed communication topologies, termed PrivSGP-VR.
Our analysis shows that under DP Gaussian noise with constant variance, PrivSGP-VR converges at a sub-linear rate $\mathcal{O}(1/\sqrt{nK})$ which is independent of stochastic gradient variance.
When given a certain privacy budget for each node, leveraging the moments accountant method, we derive an optimal number of iterations $K$ to maximize the model utility.
With this optimized $K$, we achieve a tight utility bound which matches that of the server-client distributed counterparts, and exhibits an extra factor of $1/\sqrt{n}$ improvement compared to that of the existing decentralized counterparts. Extensive experiments are conducted to validate our theoretical findings.

\section*{Acknowledgments}
This work is supported in parts by the National
Key R\&D Program of China under Grant No. 2022YFB3102100, and in parts by National Natural Science Foundation of China under Grants 62373323, 62088101, 62003302.


\bibliographystyle{named}
\bibliography{ijcai24}

\begin{thebibliography}{}

\bibitem[\protect\citeauthoryear{Abadi \bgroup \em et al.\egroup }{2016}]{abadi2016deep}
Martin Abadi, Andy Chu, Ian Goodfellow, H~Brendan McMahan, Ilya Mironov, Kunal Talwar, and Li~Zhang.
\newblock Deep learning with differential privacy.
\newblock In {\em Proceedings of the 2016 ACM SIGSAC conference on computer and communications security}, pages 308--318, 2016.

\bibitem[\protect\citeauthoryear{Assran \bgroup \em et al.\egroup }{2019}]{assran2019stochastic}
Mahmoud Assran, Nicolas Loizou, Nicolas Ballas, and Mike Rabbat.
\newblock Stochastic gradient push for distributed deep learning.
\newblock In {\em International Conference on Machine Learning}, pages 344--353. PMLR, 2019.

\bibitem[\protect\citeauthoryear{Bun and Steinke}{2016}]{bun2016concentrated}
Mark Bun and Thomas Steinke.
\newblock Concentrated differential privacy: Simplifications, extensions, and lower bounds.
\newblock In {\em Theory of Cryptography: 14th International Conference, TCC 2016-B, Beijing, China, October 31-November 3, 2016, Proceedings, Part I}, pages 635--658. Springer, 2016.

\bibitem[\protect\citeauthoryear{Carlini \bgroup \em et al.\egroup }{2019}]{carlini2019secret}
Nicholas Carlini, Chang Liu, {\'U}lfar Erlingsson, Jernej Kos, and Dawn Song.
\newblock The secret sharer: Evaluating and testing unintended memorization in neural networks.
\newblock In {\em USENIX Security Symposium}, volume 267, 2019.

\bibitem[\protect\citeauthoryear{Chen \bgroup \em et al.\egroup }{2020}]{chen2020understanding}
Xiangyi Chen, Steven~Z Wu, and Mingyi Hong.
\newblock Understanding gradient clipping in private sgd: A geometric perspective.
\newblock {\em Advances in Neural Information Processing Systems}, 33:13773--13782, 2020.

\bibitem[\protect\citeauthoryear{Cheng \bgroup \em et al.\egroup }{2018}]{cheng2018leasgd}
Hsin-Pai Cheng, Patrick Yu, Haojing Hu, Feng Yan, Shiyu Li, Hai Li, and Yiran Chen.
\newblock Leasgd: an efficient and privacy-preserving decentralized algorithm for distributed learning.
\newblock {\em arXiv preprint arXiv:1811.11124}, 2018.

\bibitem[\protect\citeauthoryear{Cheng \bgroup \em et al.\egroup }{2019}]{cheng2019towards}
Hsin-Pai Cheng, Patrick Yu, Haojing Hu, Syed Zawad, Feng Yan, Shiyu Li, Hai Li, and Yiran Chen.
\newblock Towards decentralized deep learning with differential privacy.
\newblock In {\em International Conference on Cloud Computing}, pages 130--145. Springer, 2019.

\bibitem[\protect\citeauthoryear{Defazio \bgroup \em et al.\egroup }{2014}]{defazio2014saga}
Aaron Defazio, Francis Bach, and Simon Lacoste-Julien.
\newblock Saga: A fast incremental gradient method with support for non-strongly convex composite objectives.
\newblock {\em Advances in neural information processing systems}, 27, 2014.

\bibitem[\protect\citeauthoryear{Deng}{2012}]{deng2012mnist}
Li~Deng.
\newblock The mnist database of handwritten digit images for machine learning research [best of the web].
\newblock {\em IEEE signal processing magazine}, 29(6):141--142, 2012.

\bibitem[\protect\citeauthoryear{Dwork and Lei}{2009}]{dwork2009differential}
Cynthia Dwork and Jing Lei.
\newblock Differential privacy and robust statistics.
\newblock In {\em Proceedings of the forty-first annual ACM symposium on Theory of computing}, pages 371--380, 2009.

\bibitem[\protect\citeauthoryear{Dwork \bgroup \em et al.\egroup }{2006}]{dwork2006our}
Cynthia Dwork, Krishnaram Kenthapadi, Frank McSherry, Ilya Mironov, and Moni Naor.
\newblock Our data, ourselves: Privacy via distributed noise generation.
\newblock In {\em Annual international conference on the theory and applications of cryptographic techniques}, pages 486--503. Springer, 2006.

\bibitem[\protect\citeauthoryear{Dwork \bgroup \em et al.\egroup }{2010}]{dwork2010boosting}
Cynthia Dwork, Guy~N Rothblum, and Salil Vadhan.
\newblock Boosting and differential privacy.
\newblock In {\em 2010 IEEE 51st Annual Symposium on Foundations of Computer Science}, pages 51--60. IEEE, 2010.

\bibitem[\protect\citeauthoryear{Dwork \bgroup \em et al.\egroup }{2014}]{dwork2014algorithmic}
Cynthia Dwork, Aaron Roth, et~al.
\newblock The algorithmic foundations of differential privacy.
\newblock {\em Foundations and Trends{\textregistered} in Theoretical Computer Science}, 9(3--4):211--407, 2014.

\bibitem[\protect\citeauthoryear{He \bgroup \em et al.\egroup }{2016}]{he2016deep}
Kaiming He, Xiangyu Zhang, Shaoqing Ren, and Jian Sun.
\newblock Deep residual learning for image recognition.
\newblock In {\em Proceedings of the IEEE conference on computer vision and pattern recognition}, pages 770--778, 2016.

\bibitem[\protect\citeauthoryear{Kempe \bgroup \em et al.\egroup }{2003}]{kempe2003gossip}
David Kempe, Alin Dobra, and Johannes Gehrke.
\newblock Gossip-based computation of aggregate information.
\newblock In {\em 44th Annual IEEE Symposium on Foundations of Computer Science, 2003. Proceedings.}, pages 482--491. IEEE, 2003.

\bibitem[\protect\citeauthoryear{Krizhevsky \bgroup \em et al.\egroup }{2009}]{krizhevsky2009learning}
Alex Krizhevsky, Geoffrey Hinton, et~al.
\newblock Learning multiple layers of features from tiny images.
\newblock 2009.

\bibitem[\protect\citeauthoryear{Langer \bgroup \em et al.\egroup }{2020}]{langer2020distributed}
Matthias Langer, Zhen He, Wenny Rahayu, and Yanbo Xue.
\newblock Distributed training of deep learning models: A taxonomic perspective.
\newblock {\em IEEE Transactions on Parallel and Distributed Systems}, 31(12):2802--2818, 2020.

\bibitem[\protect\citeauthoryear{Li \bgroup \em et al.\egroup }{2014}]{li2014scaling}
Mu~Li, David~G Andersen, Jun~Woo Park, Alexander~J Smola, Amr Ahmed, Vanja Josifovski, James Long, Eugene~J Shekita, and Bor-Yiing Su.
\newblock Scaling distributed machine learning with the parameter server.
\newblock In {\em 11th $\{$USENIX$\}$ Symposium on Operating Systems Design and Implementation ($\{$OSDI$\}$ 14)}, pages 583--598, 2014.

\bibitem[\protect\citeauthoryear{Li \bgroup \em et al.\egroup }{2019}]{li2019asynchronous}
Yanan Li, Shusen Yang, Xuebin Ren, and Cong Zhao.
\newblock Asynchronous federated learning with differential privacy for edge intelligence.
\newblock {\em arXiv preprint arXiv:1912.07902}, 2019.

\bibitem[\protect\citeauthoryear{Li \bgroup \em et al.\egroup }{2022}]{li2022soteriafl}
Zhize Li, Haoyu Zhao, Boyue Li, and Yuejie Chi.
\newblock Soteriafl: A unified framework for private federated learning with communication compression.
\newblock {\em Advances in Neural Information Processing Systems}, 35:4285--4300, 2022.

\bibitem[\protect\citeauthoryear{Lian \bgroup \em et al.\egroup }{2017}]{lian2017can}
Xiangru Lian, Ce~Zhang, Huan Zhang, Cho-Jui Hsieh, Wei Zhang, and Ji~Liu.
\newblock Can decentralized algorithms outperform centralized algorithms? a case study for decentralized parallel stochastic gradient descent.
\newblock {\em Advances in Neural Information Processing Systems}, 30, 2017.

\bibitem[\protect\citeauthoryear{Lian \bgroup \em et al.\egroup }{2018}]{lian2018asynchronous}
Xiangru Lian, Wei Zhang, Ce~Zhang, and Ji~Liu.
\newblock Asynchronous decentralized parallel stochastic gradient descent.
\newblock In {\em International Conference on Machine Learning}, pages 3043--3052. PMLR, 2018.

\bibitem[\protect\citeauthoryear{Liu \bgroup \em et al.\egroup }{2022}]{liu2022loss}
Tianyu Liu, Boya Di, Bin Wang, and Lingyang Song.
\newblock Loss-privacy tradeoff in federated edge learning.
\newblock {\em IEEE Journal of Selected Topics in Signal Processing}, 16(3):546--558, 2022.

\bibitem[\protect\citeauthoryear{Lowy \bgroup \em et al.\egroup }{2023}]{lowy2023private}
Andrew Lowy, Ali Ghafelebashi, and Meisam Razaviyayn.
\newblock Private non-convex federated learning without a trusted server, 2023.

\bibitem[\protect\citeauthoryear{Luo \bgroup \em et al.\egroup }{2021}]{luo2021cost}
Bing Luo, Xiang Li, Shiqiang Wang, Jianwei Huang, and Leandros Tassiulas.
\newblock Cost-effective federated learning design.
\newblock In {\em IEEE INFOCOM 2021-IEEE Conference on Computer Communications}, pages 1--10. IEEE, 2021.

\bibitem[\protect\citeauthoryear{McMahan \bgroup \em et al.\egroup }{2017}]{mcmahan2017communication}
Brendan McMahan, Eider Moore, Daniel Ramage, Seth Hampson, and Blaise~Aguera y~Arcas.
\newblock Communication-efficient learning of deep networks from decentralized data.
\newblock In {\em Artificial intelligence and statistics}, pages 1273--1282. PMLR, 2017.

\bibitem[\protect\citeauthoryear{Mironov}{2017}]{mironov2017renyi}
Ilya Mironov.
\newblock R{\'e}nyi differential privacy.
\newblock In {\em 2017 IEEE 30th computer security foundations symposium (CSF)}, pages 263--275. IEEE, 2017.

\bibitem[\protect\citeauthoryear{Truex \bgroup \em et al.\egroup }{2019}]{truex2019hybrid}
Stacey Truex, Nathalie Baracaldo, Ali Anwar, Thomas Steinke, Heiko Ludwig, Rui Zhang, and Yi~Zhou.
\newblock A hybrid approach to privacy-preserving federated learning.
\newblock In {\em Proceedings of the 12th ACM workshop on artificial intelligence and security}, pages 1--11, 2019.

\bibitem[\protect\citeauthoryear{Truex \bgroup \em et al.\egroup }{2020}]{truex2020ldp}
Stacey Truex, Ling Liu, Ka-Ho Chow, Mehmet~Emre Gursoy, and Wenqi Wei.
\newblock Ldp-fed: Federated learning with local differential privacy.
\newblock In {\em Proceedings of the Third ACM International Workshop on Edge Systems, Analytics and Networking}, pages 61--66, 2020.

\bibitem[\protect\citeauthoryear{Wang and Nedic}{2022}]{wang2022tailoring}
Yongqiang Wang and Angelia Nedic.
\newblock Tailoring gradient methods for differentially-private distributed optimization.
\newblock {\em arXiv preprint arXiv:2202.01113}, 2022.

\bibitem[\protect\citeauthoryear{Wang \bgroup \em et al.\egroup }{2017}]{wang2017differentially}
Di~Wang, Minwei Ye, and Jinhui Xu.
\newblock Differentially private empirical risk minimization revisited: Faster and more general.
\newblock {\em Advances in Neural Information Processing Systems}, 30, 2017.

\bibitem[\protect\citeauthoryear{Wang \bgroup \em et al.\egroup }{2019a}]{wang2019efficient}
Lingxiao Wang, Bargav Jayaraman, David Evans, and Quanquan Gu.
\newblock Efficient privacy-preserving stochastic nonconvex optimization.
\newblock {\em arXiv e-prints}, pages arXiv--1910, 2019.

\bibitem[\protect\citeauthoryear{Wang \bgroup \em et al.\egroup }{2019b}]{wang2019adaptive}
Shiqiang Wang, Tiffany Tuor, Theodoros Salonidis, Kin~K Leung, Christian Makaya, Ting He, and Kevin Chan.
\newblock Adaptive federated learning in resource constrained edge computing systems.
\newblock {\em IEEE journal on selected areas in communications}, 37(6):1205--1221, 2019.

\bibitem[\protect\citeauthoryear{Wang \bgroup \em et al.\egroup }{2019c}]{wang2019beyond}
Zhibo Wang, Mengkai Song, Zhifei Zhang, Yang Song, Qian Wang, and Hairong Qi.
\newblock Beyond inferring class representatives: User-level privacy leakage from federated learning.
\newblock In {\em IEEE INFOCOM 2019-IEEE conference on computer communications}, pages 2512--2520. IEEE, 2019.

\bibitem[\protect\citeauthoryear{Wang \bgroup \em et al.\egroup }{2020}]{wang2020differentially}
Di~Wang, Hanshen Xiao, Srinivas Devadas, and Jinhui Xu.
\newblock On differentially private stochastic convex optimization with heavy-tailed data.
\newblock In {\em International Conference on Machine Learning}, pages 10081--10091. PMLR, 2020.

\bibitem[\protect\citeauthoryear{Wang \bgroup \em et al.\egroup }{2023}]{wang2023efficient}
Lingxiao Wang, Bargav Jayaraman, David Evans, and Quanquan Gu.
\newblock Efficient privacy-preserving stochastic nonconvex optimization.
\newblock In {\em Uncertainty in Artificial Intelligence}, pages 2203--2213. PMLR, 2023.

\bibitem[\protect\citeauthoryear{Wei \bgroup \em et al.\egroup }{2020}]{wei2020federated}
Kang Wei, Jun Li, Ming Ding, Chuan Ma, Howard~H Yang, Farhad Farokhi, Shi Jin, Tony~QS Quek, and H~Vincent Poor.
\newblock Federated learning with differential privacy: Algorithms and performance analysis.
\newblock {\em IEEE Transactions on Information Forensics and Security}, 15:3454--3469, 2020.

\bibitem[\protect\citeauthoryear{Wei \bgroup \em et al.\egroup }{2021}]{wei2021user}
Kang Wei, Jun Li, Ming Ding, Chuan Ma, Hang Su, Bo~Zhang, and H~Vincent Poor.
\newblock User-level privacy-preserving federated learning: Analysis and performance optimization.
\newblock {\em IEEE Transactions on Mobile Computing}, 21(9):3388--3401, 2021.

\bibitem[\protect\citeauthoryear{Wu \bgroup \em et al.\egroup }{2020}]{wu2020value}
Nan Wu, Farhad Farokhi, David Smith, and Mohamed~Ali Kaafar.
\newblock The value of collaboration in convex machine learning with differential privacy.
\newblock In {\em 2020 IEEE Symposium on Security and Privacy (SP)}, pages 304--317. IEEE, 2020.

\bibitem[\protect\citeauthoryear{Xu \bgroup \em et al.\egroup }{2021}]{xu2021dp}
Jie Xu, Wei Zhang, and Fei Wang.
\newblock A (dp)\^{} 2sgd: Asynchronous decentralized parallel stochastic gradient descent with differential privacy.
\newblock {\em IEEE transactions on pattern analysis and machine intelligence}, 2021.

\bibitem[\protect\citeauthoryear{Yu \bgroup \em et al.\egroup }{2021}]{yu2021decentralized}
Dongxiao Yu, Zongrui Zou, Shuzhen Chen, Youming Tao, Bing Tian, Weifeng Lv, and Xiuzhen Cheng.
\newblock Decentralized parallel sgd with privacy preservation in vehicular networks.
\newblock {\em IEEE Transactions on Vehicular Technology}, 70(6):5211--5220, 2021.

\bibitem[\protect\citeauthoryear{Zeng \bgroup \em et al.\egroup }{2021}]{zeng2021differentially}
Yiming Zeng, Yixuan Lin, Yuanyuan Yang, and Ji~Liu.
\newblock Differentially private federated temporal difference learning.
\newblock {\em IEEE Transactions on Parallel \& Distributed Systems}, (01):1--1, 2021.

\bibitem[\protect\citeauthoryear{Zhang \bgroup \em et al.\egroup }{2020}]{zhang2020private}
Xin Zhang, Minghong Fang, Jia Liu, and Zhengyuan Zhu.
\newblock Private and communication-efficient edge learning: a sparse differential gaussian-masking distributed sgd approach.
\newblock In {\em Proceedings of the Twenty-First International Symposium on Theory, Algorithmic Foundations, and Protocol Design for Mobile Networks and Mobile Computing}, pages 261--270, 2020.

\bibitem[\protect\citeauthoryear{Zhou \bgroup \em et al.\egroup }{2023}]{zhou2023optimizing}
Yipeng Zhou, Xuezheng Liu, Yao Fu, Di~Wu, Jessie~Hui Wang, and Shui Yu.
\newblock Optimizing the numbers of queries and replies in convex federated learning with differential privacy.
\newblock {\em IEEE Transactions on Dependable and Secure Computing}, 2023.

\end{thebibliography}

\onecolumn
\appendix
\begin{center}
\LARGE{\textbf{Appendix}}
\end{center}
{\footnotesize
\tableofcontents
}

\section{Proof of Main Result}
\label{proof_of_theorem}
To facilitate our  analysis, we first rewrite the $12^{th}$ step of the proposed PrivSGP-VR (c.f., Algorithm~\ref{PrivSGP-VR}) in a compact form: 
\begin{equation}
\label{iterate_saga}
X^{k+1}=\left( X^k-\gamma \left( G^k+N^k \right) \right) \left( P^k \right) ^{\top},
\end{equation}
where $\left( P^k \right) ^{\top}\in \mathbb{R}^{n\times n}
$ is the transpose of the mixing matrix $P^k$ at iteration $k$, and

$X^{k}:=\left[ x_{1}^{k},x_{2}^{k},\cdot \cdot \cdot ,x_{n}^{k} \right] \in \mathbb{R}^{d\times n}
$: the collection of all nodes' parameters at iteration $k$;


$G^k:=\left[ g_{1}^{k},g_{1}^{k},\cdot \cdot \cdot ,g_{n}^{k} \right] \in \mathbb{R}^{d\times n}$: the collection of all nodes' corrected gradients at iteration $k$;

$N^k:=\left[ N_{1}^{k},N_{2}^{k},\cdot \cdot \cdot ,N_{n}^{k} \right] \in \mathbb{R}^{d\times n}$: the collection of all nodes' added Gaussian noises at iteration $k$.

Now, let $\bar{x}^k=\frac{1}{n}X^k\mathbf{1}=\frac{1}{n}\sum_{i=1}^n{x_{i}^{k}}\in \mathbb{R}^d
$ denote the average of all nodes' parameters at iteration $k$. Then, the update of average system of~\eqref{iterate_saga} becomes
\begin{equation}
\label{average_system_saga}
\bar{x}^{k+1}=\bar{x}^k-\gamma \cdot \left( \frac{1}{n}\sum_{i=1}^n{g_{i}^{k}}+\frac{1}{n}\sum_{i=1}^n{N_{i}^{k}} \right) ,
\end{equation}
which can be easily obtained by right multiplying $\frac{1}{n}\mathbf{1}$ from both sides of~\eqref{iterate_saga}
and using the column-stochastic property of $P^k$ (c.f., Assumption~\ref{Ass_weight_matrix}). The above average system will be useful in subsequent analysis.

In addition, we denote by $\mathcal{F}^k$ the history sequence $\left\{ \bigcup\nolimits_{i=1}^n{\left( x_{i}^{0},z_{i}^{0},\xi _{i}^{0},N_{i}^{0},\cdot \cdot \cdot ,x_{i}^{k-1},z_{i}^{k-1},\xi _{i}^{k-1},N_{i}^{k-1},x_{i}^{k},z_{i}^{k} \right)} \right\} 
$, and define $\mathbb{E}\left[ \cdot \left| \mathcal{F}^k \right. \right] 
$ as the conditional expectation given $\mathcal{F}^k$.

\subsection{Important Upper Bounds}\label{important_upper_bounds}
In this section, we first provide several technical lemmas  to facilitate the subsequent analysis.

\begin{Lem}
\label{supporting_lemma_2}
Let $\left\{ v^k \right\} _{k=0}^{\infty}$ be a non-negative sequence and $\lambda \in (0,1)$. Then, we have
\begin{equation}
    \left( \sum_{l=0}^k{\lambda ^{k-l}v^l} \right) ^2\leqslant \frac{1}{1-\lambda}\sum_{l=0}^k{\lambda ^{k-l}\left( v^l \right) ^2}.
\end{equation}
\end{Lem}

\begin{proof}
Using Cauchy-Swarchz inequality, we have
\begin{equation*}
\begin{aligned}
\left( \sum_{l=0}^k{\lambda ^{k-l}v^l} \right) ^2&=\left( \sum_{l=0}^k{\lambda ^{\frac{k-l}{2}}\left( \lambda ^{\frac{k-l}{2}}v^l \right)} \right) ^2
\\
&\leqslant \sum_{l=0}^k{\left( \lambda ^{\frac{k-l}{2}} \right) ^2}\cdot \sum_{l=0}^k{\left( \lambda ^{\frac{k-l}{2}}v^l \right) ^2}
\\
&\leqslant \frac{1}{1-\lambda}\sum_{l=0}^k{\lambda ^{k-l}\left( v^l \right) ^2},
\end{aligned}
\end{equation*}
which completes the proof.
\end{proof}

\begin{Lem}[Unbiased Estimate of Gradient]
\label{ubias}
Suppose Assumption~\ref{assumption_unbiased_gradient} hold. Then, we have
\begin{equation}
\label{unbiased_corrected_gradient}
\mathbb{E}\left[ g_{i}^{k}\left| \mathcal{F}^k \right. \right] =\nabla f_i\left( z_{i}^{k} \right) 
,
\end{equation}
where $g_i^k$ is the corrected stochastic gradient of node $i$ at iteration $k$.
\end{Lem}
\begin{proof}
According to the definition of $g_i^k$ in~\eqref{corrected_stochastic_gradient}, we have
\begin{align*}
\mathbb{E}\left[ g_{i}^{k}\left| \mathcal{F}^k \right. \right] & =\mathbb{E}\left[ \nabla f_i\left( z_{i}^{k};\xi _{i}^{k} \right) -\nabla f_i\left( \phi _{i,\xi _{i}^{k}}^{k};\xi _{i}^{k} \right) +\frac{1}{J}\sum_{j=1}^J{\nabla f_i\left( \phi _{i,j}^{k};j \right)}\left| \mathcal{F}^k \right. \right] 
\\
&=\mathbb{E}\left[ \nabla f_i\left( z_{i}^{k};\xi _{i}^{k} \right) \left| \mathcal{F}^k \right. \right] -\mathbb{E}\left[ \nabla f_i\left( \phi _{i,\xi _{i}^{k}}^{k};\xi _{i}^{k} \right) \left| \mathcal{F}^k \right. \right] +\frac{1}{J}\sum_{j=1}^J{\nabla f_i\left( \phi _{i,j}^{k};j \right)}
\\
&\overset{\left( a \right)}{=}\nabla f_i\left( z_i^k \right) -\frac{1}{J}\sum_{j=1}^J{\nabla f_i\left( \phi _{i,j}^{k};j \right)}+\frac{1}{J}\sum_{j=1}^J{\nabla f_i\left( \phi _{i,j}^{k};j \right)}
\\
&=\nabla f_i\left( z_i^k \right) ,
\end{align*}
where in $(a)$ we used Assumption \ref{assumption_unbiased_gradient}.
\end{proof}

\begin{Lem}[Bounded Variance]
\label{bounded}
Suppose Assumptions~\ref{assumption_smooth_saga} and~\ref{assumption_unbiased_gradient} hold. Then, we have
\begin{equation}
\label{upper_bound_of_corrected_gradient}
\mathbb{E}\left[ \left\| g_{i}^{k}-\nabla f_i\left( z_{i}^{k} \right) \right\| ^2\left| \mathcal{F}^k \right. \right] \leqslant L^2\cdot \frac{1}{J}\sum_{j=1}^J{\left\| z_{i}^{k}-\phi _{i,j}^{k} \right\| ^2}
.
\end{equation}
\end{Lem}
\begin{proof}
Using the definition of $g_i^k$ in \eqref{corrected_stochastic_gradient}, we have 
\begin{align*}
\mathbb{E}\left[ \left\| g_{i}^{k}-\nabla f_i\left( z_{i}^{k} \right) \right\| ^2\left| \mathcal{F}^k \right. \right] &=\mathbb{E}\left[ \left\| \nabla f_i\left( z_{i}^{k};\xi _{i}^{k} \right) -\nabla f_i\left( \phi _{i,\xi _{i}^{k}}^{k};\xi _{i}^{k} \right) +\frac{1}{J}\sum_{j=1}^J{\nabla f_i\left( \phi _{i,j}^{k};j \right)}-\nabla f_i\left( z_{i}^{k} \right) \right\| ^2\left| \mathcal{F}^k \right. \right] 
\\
&=\mathbb{E}\left[ \left\| \nabla f_i\left( z_{i}^{k};\xi _{i}^{k} \right) -\nabla f_i\left( \phi _{i,\xi _{i}^{k}}^{k};\xi _{i}^{k} \right) -\left( \nabla f_i\left( z_{i}^{k} \right) -\frac{1}{J}\sum_{j=1}^J{\nabla f_i\left( \phi _{i,j}^{k};j \right)} \right) \right\| ^2\left| \mathcal{F}^k \right. \right] 
\\
&\overset{\left( a \right)}{\leqslant}\mathbb{E}\left[ \left\| \nabla f_i\left( z_{i}^{k};\xi _{i}^{k} \right) -\nabla f_i\left( \phi _{i,\xi _{i}^{k}}^{k};\xi _{i}^{k} \right) \right\| ^2\left| \mathcal{F}^k \right. \right] 
\\
&\overset{\left( b \right)}{\leqslant}L^2\mathbb{E}\left[ \left\| z_{i}^{k}-\phi _{i,\xi _{i}^{k}}^{k} \right\| ^2\left| \mathcal{F}^k \right. \right] 
\\
&\overset{\left( c \right)}{=}
L^2\cdot \frac{1}{J}\sum_{j=1}^J{\left\| z_{i}^{k}-\phi _{i,j}^{k} \right\| ^2}
,
\end{align*}
where in $(a)$ we used the fact that
$\mathbb{E}\left[\left\|x-\mathbb{E}\left[x\right]\right\|^2\right]\leq \mathbb{E}\left[\left\|x\right\|^2\right]
$
and Assumption~\ref{assumption_unbiased_gradient}; $(b)$ is due to Assumption~\ref{assumption_smooth_saga}, and the last equality is due to the property of uniform sampling (c.f., step 5 in Algorithm~\ref{PrivSGP-VR}).
\end{proof}

\begin{Lem}
\label{supporting_lemma_3_saga}
Suppose Assumptions~\ref{assumption_smooth_saga} and~\ref{assumption_bounede_outer_variation} hold. Then, we have
\begin{equation*}
\left\| \nabla f_i\left( z_{i}^{k} \right) \right\| ^2\leqslant 3L^2\left\| z_{i}^{k}-\bar{x}^k \right\| ^2+3b^2+3\left\| \nabla f\left( \bar{x}^k \right) \right\| ^2.
\end{equation*}
\end{Lem}

\begin{proof}
Using Assumption~\ref{assumption_smooth_saga} ($L$-smooth) and Assumption~\ref{assumption_bounede_outer_variation} (bounded data heterogeneity), we have
\begin{equation*}
\begin{aligned}
\left\| \nabla f_i\left( z_{i}^{k} \right) \right\| ^2&=\left\| \nabla f_i\left( z_{i}^{k} \right) -\nabla f_i\left( \bar{x}^k \right) +\nabla f_i\left( \bar{x}^k \right) -\nabla f\left( \bar{x}^k \right) +\nabla f\left( \bar{x}^k \right) \right\| ^2
\\
&\leqslant 3\left\| \nabla f_i\left( z_{i}^{k} \right) -\nabla f_i\left( \bar{x}^k \right) \right\| ^2+3\left\| \nabla f_i\left( \bar{x}^k \right) -\nabla f\left( \bar{x}^k \right) \right\| ^2+3\left\| \nabla f\left( \bar{x}^k \right) \right\| ^2
\\
&\leqslant 3L^2\left\| z_{i}^{k}-\bar{x}^k \right\| ^2+3b^2+3\left\| \nabla f\left( \bar{x}^k \right) \right\| ^2,
\end{aligned}
\end{equation*}
which completes the proof.
\end{proof}

\subsection{Supporting Lemmas}
\label{supporting_emmas}
In this section, we provide several supporting lemmas to present three key inequalities associated with non-convex stochastic optimization (c.f.,~\eqref{eq_1}), variance reduction (c.f.,~\eqref{eq_2}) and consensus error (c.f.,~\eqref{eq_3}), respectively, based on which we later carry out the proof of Theorem~\ref{Theorem_1_saga} in Section~\ref{main_proof_theorem_1_saga}.



The following lemma is crucial to the proof of sub-linear convergence rate in the non-convex stochastic optimization analysis, which is obtained by applying the descent lemma recursively from $k=0$ to the total number of iterations.

\begin{Lem}
\label{Lemma_descent_lemma}
Suppose Assumption~\ref{Ass_weight_matrix},~\ref{assumption_smooth_saga} and~\ref{assumption_unbiased_gradient} hold. For a given constant step size $\gamma$, we have
\begin{equation}
\label{eq_1}
\begin{aligned}
&\frac{\gamma}{2}\sum_{k=0}^{K-1}{\mathbb{E}\left[ \left\| \nabla f\left( \bar{x}^k \right) \right\| ^2 \right]}+\sum_{k=0}^{K-1}{\mathbb{E}\left[ f\left( \bar{x}^{k+1} \right) \right]}
\\
\leqslant & \sum_{k=0}^{K-1}{\mathbb{E}\left[ f\left( \bar{x}^k \right) \right]}-\frac{\gamma -\gamma ^2L}{2}\sum_{k=0}^{K-1}{\mathbb{E}\left[ \left\| \frac{1}{n}\sum_{i=1}^n{\nabla f_i\left( z_{i}^{k} \right)} \right\| ^2 \right]}+\left( \frac{\gamma L^2}{2}+\frac{\gamma ^2L^3}{n} \right) \sum_{k=0}^{K-1}{\frac{1}{n}\sum_{i=1}^n{\mathbb{E}\left[ \left\| z_{i}^{k}-\bar{x}^k \right\| ^2 \right]}}
\\
&+\frac{\gamma ^2L^3}{n}\sum_{k=0}^{K-1}{\frac{1}{n}\sum_{i=1}^n{\frac{1}{J}\sum_{j=1}^J{\mathbb{E}\left[ \left\| \bar{x}^k-\phi _{i,j}^{k} \right\| ^2 \right]}}}+K\cdot \frac{\gamma ^2L}{2n}\cdot \frac{d}{n}\sum_{i=1}^n{\sigma _{i}^{2}}
.
\end{aligned}
\end{equation}
\end{Lem}

\begin{proof}
Applying the descent lemma to $f$ at $\bar{x}^k$ and $\bar{x}^{k+1}$, we have
\begin{equation}
\begin{aligned}
f\left( \bar{x}^{k+1} \right) 
\leqslant & f\left( \bar{x}^k \right) +\left< \nabla f\left( \bar{x}^k \right) ,\bar{x}^{k+1}-\bar{x}^k \right> +\frac{L}{2}\left\| \bar{x}^{k+1}-\bar{x}^k \right\| ^2
\\
\overset{\eqref{average_system_saga}}{=} & f\left( \bar{x}^k \right) -\gamma \left< \nabla f\left( \bar{x}^k \right) ,\frac{1}{n}\sum_{i=1}^n{g_{i}^{k}}+\frac{1}{n}\sum_{i=1}^n{N_{i}^{k}} \right> +\frac{\gamma ^2L}{2}\left\| \frac{1}{n}\sum_{i=1}^n{g_{i}^{k}}+\frac{1}{n}\sum_{i=1}^n{N_{i}^{k}} \right\| ^2
\\
= & f\left( \bar{x}^k \right) -\gamma \left< \nabla f\left( \bar{x}^k \right) ,\frac{1}{n}\sum_{i=1}^n{g_{i}^{k}}+\frac{1}{n}\sum_{i=1}^n{N_{i}^{k}} \right> +\frac{\gamma ^2L}{2}\left\| \frac{1}{n}\sum_{i=1}^n{g_{i}^{k}} \right\| ^2
\\
&+\frac{\gamma ^2L}{2}\left\| \frac{1}{n}\sum_{i=1}^n{N_{i}^{k}} \right\| ^2+\gamma ^2L\left< \frac{1}{n}\sum_{i=1}^n{g_{i}^{k}},\frac{1}{n}\sum_{i=1}^n{N_{i}^{k}} \right> .
\end{aligned}
\end{equation}

Taking the expectation of both sides conditioned on $\mathcal{F}^k$, we obtain
\begin{equation}
\label{descent_lemma_after_expextation}
\begin{aligned}
\mathbb{E}\left[ f\left( \bar{x}^{k+1} \right) \left| \mathcal{F}^k \right. \right]
\overset{\eqref{unbiased_corrected_gradient}}{\leqslant} & f\left( \bar{x}^k \right) \underset{A_1}{\underbrace{-\gamma \left< \nabla f\left( \bar{x}^k \right) ,\frac{1}{n}\sum_{i=1}^n{\nabla f_i\left( z_{i}^{k} \right)} \right> }}
\\
&+\frac{\gamma ^2L}{2}\underset{A_2}{\underbrace{\mathbb{E}\left[ \left\| \frac{1}{n}\sum_{i=1}^n{g_{i}^{k}} \right\| ^2\left| \mathcal{F}^k \right. \right] }}+\frac{\gamma ^2L}{2}\underset{A_3}{\underbrace{\mathbb{E}\left[ \left\| \frac{1}{n}\sum_{i=1}^n{N_{i}^{k}} \right\| ^2\left| \mathcal{F}^k \right. \right] }}
.
\end{aligned}
\end{equation}

For $A_1$, we have
\begin{equation}
\label{upper_bound_A_1}
\begin{aligned}
A_1&=-\frac{\gamma}{2}\left\| \nabla f\left( \bar{x}^k \right) \right\| ^2-\frac{\gamma}{2}\left\| \frac{1}{n}\sum_{i=1}^n{\nabla f_i\left( z_{i}^{k} \right)} \right\| ^2+\frac{\gamma}{2}\left\| \frac{1}{n}\sum_{i=1}^n{\nabla f_i\left( z_{i}^{k} \right)}-\nabla f\left( \bar{x}^k \right) \right\| ^2
\\
&=-\frac{\gamma}{2}\left\| \nabla f\left( \bar{x}^k \right) \right\| ^2-\frac{\gamma}{2}\left\| \frac{1}{n}\sum_{i=1}^n{\nabla f_i\left( z_{i}^{k} \right)} \right\| ^2+\frac{\gamma}{2}\left\| \frac{1}{n}\sum_{i=1}^n{\left( \nabla f_i\left( z_{i}^{k} \right) -\nabla f_i\left( \bar{x}^k \right) \right)} \right\| ^2
\\
&\leqslant -\frac{\gamma}{2}\left\| \nabla f\left( \bar{x}^k \right) \right\| ^2-\frac{\gamma}{2}\left\| \frac{1}{n}\sum_{i=1}^n{\nabla f_i\left( z_{i}^{k} \right)} \right\| ^2+\frac{\gamma}{2n}\sum_{i=1}^n{\left\| \nabla f_i\left( z_{i}^{k} \right) -\nabla f_i\left( \bar{x}^k \right) \right\| ^2}
\\
&\overset{\left( a \right)}{\leqslant}-\frac{\gamma}{2}\left\| \nabla f\left( \bar{x}^k \right) \right\| ^2-\frac{\gamma}{2}\left\| \frac{1}{n}\sum_{i=1}^n{\nabla f_i\left( z_{i}^{k} \right)} \right\| ^2+\frac{\gamma L^2}{2n}\sum_{i=1}^n{\left\| z_{i}^{k}-\bar{x}^k \right\| ^2},
\end{aligned}
\end{equation}
where in $(a)$ we used Assumption~\ref{assumption_smooth_saga}.

For $A_2$, we have
\begin{equation}
\label{upper_bound_A_2}
\begin{aligned}
A_2  = &\mathbb{E}\left[ \left\| \frac{1}{n}\sum_{i=1}^n{\left( g_{i}^{k}-\nabla f_i\left( z_{i}^{k} \right) \right)}+\frac{1}{n}\sum_{i=1}^n{\nabla f_i\left( z_{i}^{k} \right)} \right\| ^2\left| \mathcal{F}^k \right. \right] 
\\
=&\mathbb{E}\left[ \left\| \frac{1}{n}\sum_{i=1}^n{\left( g_{i}^{k}-\nabla f_i\left( z_{i}^{k} \right) \right)} \right\| ^2\left| \mathcal{F}^k \right. \right] +\mathbb{E}\left[ \left\| \frac{1}{n}\sum_{i=1}^n{\nabla f_i\left( z_{i}^{k} \right)} \right\| ^2\left| \mathcal{F}^k \right. \right] 
\\
&+2\mathbb{E}\left[ \left< \frac{1}{n}\sum_{i=1}^n{\left( g_{i}^{k}-\nabla f_i\left( z_{i}^{k} \right) \right)},\frac{1}{n}\sum_{i=1}^n{\nabla f_i\left( z_{i}^{k} \right)} \right> \left| \mathcal{F}^k \right. \right] 
\\
\overset{\left( b \right)}{\leqslant}&\frac{L^2}{n^2}\sum_{i=1}^n{\frac{1}{J}\sum_{j=1}^J{\left\| z_{i}^{k}-\phi _{i,j}^{k} \right\| ^2}}+\left\| \frac{1}{n}\sum_{i=1}^n{\nabla f_i\left( z_{i}^{k} \right)} \right\| ^2
\\
\leqslant &\frac{2L^2}{n}\cdot \frac{1}{n}\sum_{i=1}^n{\left\| z_{i}^{k}-\bar{x}^k \right\| ^2}+\frac{2L^2}{n}\cdot \frac{1}{n}\sum_{i=1}^n{\frac{1}{J}\sum_{j=1}^J{\left\| \bar{x}^k-\phi _{i,j}^{k} \right\| ^2}}+\left\| \frac{1}{n}\sum_{i=1}^n{\nabla f_i\left( z_{i}^{k} \right)} \right\| ^2
,
\end{aligned}
\end{equation}
where in $(b)$ we used~\eqref{unbiased_corrected_gradient} in Lemma~\ref{ubias}, and~\eqref{upper_bound_of_corrected_gradient} in Lemma~\ref{bounded}.

For $A_3$, since Gaussian noises $\{N_i^k\}_{i=1,2,...,n}$ are independent with each other, we have 
\begin{equation}
\label{upper_bound_A_3}
A_3=\frac{1}{n^2}\sum_{i=1}^n{\mathbb{E}\left[ \left\| N_{i}^{k} \right\| ^2 \right]} = \frac{d}{n^2}\sum_{i=1}^n{\sigma _{i}^{2}}.
\end{equation}

Then, substituting~\eqref{upper_bound_A_1},~\eqref{upper_bound_A_2} and~\eqref{upper_bound_A_3} into~\eqref{descent_lemma_after_expextation} yields 
\begin{equation}
\begin{aligned}
\mathbb{E}\left[ f\left( \bar{x}^{k+1} \right) \left| \mathcal{F}^k \right. \right] \leqslant & f\left( \bar{x}^k \right) -\frac{\gamma}{2}\left\| \nabla f\left( \bar{x}^k \right) \right\| ^2-\frac{\gamma -\gamma ^2L}{2}\left\| \frac{1}{n}\sum_{i=1}^n{\nabla f_i\left( z_{i}^{k} \right)} \right\| ^2
\\
&+\left( \frac{\gamma L^2}{2}+\frac{\gamma ^2L^3}{n} \right) \cdot \frac{1}{n}\sum_{i=1}^n{\left\| z_{i}^{k}-\bar{x}^k \right\| ^2}+\frac{\gamma ^2L^3}{n}\cdot \frac{1}{n}\sum_{i=1}^n{\frac{1}{J}\sum_{j=1}^J{\left\| \bar{x}^k-\phi _{i,j}^{k} \right\| ^2}}+\frac{\gamma ^2L}{2n}\cdot \frac{d}{n}\sum_{i=1}^n{\sigma _{i}^{2}}.
\end{aligned}
\end{equation}

Taking total expectation on both sides of the above inequality, we further have
\begin{equation}
\label{before_descent_lemma}
\begin{aligned}
\mathbb{E}\left[ f\left( \bar{x}^{k+1} \right) \right] \leqslant & \mathbb{E}\left[ f\left( \bar{x}^k \right) \right] -\frac{\gamma}{2}\mathbb{E}\left[ \left\| \nabla f\left( \bar{x}^k \right) \right\| ^2 \right] -\frac{\gamma -\gamma ^2L}{2}\mathbb{E}\left[ \left\| \frac{1}{n}\sum_{i=1}^n{\nabla f_i\left( z_{i}^{k} \right)} \right\| ^2 \right] 
\\
&+\left( \frac{\gamma L^2}{2}+\frac{\gamma ^2L^3}{n} \right) \cdot \frac{1}{n}\sum_{i=1}^n{\mathbb{E}\left[ \left\| z_{i}^{k}-\bar{x}^k \right\| ^2 \right]}
\\
&+\frac{\gamma ^2L^3}{n}\cdot \frac{1}{n}\sum_{i=1}^n{\frac{1}{J}\sum_{j=1}^J{\mathbb{E}\left[ \left\| \bar{x}^k-\phi _{i,j}^{k} \right\| ^2 \right]}}+\frac{\gamma ^2L}{2n}\cdot \frac{d}{n}\sum_{i=1}^n{\sigma _{i}^{2}}.
\end{aligned}
\end{equation}

Summing~\eqref{before_descent_lemma} from $k=0$ to $K-1$, we obtain \eqref{eq_1}, which completes the proof.
\end{proof}

The following lemma establishes a corresponding inequality associated with variance reduction (VR), which bounds the accumulative VR-related error term with the accumulative consensus error. 
\begin{Lem}
\label{iterate_of_T_k}
Define
\begin{equation}
\label{def_of_T_k}
T^k:=\frac{1}{n}\sum_{i=1}^n{\frac{1}{J}\sum_{j=1}^J{\left\| \bar{x}^k-\phi _{i,j}^{k} \right\| ^2}},
\end{equation}
and suppose Assumption~\ref{assumption_smooth_saga} and ~\ref{assumption_unbiased_gradient} hold. Then, we have
\begin{equation}
\label{eq_2}
\begin{aligned}
\sum_{k=0}^{K-1}{\mathbb{E}\left[ T^{k+1} \right]}\leqslant & \left[ \frac{2\left( 1+J \right) \gamma ^2L^2}{n}+\left( 1+\frac{1}{J} \right) \frac{1}{J} \right] \sum_{k=0}^{K-1}{\frac{1}{n}\sum_{i=1}^n{\mathbb{E}\left[ \left\| z_{i}^{k}-\bar{x}^k \right\| ^2 \right]}}
\\
&+\left( 1-\frac{1}{J^2}+\frac{2\left( 1+J \right) \gamma ^2L^2}{n} \right) \sum_{k=0}^{K-1}{\mathbb{E}\left[ T^k \right]}
\\
&+\left( 1+J \right) \gamma ^2\sum_{k=0}^{K-1}{\mathbb{E}\left[ \left\| \frac{1}{n}\sum_{i=1}^n{\nabla f_i\left( z_{i}^{k} \right)} \right\| ^2 \right]}+K\cdot \left( 1+J \right) \gamma ^2\frac{d}{n^2}\sum_{i=1}^n{\sigma _{i}^{2}}.
\end{aligned}
\end{equation}
\end{Lem}

\begin{proof}
According to the definition in~\eqref{def_of_T_k}, we have
\begin{equation}
\label{derivation_1}
\begin{aligned}
\mathbb{E}\left[ T^{k+1}\left| \mathcal{F}^k \right. \right] & =\frac{1}{n}\sum_{i=1}^n{\frac{1}{J}\sum_{j=1}^J{\mathbb{E}\left[ \left\| \bar{x}^{k+1}-\phi _{i,j}^{k+1} \right\| ^2\left| \mathcal{F}^k \right. \right]}}
\\
&\overset{(a)}{\leqslant}\frac{1}{n}\sum_{i=1}^n{\frac{1}{J}\sum_{j=1}^J{\left( \left( 1+J \right) \mathbb{E}\left[ \left\| \bar{x}^{k+1}-\bar{x}^k \right\| ^2\left| \mathcal{F}^k \right. \right] +\left( 1+\frac{1}{J} \right) \mathbb{E}\left[ \left\| \bar{x}^k-\phi _{i,j}^{k+1} \right\| ^2\left| \mathcal{F}^k \right. \right] \right)}}
\\
&=\left( 1+J \right) \mathbb{E}\left[ \left\| \bar{x}^{k+1}-\bar{x}^k \right\| ^2\left| \mathcal{F}^k \right. \right] +\left( 1+\frac{1}{J} \right) \cdot \frac{1}{n}\sum_{i=1}^n{\frac{1}{J}\sum_{j=1}^J{\mathbb{E}\left[ \left\| \bar{x}^k-\phi _{i,j}^{k+1} \right\| ^2\left| \mathcal{F}^k \right. \right]}}
,
\end{aligned}
\end{equation}
where in (a) we used Young's inequality.

For the first term in the right hand side of~\eqref{derivation_1}, we have
\begin{equation}
\label{derivation_2}
\begin{aligned}
\mathbb{E}\left[ \left\| \bar{x}^{k+1}-\bar{x}^k \right\| ^2\left| \mathcal{F}^k \right. \right] &\overset{\eqref{average_system_saga}}{=}\gamma ^2\mathbb{E}\left[ \left\| \frac{1}{n}\sum_{i=1}^n{g_{i}^{k}}+\frac{1}{n}\sum_{i=1}^n{N_{i}^{k}} \right\| ^2\left| \mathcal{F}^k \right. \right] 
\\
&=\gamma ^2\mathbb{E}\left[ \left\| \frac{1}{n}\sum_{i=1}^n{g_{i}^{k}} \right\| ^2\left| \mathcal{F}^k \right. \right] +\gamma ^2\mathbb{E}\left[ \left\| \frac{1}{n}\sum_{i=1}^n{N_{i}^{k}} \right\| ^2\left| \mathcal{F}^k \right. \right] 
\\
&\overset{\eqref{upper_bound_A_2},\eqref{upper_bound_A_3}}{\leqslant}  \frac{2\gamma ^2L^2}{n}\cdot \frac{1}{n}\sum_{i=1}^n{\left\| z_{i}^{k}-\bar{x}^k \right\| ^2}+\frac{2\gamma ^2L^2}{n}\cdot T^k+\gamma ^2\left\| \frac{1}{n}\sum_{i=1}^n{\nabla f_i\left( z_{i}^{k} \right)} \right\| ^2+\gamma ^2\frac{d}{n^2}\sum_{i=1}^n{\sigma _{i}^{2}}.
\end{aligned}
\end{equation}

For the second term in the right hand side of~\eqref{derivation_1}, we have 
\begin{equation}
\label{derivation_3}
\begin{aligned}
\frac{1}{n}\sum_{i=1}^n{\frac{1}{J}\sum_{j=1}^J{\mathbb{E}\left[ \left\| \bar{x}^k-\phi _{i,j}^{k+1} \right\| ^2\left| \mathcal{F}^k \right. \right]}}
&\overset{(b)}{=}
\frac{1}{n}\sum_{i=1}^n{\frac{1}{J}\sum_{j=1}^J{\left\{ \frac{1}{J}\left\| \bar{x}^k-z_{i}^{k} \right\| ^2+\left( 1-\frac{1}{J} \right) \left\| \bar{x}^k-\phi _{i,j}^{k} \right\| ^2 \right\}}}
\\
&=\frac{1}{J}\cdot \frac{1}{n}\sum_{i=1}^n{\left\| z_{i}^{k}-\bar{x}^k \right\| ^2}+\left( 1-\frac{1}{J} \right) \cdot T^k.
\end{aligned}
\end{equation}
where $(b)$ holds because at iteration $k$, node $i$ uniformly at random chooses one out of $J$ data samples. For the chosen data sample $j$, $\phi_{i,j}^{k+1}=z_i^k$; otherwise $\phi_{i,j}^{k+1}=\phi_{i,j}^{k}$ (c.f.,~\eqref{saga_table}).

Substituting~\eqref{derivation_2} and~\eqref{derivation_3} into~\eqref{derivation_1}, we get
\begin{equation}
\begin{aligned}
\mathbb{E}\left[ T^{k+1}\left| \mathcal{F}^k \right. \right] 
\leqslant & \left[ \frac{2\left( 1+J \right) \gamma ^2L^2}{n}+\left( 1+\frac{1}{J} \right) \frac{1}{J} \right] \cdot \frac{1}{n}\sum_{i=1}^n{\left\| z_{i}^{k}-\bar{x}^k \right\| ^2}
\\
&+\left( 1-\frac{1}{J^2}+\frac{2\left( 1+J \right) \gamma ^2L^2}{n} \right) T^k+\left( 1+J \right) \gamma ^2\left\| \frac{1}{n}\sum_{i=1}^n{\nabla f_i\left( z_{i}^{k} \right)} \right\| ^2
+\left( 1+J \right) \gamma ^2\frac{d}{n^2}\sum_{i=1}^n{\sigma _{i}^{2}}.
\end{aligned}
\end{equation}

Taking the total expectation on both sides of the above inequality, we further have
\begin{equation}
\label{before_T_k}
\begin{aligned}
\mathbb{E}\left[ T^{k+1} \right] \leqslant & \left[ \frac{2\left( 1+J \right) \gamma ^2L^2}{n}+\left( 1+\frac{1}{J} \right) \frac{1}{J} \right] \cdot \frac{1}{n}\sum_{i=1}^n{\mathbb{E}\left[ \left\| z_{i}^{k}-\bar{x}^k \right\| ^2 \right]}
\\
&+\left( 1-\frac{1}{J^2}+\frac{2\left( 1+J \right) \gamma ^2L^2}{n} \right) \mathbb{E}\left[ T^k \right] +\left( 1+J \right) \gamma ^2\mathbb{E}\left[ \left\| \frac{1}{n}\sum_{i=1}^n{\nabla f_i\left( z_{i}^{k} \right)} \right\| ^2 \right] 
\\
&+\left( 1+J \right) \gamma ^2\frac{d}{n^2}\sum_{i=1}^n{\sigma _{i}^{2}}.
\end{aligned}
\end{equation}

Summing~\eqref{before_T_k} from $k=0$ to $K-1$, we obtain ~\eqref{eq_2}, which completes the proof.
\end{proof}
The following lemma bounds the distance between the de-biased parameters $z_i^k$ at each node $i$ and the node-wise average $\bar{x}^k$, which can be adapted from Lemma 3 in~\cite{assran2019stochastic}.
\begin{Lem}
\label{def_of_C_and_q}
Suppose that Assumptions~\ref{Ass_weight_matrix} and~\ref{assumption_mixing_matrix} hold.
Let $\varepsilon$ be the minimum of all non-zero mixing weights, $\lambda =1-n\varepsilon^{\bigtriangleup B}$ and $q=\lambda ^{\frac{1}{\bigtriangleup B+1}}
$. Then, there exists a constant 
\begin{equation}
C<\frac{2\sqrt{d}\varepsilon^{-\bigtriangleup B}}{\lambda ^{\frac{\bigtriangleup B+2}{\bigtriangleup B+1}}}
,
\end{equation}
such that for any $i \in \{1,2,...,n\}$ and $k \geqslant 0$, we have
\begin{equation}
\label{derivation_4_saga}
\left\| z_{i}^{k}-\bar{x}^k \right\| \leqslant Cq^k\left\| x_{i}^{0} \right\| +\gamma C\sum_{s=0}^k{q^{k-s}\left\| g_i^s +N_{i}^{s} \right\|}
.
\end{equation}
\end{Lem}

Now, we attempt to upper bound the accumulative consensus error $\sum_{k=0}^{K-1}{\frac{1}{n}\sum_{i=1}^n{\mathbb{E}\left[ \left\| z_{i}^{k}-\bar{x}^k \right\| ^2 \right]}}$ using $\sum_{k=0}^{K-1}{\mathbb{E}\left[ \left\| \nabla f\left( \bar{x}^k \right) \right\| ^2 \right]}$ and $\sum_{k=0}^{K-1}{\mathbb{E}\left[ T^k \right]}$, which is summarized in the following lemma.

\begin{Lem}
\label{iterate_of_M_k_saga}
Define
\begin{equation}
\label{def_of_M_k_saga}
M^k:=\frac{1}{n}\sum_{i=1}^n{\mathbb{E}\left[ \left\| z_{i}^{k}-\bar{x}^k \right\| ^2 \right]},
\end{equation}
and suppose Assumptions~\ref{Ass_weight_matrix}-\ref{assumption_bounede_outer_variation} hold. 
If the constant step size $\gamma$ satisfies
\begin{equation}
\label{lr_condition_1}
\gamma \leqslant \frac{1-q}{\sqrt{30}LC},
\end{equation}
we have
\begin{equation}
\label{eq_3}
\begin{aligned}
\sum_{k=0}^{K-1}{M^k}\leqslant & \frac{6C^2}{\left( 1-q^2 \right) n}\sum_{i=1}^n{\left\| x_{i}^{0} \right\| ^2}+\frac{6\gamma ^2C^2}{\left( 1-q \right) ^2}\left( 3b^2+\frac{d}{n}\sum_{i=1}^n{\sigma _{i}^{2}} \right) K
\\
&+\frac{18\gamma ^2C^2}{\left( 1-q \right) ^2}\sum_{k=0}^{K-1}{\mathbb{E}\left[ \left\| \nabla f\left( \bar{x}^k \right) \right\| ^2 \right]}+\frac{12\gamma ^2L^2C^2}{\left( 1-q \right) ^2}\sum_{k=0}^{K-1}{\mathbb{E}\left[ T^k \right]}.
\end{aligned}
\end{equation}
\end{Lem}

\begin{proof}
According to~\eqref{derivation_4_saga}, we have
\begin{equation}
\left\| z_{i}^{k}-\bar{x}^k \right\| \leqslant Cq^k\left\| x_{i}^{0} \right\| +\gamma C\sum_{s=0}^k{q^{k-s}\left\| g_{i}^{s} \right\|}+\gamma C\sum_{s=0}^k{q^{k-s}\left\| N_{i}^{s} \right\|}.
\end{equation}

Squaring on both sides of the above inequality, we have
\begin{equation}
\label{why_here}
\begin{aligned}
\left\| z_{i}^{k}-\bar{x}^k \right\| ^2 & \leqslant 3C^2q^{2k}\left\| x_{i}^{0} \right\| ^2+3\gamma ^2C^2\left( \sum_{s=0}^k{q^{k-s}\left\| g_{i}^{s} \right\|} \right) ^2+3\gamma ^2C^2\left( \sum_{s=0}^k{q^{k-s}\left\| N_{i}^{s} \right\|} \right) ^2
\\
&\leqslant 3C^2q^{2k}\left\| x_{i}^{0} \right\| ^2+\frac{3\gamma ^2C^2}{1-q}\sum_{s=0}^k{q^{k-s}\left\| g_{i}^{s} \right\| ^2}+\frac{3\gamma ^2C^2}{1-q}\sum_{s=0}^k{q^{k-s}\left\| N_{i}^{s} \right\| ^2},
\end{aligned}
\end{equation}
where in the second inequality we used Lemma~\ref{supporting_lemma_2}.

Taking total expectation on both sides of~\eqref{why_here}, yields
\begin{equation}
\begin{aligned}
&\mathbb{E}\left[ \left\| z_{i}^{k}-\bar{x}^k \right\| ^2 \right] 
\\
\leqslant& 3C^2q^{2k}\left\| x_{i}^{0} \right\| ^2+\frac{3\gamma ^2C^2}{1-q}\sum_{s=0}^k{q^{k-s}\mathbb{E}\left[ \left\| g_{i}^{s}-\nabla f_i\left( z_{i}^{s} \right) +\nabla f_i\left( z_{i}^{s} \right) \right\| ^2 \right]}
+\frac{3\gamma ^2C^2}{1-q}\sum_{s=0}^k{q^{k-s}\mathbb{E}\left[ \left\| N_{i}^{s} \right\| ^2 \right]}
\\
=&3C^2q^{2k}\left\| x_{i}^{0} \right\| ^2+\frac{3\gamma ^2C^2}{1-q}\sum_{s=0}^k{q^{k-s}\mathbb{E}\left[ \left\| g_{i}^{s}-\nabla f_i\left( z_{i}^{s} \right) \right\| ^2 \right]}
\\
&+\frac{3\gamma ^2C^2}{1-q}\sum_{s=0}^k{q^{k-s}\mathbb{E}\left[ \left\| \nabla f_i\left( z_{i}^{s} \right) \right\| ^2 \right]}+\frac{3\gamma ^2C^2}{1-q}\sum_{s=0}^k{q^{k-s}\mathbb{E}\left[ \left\| N_{i}^{s} \right\| ^2 \right]}
\\
\overset{\left( a \right)}{\leqslant} & 3C^2q^{2k}\left\| x_{i}^{0} \right\| ^2+\frac{3\gamma ^2C^2}{1-q}\sum_{s=0}^k{q^{k-s}L^2\cdot \frac{1}{J}\sum_{j=1}^J{\mathbb{E}\left[ \left\| z_{i}^{s}-\phi _{i,j}^{s} \right\| ^2 \right]}}
\\
&+\frac{3\gamma ^2C^2}{1-q}\sum_{s=0}^k{q^{k-s}\mathbb{E}\left[ 3L^2\left\| z_{i}^{s}-\bar{x}^s \right\| ^2+3b^2+3\left\| \nabla f\left( \bar{x}^s \right) \right\| ^2 \right]}+\frac{3\gamma ^2C^2}{1-q}\sum_{s=0}^k{q^{k-s}\cdot d\sigma _{i}^{2}}
\\
\leqslant & 3C^2q^{2k}\left\| x_{i}^{0} \right\| ^2+\frac{3\gamma ^2C^2}{1-q}\sum_{s=0}^k{q^{k-s}L^2\cdot \frac{1}{J}\sum_{j=1}^J{\left( 2\mathbb{E}\left[ \left\| z_{i}^{s}-\bar{x}^s \right\| ^2 \right] +2\mathbb{E}\left[ \left\| \bar{x}^s-\phi _{i,j}^{s} \right\| ^2 \right] \right)}}
\\
&+\frac{3\gamma ^2C^2}{1-q}\sum_{s=0}^k{q^{k-s}\mathbb{E}\left[ 3L^2\left\| z_{i}^{s}-\bar{x}^s \right\| ^2+3b^2+3\left\| \nabla f\left( \bar{x}^s \right) \right\| ^2 \right]}+\frac{3\gamma ^2C^2}{\left( 1-q \right) ^2}d\sigma _{i}^{2}
\\
\leqslant & 3C^2q^{2k}\left\| x_{i}^{0} \right\| ^2+\frac{3\gamma ^2C^2}{\left( 1-q \right) ^2}\left( 3b^2+d\sigma _{i}^{2} \right) +\frac{15\gamma ^2L^2C^2}{1-q}\sum_{s=0}^k{q^{k-s}\mathbb{E}\left[ \left\| z_{i}^{s}-\bar{x}^s \right\| ^2 \right]}
\\
&+\frac{9\gamma ^2C^2}{1-q}\sum_{s=0}^k{q^{k-s}\mathbb{E}\left[ \left\| \nabla f\left( \bar{x}^s \right) \right\| ^2 \right]}+\frac{6\gamma ^2L^2C^2}{1-q}\sum_{s=0}^k{q^{k-s}\cdot \frac{1}{J}\sum_{j=1}^J{\mathbb{E}\left[ \left\| \bar{x}^s-\phi _{i,j}^{s} \right\| ^2 \right]}},
\end{aligned}
\end{equation}
where in $(a)$ we have used Lemma~\ref{bounded} and Lemma~\ref{supporting_lemma_3_saga}.

According to the definition of 
$M^k$ in~\eqref{def_of_M_k_saga}, we have
\begin{equation}
\begin{aligned}
M^k=&\frac{1}{n}\sum_{i=1}^n{\mathbb{E}\left[ \left\| z_{i}^{k}-\bar{x}^k \right\| ^2 \right]}
\\
\leqslant & \frac{3C^2q^{2k}}{n}\sum_{i=1}^n{\left\| x_{i}^{0} \right\| ^2}+\frac{3\gamma ^2C^2}{\left( 1-q \right) ^2}\left( 3b^2+\frac{d}{n}\sum_{i=1}^n{\sigma _{i}^{2}} \right) 
+\frac{15\gamma ^2L^2C^2}{1-q}\sum_{s=0}^k{q^{k-s}\cdot \frac{1}{n}\sum_{i=1}^n{\mathbb{E}\left[ \left\| z_{i}^{s}-\bar{x}^s \right\| ^2 \right]}}
\\
&+\frac{9\gamma ^2C^2}{1-q}\sum_{s=0}^k{q^{k-s}\mathbb{E}\left[ \left\| \nabla f\left( \bar{x}^s \right) \right\| ^2 \right]}+\frac{6\gamma ^2L^2C^2}{1-q}\sum_{s=0}^k{q^{k-s}\cdot \frac{1}{n}\sum_{i=1}^n{\frac{1}{J}\sum_{j=1}^J{\mathbb{E}\left[ \left\| \bar{x}^s-\phi _{i,j}^{s} \right\| ^2 \right]}}}
\\
\overset{\eqref{def_of_M_k_saga}\eqref{def_of_T_k}}{=}&\frac{3C^2q^{2k}}{n}\sum_{i=1}^n{\left\| x_{i}^{0} \right\| ^2}+\frac{3\gamma ^2C^2}{\left( 1-q \right) ^2}\left( 3b^2+\frac{d}{n}\sum_{i=1}^n{\sigma _{i}^{2}} \right) +\frac{15\gamma ^2L^2C^2}{1-q}\sum_{s=0}^k{q^{k-s}M^s}
\\
&+\frac{9\gamma ^2C^2}{1-q}\sum_{s=0}^k{q^{k-s}\mathbb{E}\left[ \left\| \nabla f\left( \bar{x}^s \right) \right\| ^2 \right]}+\frac{6\gamma ^2L^2C^2}{1-q}\sum_{s=0}^k{q^{k-s}\mathbb{E}\left[ T^s \right]}
.
\end{aligned}
\end{equation}

Summing the above from $k=0$ to $K-1$, yields that
\begin{equation}
\label{derivation_5}
\begin{aligned}
\sum_{k=0}^{K-1}{M^k}
\leqslant & \frac{3C^2}{\left( 1-q^2 \right) n}\sum_{i=1}^n{\left\| x_{i}^{0} \right\| ^2}+\frac{3\gamma ^2C^2}{\left( 1-q \right) ^2}\left( 3b^2+\frac{d}{n}\sum_{i=1}^n{\sigma _{i}^{2}} \right) K+\frac{15\gamma ^2L^2C^2}{1-q}\sum_{k=0}^{K-1}{\sum_{s=0}^k{q^{k-s}M^s}}
\\
&+\frac{9\gamma ^2C^2}{1-q}\sum_{k=0}^{K-1}{\sum_{s=0}^k{q^{k-s}\mathbb{E}\left[ \left\| \nabla f\left( \bar{x}^s \right) \right\| ^2 \right]}}+\frac{6\gamma ^2L^2C^2}{1-q}\sum_{k=0}^{K-1}{\sum_{s=0}^k{q^{k-s}\mathbb{E}\left[ T^s \right]}}
\\
\overset{\left( b \right)}{\leqslant} & \frac{3C^2}{\left( 1-q^2 \right) n}\sum_{i=1}^n{\left\| x_{i}^{0} \right\| ^2}+\frac{3\gamma ^2C^2}{\left( 1-q \right) ^2}\left( 3b^2+\frac{d}{n}\sum_{i=1}^n{\sigma _{i}^{2}} \right) K+\frac{15\gamma ^2L^2C^2}{\left( 1-q \right) ^2}\sum_{k=0}^{K-1}{M^k}
\\
&+\frac{9\gamma ^2C^2}{\left( 1-q \right) ^2}\sum_{k=0}^{K-1}{\mathbb{E}\left[ \left\| \nabla f\left( \bar{x}^k \right) \right\| ^2 \right]}+\frac{6\gamma ^2L^2C^2}{\left( 1-q \right) ^2}\sum_{k=0}^{K-1}{\mathbb{E}\left[ T^k \right]},
\end{aligned}
\end{equation}
where in $(b)$ we used $\sum_{k=0}^{K-1}{\sum_{s=0}^k{q^{k-s}a_s}}\leqslant \frac{1}{1-q}\sum_{k=0}^{K-1}{a_k}
$.

Rearranging the term in~\eqref{derivation_5}, we get
\begin{equation}
\label{asdf}
\begin{aligned}
\left( 1-\frac{15\gamma ^2L^2C^2}{\left( 1-q \right) ^2} \right) \sum_{k=0}^{K-1}{M^k}\leqslant &\frac{3C^2}{\left( 1-q^2 \right) n}\sum_{i=1}^n{\left\| x_{i}^{0} \right\| ^2}+\frac{3\gamma ^2C^2}{\left( 1-q \right) ^2}\left( 3b^2+\frac{d}{n}\sum_{i=1}^n{\sigma _{i}^{2}} \right) K
\\
&+\frac{9\gamma ^2C^2}{\left( 1-q \right) ^2}\sum_{k=0}^{K-1}{\mathbb{E}\left[ \left\| \nabla f\left( \bar{x}^k \right) \right\| ^2 \right]}+\frac{6\gamma ^2L^2C^2}{\left( 1-q \right) ^2}\sum_{k=0}^{K-1}{\mathbb{E}\left[ T^k \right]}.
\end{aligned}
\end{equation}

According to~\eqref{lr_condition_1}, we have
\begin{equation}
1-\frac{15\gamma ^2L^2C^2}{\left( 1-q \right) ^2}\geqslant \frac{1}{2},
\end{equation}
thus~\eqref{asdf} can be further relaxed as
\begin{equation}
\begin{aligned}
\sum_{k=0}^{K-1}{M^k}\leqslant & \frac{6C^2}{\left( 1-q^2 \right) n}\sum_{i=1}^n{\left\| x_{i}^{0} \right\| ^2}+\frac{6\gamma ^2C^2}{\left( 1-q \right) ^2}\left( 3b^2+\frac{d}{n}\sum_{i=1}^n{\sigma _{i}^{2}} \right) K
\\
&+\frac{18\gamma ^2C^2}{\left( 1-q \right) ^2}\sum_{k=0}^{K-1}{\mathbb{E}\left[ \left\| \nabla f\left( \bar{x}^k \right) \right\| ^2 \right]}+\frac{12\gamma ^2L^2C^2}{\left( 1-q \right) ^2}\sum_{k=0}^{K-1}{\mathbb{E}\left[ T^k \right]},
\end{aligned}
\end{equation}
which completes the proof.
\end{proof}

\subsection{Proof of Theorem~\ref{Theorem_1_saga}}
\label{main_proof_theorem_1_saga}
\textbf{Proof Sketch.}
With the above supporting lemmas, we are now ready to provide the proof of Theorem~\ref{Theorem_1_saga}, which consists of three key steps.
The first step (\textbf{Step 1}) is to upper bound $\frac{1}{K}\sum_{k=0}^{K-1}{\mathbb{E}\left[ \left\| \nabla f\left( \bar{x}^k \right) \right\| ^2 \right]}$, which could be achieved by carefully coupling the above-mentioned three key inequalities (refer to~\eqref{combine_with_eq1_and_eq2} and~\eqref{ed_saga}) and properly designing the value of coupling coefficient $\beta$ (refer to~\eqref{value_of_beta}).
The second step (\textbf{Step 2}) is to upper bound $\frac{1}{K}\sum_{k=0}^{K-1}{\frac{1}{n}\sum_{i=1}^n{\mathbb{E}\left[ \left\| z_{i}^{k}-\bar{x}^k \right\| ^2 \right]}}$ through connecting three key inequalities (refer to~\eqref{my_eq} and~\eqref{last_eq}).
The last step (\textbf{Step 3}) is to upper bound $\frac{1}{K}\sum_{k=0}^{K-1}{\frac{1}{n}\sum_{i=1}^n{\mathbb{E}\left[ \left\| \nabla f\left( z_{i}^{k} \right) \right\| ^2 \right]}}$ based on those two upper bounds obtained in the above two steps.
\\
\\
\noindent
\textbf{Step 1:} \quad \textbf{Upper bounding} $\frac{1}{K}\sum_{k=0}^{K-1}{\mathbb{E}\left[ \left\| \nabla f\left( \bar{x}^k \right) \right\| ^2 \right]}$
\\
\\
According to the definition of $T^k$ (c.f.,~\eqref{def_of_T_k}) and $M^k$ (c.f.,~\eqref{def_of_M_k_saga}), \eqref{eq_1} and~\eqref{eq_2} can be rewritten as
\begin{equation}
\label{after_eq_1}
\begin{aligned}
&\frac{\gamma}{2}\sum_{k=0}^{K-1}{\mathbb{E}\left[ \left\| \nabla f\left( \bar{x}^k \right) \right\| ^2 \right]}+\sum_{k=0}^{K-1}{\mathbb{E}\left[ f\left( \bar{x}^{k+1} \right) \right]}
\\
\leqslant& \sum_{k=0}^{K-1}{\mathbb{E}\left[ f\left( \bar{x}^k \right) \right]}-\frac{\gamma -\gamma ^2L}{2}\sum_{k=0}^{K-1}{\mathbb{E}\left[ \left\| \frac{1}{n}\sum_{i=1}^n{\nabla f_i\left( z_{i}^{k} \right)} \right\| ^2 \right]} +\left( \frac{\gamma L^2}{2}+\frac{\gamma ^2L^3}{n} \right) \sum_{k=0}^{K-1}{M^k}
\\
&+\frac{\gamma ^2L^3}{n}\sum_{k=0}^{K-1}{\mathbb{E}\left[ T^k \right]}+K\cdot \frac{\gamma ^2L}{2n}\cdot \frac{d}{n}\sum_{i=1}^n{\sigma _{i}^{2}}
\end{aligned}
\end{equation}
and
\begin{equation}
\label{after_eq_2}
\begin{aligned}
\sum_{k=0}^{K-1}{\mathbb{E}\left[ T^{k+1} \right]}
\leqslant &\left[ \frac{2\left( 1+J \right) \gamma ^2L^2}{n}+\left( 1+\frac{1}{J} \right) \frac{1}{J} \right] \sum_{k=0}^{K-1}{M^k}+\left( 1-\frac{1}{J^2}+\frac{2\left( 1+J \right) \gamma ^2L^2}{n} \right) \sum_{k=0}^{K-1}{\mathbb{E}\left[ T^k \right]}
\\
&+\left( 1+J \right) \gamma ^2\sum_{k=0}^{K-1}{\mathbb{E}\left[ \left\| \frac{1}{n}\sum_{i=1}^n{\nabla f_i\left( z_{i}^{k} \right)} \right\| ^2 \right]}+K\cdot \left( 1+J \right) \gamma ^2\frac{d}{n^2}\sum_{i=1}^n{\sigma _{i}^{2}}
\end{aligned}
\end{equation}
respectively.

By computing~\eqref{after_eq_1}$+$\eqref{after_eq_2} $\times \beta L$ ($\beta$ is a positive constant to be properly determined later), we get
\begin{equation}
\label{combine_with_eq1_and_eq2}
\begin{aligned}
&\sum_{k=0}^{K-1}{\mathbb{E}\left[ f\left( \bar{x}^{k+1} \right) \right]}+\beta L\cdot \sum_{k=0}^{K-1}{\mathbb{E}\left[ T^{k+1} \right]}
\\
\leqslant & \sum_{k=0}^{K-1}{\mathbb{E}\left[ f\left( \bar{x}^k \right) \right]}-\frac{\gamma}{2}\sum_{k=0}^{K-1}{\mathbb{E}\left[ \left\| \nabla f\left( \bar{x}^k \right) \right\| ^2 \right]}
-\frac{\gamma -\gamma ^2L-2\left( 1+J \right) \beta \gamma ^2L}{2}\cdot \sum_{k=0}^{K-1}{\mathbb{E}\left[ \left\| \frac{1}{n}\sum_{i=1}^n{\nabla f_i\left( z_{i}^{k} \right)} \right\| ^2 \right]}
\\
&+\left[ \frac{\gamma ^2L^3}{n}+\beta \left( \left( 1-\frac{1}{J^2} \right) L+\frac{2\left( 1+J \right) \gamma ^2L^3}{n} \right) \right] \cdot \sum_{k=0}^{K-1}{\mathbb{E}\left[ T^k \right]}
\\
&+\left[ \frac{\gamma L^2}{2}+\frac{\gamma ^2L^3}{n}+\beta \left[ \frac{2\left( 1+J \right) \gamma ^2L^3}{n}+\left( 1+\frac{1}{J} \right) \frac{L}{J} \right] \right] \cdot \sum_{k=0}^{K-1}{M^k}
+K\left( \frac{\gamma ^2L}{2n}+\frac{\beta \left( 1+J \right) \gamma ^2L}{n} \right) \frac{d}{n}\sum_{i=1}^n{\sigma _{i}^{2}}.
\end{aligned}
\end{equation}

Substituting the upper bound of $\sum_{k=0}^{K-1}{M^k}$ (c.f.~\eqref{eq_3}) into the above inequality, yields that
\begin{equation}
\label{ed_saga}
\begin{aligned}
&\sum_{k=0}^{K-1}{\mathbb{E}\left[ f\left( \bar{x}^{k+1} \right) \right]}+\beta L\cdot \sum_{k=0}^{K-1}{\mathbb{E}\left[ T^{k+1} \right]}
\\
\leqslant & \sum_{k=0}^{K-1}{\mathbb{E}\left[ f\left( \bar{x}^k \right) \right]}-\frac{\gamma}{2}\sum_{k=0}^{K-1}{\mathbb{E}\left[ \left\| \nabla f\left( \bar{x}^k \right) \right\| ^2 \right]}
\\
&-\frac{\gamma -\gamma ^2L-2\left( 1+J \right) \beta \gamma ^2L}{2}\cdot \sum_{k=0}^{K-1}{\mathbb{E}\left[ \left\| \frac{1}{n}\sum_{i=1}^n{\nabla f_i\left( z_{i}^{k} \right)} \right\| ^2 \right]}+A_4\sum_{k=0}^{K-1}{\mathbb{E}\left[ T^k \right]}
\\
&+A_5\left[ \frac{6C^2}{\left( 1-q^2 \right) n}\sum_{i=1}^n{\left\| x_{i}^{0} \right\| ^2}+\frac{6\gamma ^2C^2}{\left( 1-q \right) ^2}\left( 3b^2+\frac{d}{n}\sum_{i=1}^n{\sigma _{i}^{2}} \right) K+\frac{18\gamma ^2C^2}{\left( 1-q \right) ^2}\sum_{k=0}^{K-1}{\mathbb{E}\left[ \left\| \nabla f\left( \bar{x}^k \right) \right\| ^2 \right]} \right] 
\\
&+K\left( \frac{\gamma ^2L}{2n}+\frac{\beta \left( 1+J \right) \gamma ^2L}{n} \right) \frac{d}{n}\sum_{i=1}^n{\sigma _{i}^{2}},
\end{aligned}
\end{equation}
where
\begin{equation}
\label{def_of_A_3}
\begin{aligned}
A_4=&\frac{\gamma ^2L^3}{n}+\beta \left[ \left( 1-\frac{1}{J^2} \right) L+\frac{2\left( 1+J \right) \gamma ^2L^3}{n} \right] 
\\
&+\left[ \frac{\gamma L^2}{2}+\frac{\gamma ^2L^3}{n}+\beta \left[ \frac{2\left( 1+J \right) \gamma ^2L^3}{n}+\left( 1+\frac{1}{J} \right) \frac{L}{J} \right] \right] \cdot \frac{12 \gamma ^2L^2C^2}{\left( 1-q \right) ^2}
,
\end{aligned}
\end{equation}
and
\begin{equation}
\label{def_of_A_4}
A_5=\frac{\gamma L^2}{2}+\frac{\gamma ^2L^3}{n}+\beta \left[ \frac{2\left( 1+J \right) \gamma ^2L^3}{n}+\left( 1+\frac{1}{J} \right) \frac{L}{J} \right] .
\end{equation}

If the constant step size $\gamma$ satisfies
\begin{equation}
\label{lr_condition_2}
\gamma \leqslant \frac{1}{2L\sqrt{\frac{J^2\left( 1+J \right)}{n}+\frac{12\left( 1+J \right) C^2}{\left( 1-q \right) ^2}}},
\end{equation}
by setting
\begin{equation}
\label{value_of_beta}
\beta =\frac{\frac{\gamma ^2L^2}{n}+\left( \frac{\gamma L}{2}+\frac{\gamma ^2L^2}{n} \right) \cdot \frac{12\gamma ^2L^2C^2}{\left( 1-q \right) ^2}}{\frac{1}{J^2}-\frac{2\left( 1+J \right) \gamma ^2L^2}{n}-\left[ \frac{2\left( 1+J \right) \gamma ^2L^2}{n}+\left( 1+\frac{1}{J} \right) \frac{1}{J} \right] \cdot \frac{12\gamma ^2L^2C^2}{\left( 1-q \right) ^2}}
\leqslant \frac{2J^2\gamma ^2L^2}{n}+\left( \gamma L+\frac{2\gamma ^2L^2}{n} \right) \cdot \frac{12J^2\gamma ^2L^2C^2}{\left( 1-q \right) ^2},
\end{equation}
we have 
\begin{equation}
\label{A_3_eq}
A_4=\beta L.
\end{equation}

Further, if $\gamma$ satisfies
\begin{equation}
\label{lr_condition_3}
\gamma \leqslant \min \left\{ \frac{n}{2L},\frac{\left( 1-q \right) ^2}{12 nLC^2} \right\} ,
\end{equation}
then, we have
\begin{equation}
\label{beta_upper_bound}
\beta \leqslant \frac{4J^2\gamma ^2L^2}{n},
\end{equation}
and in turn we obtain
\begin{equation}
\begin{aligned}
A_5& \overset{\eqref{def_of_A_4}}{=} \frac{\gamma L^2}{2}+\frac{\gamma ^2L^3}{n}+\beta \left[ \frac{2\left( 1+J \right) \gamma ^2L^3}{n}+\left( 1+\frac{1}{J} \right) \frac{L}{J} \right] 
\\
&\overset{\eqref{beta_upper_bound}}{\leqslant} \frac{\gamma L^2}{2}+\frac{\gamma ^2L^3}{n}+\frac{8\left( 1+J \right) J^2\gamma ^4L^5}{n^2}+\frac{4\left( 1+J \right) \gamma ^2L^3}{n}.
\end{aligned}
\end{equation}

If $\gamma$ further satisfies
\begin{equation}
\label{lr_condition_4}
\gamma \leqslant \min \left\{ \frac{\sqrt{n}}{\sqrt{2}JL},\frac{n}{4\left( 4J+5 \right) L} \right\} ,
\end{equation}
then, we have
\begin{equation}
\label{A_4_upper_bound}
A_5\leqslant \gamma L^2.
\end{equation}

Now, substituting~\eqref{A_3_eq}, \eqref{beta_upper_bound} and~\eqref{A_4_upper_bound} into~\eqref{ed_saga} yields that
\begin{equation}
\label{before_rearranging_term}
\begin{aligned}
&\sum_{k=0}^{K-1}{\mathbb{E}\left[ f\left( \bar{x}^{k+1} \right) \right]}+\beta L\sum_{k=0}^{K-1}{\mathbb{E}\left[ T^{k+1} \right]}
\\
\leqslant & \sum_{k=0}^{K-1}{\mathbb{E}\left[ f\left( \bar{x}^k \right) \right]}-\frac{\gamma}{2}\sum_{k=0}^{K-1}{\mathbb{E}\left[ \left\| \nabla f\left( \bar{x}^k \right) \right\| ^2 \right]}+\beta L\sum_{k=0}^{K-1}{\mathbb{E}\left[ T^k \right]}
\\
&-\frac{\gamma -\gamma ^2L-2\left( 1+J \right) \cdot \frac{4J^2\gamma ^2L^2}{n}\cdot \gamma ^2L}{2}\cdot \sum_{k=0}^{K-1}{\mathbb{E}\left[ \left\| \frac{1}{n}\sum_{i=1}^n{\nabla f_i\left( z_{i}^{k} \right)} \right\| ^2 \right]}
\\
&+\gamma L^2\left[ \frac{6C^2}{\left( 1-q^2 \right) n}\sum_{i=1}^n{\left\| x_{i}^{0} \right\| ^2}+\frac{6\gamma ^2C^2}{\left( 1-q \right) ^2}\left( 3b^2+\frac{d}{n}\sum_{i=1}^n{\sigma _{i}^{2}} \right) K+\frac{18\gamma ^2C^2}{\left( 1-q \right) ^2}\sum_{k=0}^{K-1}{\mathbb{E}\left[ \left\| \nabla f\left( \bar{x}^k \right) \right\| ^2 \right]} \right] 
\\
&+K\left( \frac{\gamma ^2L}{2n}+\frac{4J^2\gamma ^2L^2}{n}\cdot \frac{\left( 1+J \right) \gamma ^2L}{n} \right) \frac{d}{n}\sum_{i=1}^n{\sigma _{i}^{2}}.
\end{aligned}
\end{equation}

By rearranging terms of~\eqref{before_rearranging_term}, we have
\begin{equation}
\label{ed2}
\begin{aligned}
&\left( \frac{\gamma}{2}-\frac{18\gamma ^3L^2C^2}{\left( 1-q \right) ^2} \right) \sum_{k=0}^{K-1}{\mathbb{E}\left[ \left\| \nabla f\left( \bar{x}^k \right) \right\| ^2 \right]}
+\left( \frac{\gamma}{2}-\frac{\gamma ^2L}{2}-\frac{4\left( 1+J \right) J^2\gamma ^4L^3}{n} \right) \sum_{k=0}^{K-1}{\mathbb{E}\left[ \left\| \frac{1}{n}\sum_{i=1}^n{\nabla f_i\left( z_{i}^{k} \right)} \right\| ^2 \right]}
\\
\leqslant & f\left( \bar{x}^0 \right) -f^*+\beta L\cdot T^0+\frac{6\gamma L^2C^2}{\left( 1-q^2 \right) n}\sum_{i=1}^n{\left\| x_{i}^{0} \right\| ^2}
+K\left( \frac{\gamma ^2L}{2n}+\frac{4\left( 1+J \right) J^2\gamma ^4L^3}{n^2} \right) \frac{d}{n}\sum_{i=1}^n{\sigma _{i}^{2}}
\\
&+\frac{6\gamma ^3L^2C^2}{\left( 1-q \right) ^2}\left( 3b^2+\frac{d}{n}\sum_{i=1}^n{\sigma _{i}^{2}} \right) K.
\end{aligned}
\end{equation}

Recalling that in Algorithm~\ref{PrivSGP-VR} we initialize 
\begin{equation}
\label{initialization}
x_{i}^{0}=z_{i}^{0}=x^0 \in \mathbb{R}^d
\end{equation}
for all $i \in \{1,2,...,n\}$, thus we have
\begin{equation}
T^0=0
\end{equation}
according to the definition of $T^k$ in~\eqref{def_of_T_k}.

Therefore, \eqref{ed2} becomes
\begin{equation}
\label{before_dividing_gamma_K}
\begin{aligned}
&\left( \frac{\gamma}{2}-\frac{18\gamma ^3L^2C^2}{\left( 1-q \right) ^2} \right) \sum_{k=0}^{K-1}{\mathbb{E}\left[ \left\| \nabla f\left( \bar{x}^k \right) \right\| ^2 \right]}
+\left( \frac{\gamma}{2}-\frac{\gamma ^2L}{2}-\frac{4\left( 1+J \right) J^2\gamma ^4L^3}{n} \right) \sum_{k=0}^{K-1}{\mathbb{E}\left[ \left\| \frac{1}{n}\sum_{i=1}^n{\nabla f_i\left( z_{i}^{k} \right)} \right\| ^2 \right]}
\\
\leqslant & f\left( \bar{x}^0 \right) -f^*+\frac{6\gamma L^2C^2}{\left( 1-q^2 \right)}\left\| x^0 \right\| ^2+K\left( \frac{\gamma ^2L}{2n}+\frac{4\left( 1+J \right) J^2\gamma ^4L^3}{n^2} \right) \frac{d}{n}\sum_{i=1}^n{\sigma _{i}^{2}}
+\frac{6\gamma ^3L^2C^2}{\left( 1-q \right) ^2}\left( 3b^2+\frac{d}{n}\sum_{i=1}^n{\sigma _{i}^{2}} \right) K.
\end{aligned}
\end{equation}

Dividing by $\gamma K$ on both sides of~\eqref{before_dividing_gamma_K}, we have
\begin{equation}
\label{ed3}
\begin{aligned}
&\left( \frac{1}{2}-\frac{18\gamma ^2L^2C^2}{\left( 1-q \right) ^2} \right) \cdot \frac{1}{K}\sum_{k=0}^{K-1}{\mathbb{E}\left[ \left\| \nabla f\left( \bar{x}^k \right) \right\| ^2 \right]}
\\
&+\left( \frac{1}{2}-\frac{\gamma L}{2}-\frac{4\left( 1+J \right) J^2\gamma ^3L^3}{n} \right) \cdot \frac{1}{K}\sum_{k=0}^{K-1}{\mathbb{E}\left[ \left\| \frac{1}{n}\sum_{i=1}^n{\nabla f_i\left( z_{i}^{k} \right)} \right\| ^2 \right]}
\\
\leqslant & \frac{f\left( \bar{x}^0 \right) -f^*}{\gamma K}+\frac{6L^2C^2}{\left( 1-q^2 \right) K}\left\| x^0 \right\| ^2+\left( \frac{\gamma L}{2n}+\frac{4\left( 1+J \right) J^2\gamma ^3L^3}{n^2} \right) \frac{d}{n}\sum_{i=1}^n{\sigma _{i}^{2}}
+\frac{6\gamma ^2L^2C^2}{\left( 1-q \right) ^2}\left( 3b^2+\frac{d}{n}\sum_{i=1}^n{\sigma _{i}^{2}} \right) .
\end{aligned}
\end{equation}

We notice that if
\begin{equation}
\label{lr_condition_5}
\gamma \leqslant \min \left\{ \frac{1}{3L},\frac{\sqrt{n}}{2JL\sqrt{1+J}} \right\} ,
\end{equation}
then we have
\begin{equation}
\label{upper_bound_1}
\frac{1}{2}-\frac{\gamma L}{2}-\frac{4\left( 1+J \right) J^2\gamma ^3L^3}{n}\geqslant 0.
\end{equation}
and, if 
\begin{equation}
\label{lr_condition_6}
\gamma \leqslant \frac{1-q}{6\sqrt{2}LC},
\end{equation}
then we have
\begin{equation}
\label{upper_bound_2}
\frac{1}{2}-\frac{18\gamma ^2L^2C^2}{\left( 1-q \right) ^2}\geqslant \frac{1}{4}.
\end{equation}

Thus, using~\eqref{upper_bound_1} and~\eqref{upper_bound_2}, \eqref{ed3} becomes
\begin{equation}
\begin{aligned}
\frac{1}{4}\cdot \frac{1}{K}\sum_{k=0}^{K-1}{\mathbb{E}\left[ \left\| \nabla f\left( \bar{x}^k \right) \right\| ^2 \right]}
\leqslant & \frac{f\left( \bar{x}^0 \right) -f^*}{\gamma K}+\frac{6L^2C^2}{\left( 1-q^2 \right) K}\left\| x^0 \right\| ^2+\left( \frac{\gamma L}{2n}+\frac{4\left( 1+J \right) J^2\gamma ^3L^3}{n^2} \right) \frac{d}{n}\sum_{i=1}^n{\sigma _{i}^{2}}
\\
&+\frac{6\gamma ^2L^2C^2}{\left( 1-q \right) ^2}\left( 3b^2+\frac{d}{n}\sum_{i=1}^n{\sigma _{i}^{2}} \right),
\end{aligned}
\end{equation}
or, equivalently,
\begin{equation}
\label{ed4_copy}
\begin{aligned}
\frac{1}{K}\sum_{k=0}^{K-1}{\mathbb{E}\left[ \left\| \nabla f\left( \bar{x}^k \right) \right\| ^2 \right]}
\leqslant & \frac{4\left( f\left( \bar{x}^0 \right) -f^* \right)}{\gamma K}+\frac{24L^2C^2}{\left( 1-q^2 \right) K}\left\| x^0 \right\| ^2+4\left( \frac{\gamma L}{2n}+\frac{4\left( 1+J \right) J^2\gamma ^3L^3}{n^2} \right) \frac{d}{n}\sum_{i=1}^n{\sigma _{i}^{2}}
\\
&+\frac{24\gamma ^2L^2C^2}{\left( 1-q \right) ^2}\left( 3b^2+\frac{d}{n}\sum_{i=1}^n{\sigma _{i}^{2}} \right),
\end{aligned}
\end{equation}
where $\gamma$ need to satisfy~\eqref{lr_condition_1}, \eqref{lr_condition_2}, \eqref{lr_condition_3}, \eqref{lr_condition_4}, \eqref{lr_condition_5} and~\eqref{lr_condition_6}, i.e.,
\begin{equation}
\label{total_lr_condition_1}
\gamma \leqslant \min \left\{ \frac{1}{2L\sqrt{\frac{J^2\left( 1+J \right)}{n}+\frac{12\left( 1+J \right) C^2}{\left( 1-q \right) ^2}}},\frac{\left( 1-q \right) ^2}{12nLC^2},\frac{n}{4\left( 4J+5 \right) L},\frac{1}{3L} \right\} ,
\end{equation}
and with~\eqref{total_lr_condition_1}, \eqref{ed4_copy} can be further relaxed as
\begin{equation}
\label{ed4}
\begin{aligned}
\frac{1}{K}\sum_{k=0}^{K-1}{\mathbb{E}\left[ \left\| \nabla f\left( \bar{x}^k \right) \right\| ^2 \right]}\leqslant & \frac{4\left( f\left( \bar{x}^0 \right) -f^* \right)}{\gamma K}+\frac{6\gamma Ld}{n}\cdot \frac{1}{n}\sum_{i=1}^n{\sigma _{i}^{2}}
+\frac{24\gamma ^2L^2C^2}{\left( 1-q \right) ^2}\left( 3b^2+\frac{d}{n}\sum_{i=1}^n{\sigma _{i}^{2}} \right) +\frac{24L^2C^2}{\left( 1-q \right) ^2K}\left\| x^0 \right\| ^2
.
\end{aligned}
\end{equation}

Now, let $\gamma =\sqrt{\frac{n}{K}}$. To ensure~\eqref{total_lr_condition_1}, we thus have
\begin{equation}
\sqrt{\frac{n}{K}} \leqslant \min \left\{ \frac{1}{2L\sqrt{\frac{J^2\left( 1+J \right)}{n}+\frac{12\left( 1+J \right) C^2}{\left( 1-q \right) ^2}}},\frac{\left( 1-q \right) ^2}{12nLC^2},\frac{n}{4\left( 4J+5 \right) L},\frac{1}{3L} \right\}, 
\end{equation}
which implies that
\begin{equation}
\label{total_K_1}
K\geqslant \max \left\{ 4nL^2\left( \frac{J^2\left( 1+J \right)}{n}+\frac{12\left( 1+J \right) C^2}{\left( 1-q \right) ^2} \right) ,\frac{144n^3L^2C^4}{\left( 1-q \right) ^4},\frac{16\left( 4J+5 \right) ^2L^2}{n},9nL^2 \right\},
\end{equation}
and substituting $\gamma=\sqrt{\frac{n}{K}}$ into~\eqref{ed4} yields that
\begin{equation}
\label{upper_bound_running_sum_nabla_f_k}
\begin{aligned}
& \frac{1}{K}\sum_{k=0}^{K-1}{\mathbb{E}\left[ \left\| \nabla f\left( \bar{x}^k \right) \right\| ^2 \right]}
\\
\leqslant &\frac{4\left( f\left( \bar{x}^0 \right) -f^* \right)}{\sqrt{nK}}+\frac{6Ld}{\sqrt{nK}}\cdot \frac{1}{n}\sum_{i=1}^n{\sigma _{i}^{2}}
+\frac{24nL^2C^2}{\left( 1-q \right) ^2K}\left( 3b^2+\frac{d}{n}\sum_{i=1}^n{\sigma _{i}^{2}} \right) +\frac{24L^2C^2}{\left( 1-q \right) ^2K}\left\| x^0 \right\| ^2
\\
\overset{\eqref{total_K_1}}{\leqslant}&\frac{4\left( f\left( \bar{x}^0 \right) -f^* \right)}{\sqrt{nK}}+\frac{6Ld}{\sqrt{nK}}\cdot \frac{1}{n}\sum_{i=1}^n{\sigma _{i}^{2}}
+\frac{2L}{\sqrt{nK}}\left( 3b^2+\frac{d}{n}\sum_{i=1}^n{\sigma _{i}^{2}} \right) +\frac{2L}{\sqrt{nK}}\left\| x^0 \right\| ^2
\\
=&\frac{4\left( f\left( \bar{x}^0 \right) -f^* \right) +2L\left\| x^0 \right\| ^2+6Lb^2+8L\cdot \frac{d}{n}\sum_{i=1}^n{\sigma _{i}^{2}}}{\sqrt{nK}}.
\end{aligned}
\end{equation}
\\
\\
\textbf{Step 2:} \quad \textbf{Upper bounding} $\frac{1}{K}\sum_{k=0}^{K-1}{\frac{1}{n}\sum_{i=1}^n{\mathbb{E}\left[ \left\| z_{i}^{k}-\bar{x}^k \right\| ^2 \right]}}$
\\
\\
According to~\eqref{after_eq_2} and $T^0=0$, we can easily derive that
\begin{equation}
\label{eq_66}
\begin{aligned}
\left( \frac{1}{J^2}-\frac{2\left( 1+J \right) \gamma ^2L^2}{n} \right) \sum_{k=0}^{K-1}{\mathbb{E}\left[ T^k \right]}
\leqslant& \left[ \frac{2\left( 1+J \right) \gamma ^2L^2}{n}+\left( 1+\frac{1}{J} \right) \frac{1}{J} \right] \cdot \sum_{k=0}^{K-1}{M^k}
\\
&+\left( 1+J \right) \gamma ^2\cdot \sum_{k=0}^{K-1}{\mathbb{E}\left[ \left\| \frac{1}{n}\sum_{i=1}^n{\nabla f_i\left( z_{i}^{k} \right)} \right\| ^2 \right]}+K\left( 1+J \right) \gamma ^2\frac{d}{n^2}\sum_{i=1}^n{\sigma _{i}^{2}}.
\end{aligned}
\end{equation}

In addition, according to~\eqref{total_lr_condition_1}, we can derive that
\begin{equation}
\label{eq_67}
\frac{1}{J^2}-\frac{2\left( 1+J \right) \gamma ^2L^2}{n}\geqslant \frac{1}{2J^2}\,\,   ,
\end{equation}
and
\begin{equation}
\label{eq_68}
\frac{2\left( 1+J \right) \gamma ^2L^2}{n}+\left( 1+\frac{1}{J} \right) \frac{1}{J}\leqslant \frac{2\left( 1+J \right)}{J^2}.
\end{equation}

Substituting~\eqref{eq_67} and~\eqref{eq_68} into~\eqref{eq_66}, we have
\begin{equation}
\label{its_eq}
\begin{aligned}
\sum_{k=0}^{K-1}{\mathbb{E}\left[ T^k \right]}\leqslant & 4\left( 1+J \right) \cdot \sum_{k=0}^{K-1}{M^k}+2J^2\left( 1+J \right) \gamma ^2\cdot \sum_{k=0}^{K-1}{\mathbb{E}\left[ \left\| \frac{1}{n}\sum_{i=1}^n{\nabla f_i\left( z_{i}^{k} \right)} \right\| ^2 \right]}
\\
&+K\cdot 2J^2\left( 1+J \right) \gamma ^2\frac{d}{n^2}\sum_{i=1}^n{\sigma _{i}^{2}}.
\end{aligned}
\end{equation}

According to~\eqref{after_eq_1}, we can easily derive that
\begin{equation}
\label{sum_of_nabla}
\begin{aligned}
&\frac{\gamma}{2}\sum_{k=0}^{K-1}{\mathbb{E}\left[ \left\| \nabla f\left( \bar{x}^k \right) \right\| ^2 \right]}
\\
\leqslant & f\left( \bar{x}^0 \right) -f^*+\left( \frac{\gamma L^2}{2}+\frac{\gamma ^2L^3}{n} \right) \sum_{k=0}^{K-1}{M^k}+\frac{\gamma ^2L^3}{n}\sum_{k=0}^{K-1}{\mathbb{E}\left[ T^k \right]}-\frac{\gamma -\gamma ^2L}{2}\sum_{k=0}^{K-1}{\mathbb{E}\left[ \left\| \frac{1}{n}\sum_{i=1}^n{\nabla f_i\left( z_{i}^{k} \right)} \right\| ^2 \right]}+K\frac{\gamma ^2L}{2n}\cdot \frac{d}{n}\sum_{i=1}^n{\sigma _{i}^{2}}
\\
\overset{\eqref{total_lr_condition_1}}{\leqslant} & f\left( \bar{x}^0 \right) -f^*+\gamma L^2\sum_{k=0}^{K-1}{M^k}+\frac{\gamma ^2L^3}{n}\sum_{k=0}^{K-1}{\mathbb{E}\left[ T^k \right]}
-\frac{\gamma -\gamma ^2L}{2}\sum_{k=0}^{K-1}{\mathbb{E}\left[ \left\| \frac{1}{n}\sum_{i=1}^n{\nabla f_i\left( z_{i}^{k} \right)} \right\| ^2 \right]}+K\frac{\gamma ^2L}{2n}\cdot \frac{d}{n}\sum_{i=1}^n{\sigma _{i}^{2}}.
\end{aligned}
\end{equation}

Substituting~\eqref{sum_of_nabla} into~\eqref{eq_3}, we get
\begin{equation}
\label{my_eq}
\begin{aligned}
&\sum_{k=0}^{K-1}{M^k}
\\
\leqslant & \frac{6C^2}{\left( 1-q \right) ^2n}\sum_{i=1}^n{\left\| x_{i}^{0} \right\| ^2}+\frac{6\gamma ^2C^2}{\left( 1-q \right) ^2}\left( 3b^2+\frac{d}{n}\sum_{i=1}^n{\sigma _{i}^{2}} \right) K
+\frac{18\gamma ^3LC^2}{\left( 1-q \right) ^2n}\cdot \frac{d}{n}\sum_{i=1}^n{\sigma _{i}^{2}}\cdot K+\frac{36\gamma C^2\left( f\left( \bar{x}^0 \right) -f^* \right)}{\left( 1-q \right) ^2}
\\
&+\frac{36\gamma ^2L^2C^2}{\left( 1-q \right) ^2}\sum_{k=0}^{K-1}{M^k}+\frac{36\gamma ^3L^3C^2}{\left( 1-q \right) ^2n}\sum_{k=0}^{K-1}{\mathbb{E}\left[ T^k \right]}+\frac{12\gamma ^2L^2C^2}{\left( 1-q \right) ^2}\sum_{k=0}^{K-1}{\mathbb{E}\left[ T^k \right]}
\\
&-\frac{36\gamma C^2}{\left( 1-q \right) ^2}\cdot \frac{\gamma -\gamma ^2L}{2}\sum_{k=0}^{K-1}{\mathbb{E}\left[ \left\| \frac{1}{n}\sum_{i=1}^n{\nabla f_i\left( z_{i}^{k} \right)} \right\| ^2 \right]}
\\
\overset{\eqref{initialization}\eqref{total_lr_condition_1}}{\leqslant} & \frac{6C^2}{\left( 1-q \right) ^2}\left\| x^0 \right\| ^2+\frac{6\gamma ^2C^2}{\left( 1-q \right) ^2}\left( 3b^2+\frac{2d}{n}\sum_{i=1}^n{\sigma _{i}^{2}} \right) K
+\frac{36\gamma C^2\left( f\left( \bar{x}^0 \right) -f^* \right)}{\left( 1-q \right) ^2}+\frac{36\gamma ^2L^2C^2}{\left( 1-q \right) ^2}\sum_{k=0}^{K-1}{M^k}
\\
&+\frac{24\gamma ^2L^2C^2}{\left( 1-q \right) ^2}\sum_{k=0}^{K-1}{\mathbb{E}\left[ T^k \right]}-\frac{36\gamma C^2}{\left( 1-q \right) ^2}\cdot \frac{\gamma -\gamma ^2L}{2}\sum_{k=0}^{K-1}{\mathbb{E}\left[ \left\| \frac{1}{n}\sum_{i=1}^n{\nabla f_i\left( z_{i}^{k} \right)} \right\| ^2 \right]}.
\end{aligned}
\end{equation}

Substituting~\eqref{its_eq} into~\eqref{my_eq}, yields that
\begin{equation}
\label{last_eq}
\begin{aligned}
&\sum_{k=0}^{K-1}{M^k}
\\
\leqslant & \frac{6C^2}{\left( 1-q \right) ^2}\left\| x^0 \right\| ^2+\frac{6\gamma ^2C^2}{\left( 1-q \right) ^2}\left( 3b^2+\frac{2d}{n}\sum_{i=1}^n{\sigma _{i}^{2}} \right) K
+\frac{36\gamma C^2\left( f\left( \bar{x}^0 \right) -f^* \right)}{\left( 1-q \right) ^2}+\frac{36\gamma ^2L^2C^2}{\left( 1-q \right) ^2}\sum_{k=0}^{K-1}{M^k}
\\
&+\frac{24\gamma ^2L^2C^2}{\left( 1-q \right) ^2}\cdot 4\left( 1+J \right) \cdot \sum_{k=0}^{K-1}{M^k}
+\frac{24\gamma ^2L^2C^2}{\left( 1-q \right) ^2}\cdot K\cdot 2J^2\left( 1+J \right) \gamma ^2\frac{d}{n^2}\sum_{i=1}^n{\sigma _{i}^{2}}
\\
&+\frac{24\gamma ^2L^2C^2}{\left( 1-q \right) ^2}\cdot 2J^2\left( 1+J \right) \gamma ^2\cdot \sum_{k=0}^{K-1}{\mathbb{E}\left[ \left\| \frac{1}{n}\sum_{i=1}^n{\nabla f_i\left( z_{i}^{k} \right)} \right\| ^2 \right]}
\\
&-\frac{36\gamma C^2}{\left( 1-q \right) ^2}\cdot \frac{\gamma -\gamma ^2L}{2}\sum_{k=0}^{K-1}{\mathbb{E}\left[ \left\| \frac{1}{n}\sum_{i=1}^n{\nabla f_i\left( z_{i}^{k} \right)} \right\| ^2 \right]}
\\
\overset{\eqref{total_lr_condition_1}}{\leqslant}&\frac{6C^2}{\left( 1-q \right) ^2}\left\| x^0 \right\| ^2+\frac{6\gamma ^2C^2}{\left( 1-q \right) ^2}\left( 3b^2+\frac{4d}{n}\sum_{i=1}^n{\sigma _{i}^{2}} \right) K
+\frac{36\gamma C^2\left( f\left( \bar{x}^0 \right) -f^* \right)}{\left( 1-q \right) ^2}+\frac{\left( 96J+132 \right) \gamma ^2L^2C^2}{\left( 1-q \right) ^2}\sum_{k=0}^{K-1}{M^k}
\\
&-\frac{6\gamma ^2C^2}{\left( 1-q \right) ^2}\left[ 3-3\gamma L-8J^2\left( 1+J \right) \gamma ^2L^2 \right] \cdot \sum_{k=0}^{K-1}{\mathbb{E}\left[ \left\| \frac{1}{n}\sum_{i=1}^n{\nabla f_i\left( z_{i}^{k} \right)} \right\| ^2 \right]}.
\end{aligned}
\end{equation}

In addition to satisfying~\eqref{total_lr_condition_1}, if $\gamma$ further satisfies 
\begin{equation}
\label{total_lr_condition_2}
\gamma \leqslant \min \left\{ \frac{3}{8J^2\left( 1+J \right) L},\frac{1-q}{2LC\sqrt{48J+66}} \right\}  ,
\end{equation}
we have
\begin{equation}
\label{coef_1}
3-3\gamma L-8J^2\left( 1+J \right) \gamma ^2L^2\geqslant 0
\end{equation}
and
\begin{equation}
\label{coef_2}
\frac{\left( 96J+132 \right) \gamma ^2L^2C^2}{\left( 1-q \right) ^2}\leqslant \frac{1}{2}.
\end{equation}

Substituting~\eqref{coef_1} and~\eqref{coef_2} into~\eqref{last_eq}, we further have
\begin{equation}
\label{futher_inequal}
\sum_{k=0}^{K-1}{M^k}\leqslant \frac{12C^2}{\left( 1-q \right) ^2}\left\| x^0 \right\| ^2+\frac{12\gamma ^2C^2}{\left( 1-q \right) ^2}\left( 3b^2+\frac{4d}{n}\sum_{i=1}^n{\sigma _{i}^{2}} \right) K+\frac{72\gamma C^2\left( f\left( \bar{x}^0 \right) -f^* \right)}{\left( 1-q \right) ^2}.
\end{equation}

By now, the constant step size $\gamma$ need to satisfy~\eqref{total_lr_condition_1} and~\eqref{total_lr_condition_2}, i.e.,
\begin{equation}
\label{total_lr_condition_3}
\gamma \leqslant \min \left\{ \frac{1}{2L\sqrt{\frac{J^2\left( 1+J \right)}{n}+\frac{12\left( 1+J \right) C^2}{\left( 1-q \right) ^2}}},\frac{\left( 1-q \right) ^2}{12nLC^2},\frac{n}{4\left( 4J+5 \right) L},\frac{1-q}{2LC\sqrt{48J+66}},\frac{3}{8J^2\left( 1+J \right) L} \right\} .
\end{equation}

Now, let $\gamma =\sqrt{\frac{n}{K}}$. To ensure~\eqref{total_lr_condition_3}, we thus have
\begin{equation}
\sqrt{\frac{n}{K}} \leqslant \min \left\{ \frac{1}{2L\sqrt{\frac{J^2\left( 1+J \right)}{n}+\frac{12\left( 1+J \right) C^2}{\left( 1-q \right) ^2}}},\frac{\left( 1-q \right) ^2}{12nLC^2},\frac{n}{4\left( 4J+5 \right) L},\frac{1-q}{2LC\sqrt{48J+66}},\frac{3}{8J^2\left( 1+J \right) L} \right\},
\end{equation}
which implies that
\begin{equation}
\label{total_total_iteration_K}
\begin{aligned}
&K\geqslant 
\\
&\underset{\hat{K}(C,q)}{\underbrace{\max \left\{ 4nL^2\left( \frac{J^2\left( 1+J \right)}{n}+\frac{12\left( 1+J \right) C^2}{\left( 1-q \right) ^2} \right) ,\frac{144n^3L^2C^4}{\left( 1-q \right) ^4},\frac{16\left( 4J+5 \right) ^2L^2}{n},\frac{4L^2C^2\left( 48J+66 \right) n}{\left( 1-q \right) ^2},\frac{64J^4\left( 1+J \right) ^2L^2n}{9} \right\} }},
\end{aligned}
\end{equation}
and substituting $\gamma=\sqrt{\frac{n}{K}}$ into~\eqref{futher_inequal} yields that
\begin{equation}
\sum_{k=0}^{K-1}{M^k}\leqslant \frac{12C^2}{\left( 1-q \right) ^2}\left\| x^0 \right\| ^2+\frac{12nC^2}{\left( 1-q \right) ^2}\left( 3b^2+\frac{4d}{n}\sum_{i=1}^n{\sigma _{i}^{2}} \right) +\frac{72\sqrt{n}C^2\left( f\left( \bar{x}^0 \right) -f^* \right)}{\left( 1-q \right) ^2\sqrt{K}},
\end{equation}
i.e.,
\begin{equation}
\label{consensus_error}
\frac{1}{K}\sum_{k=0}^{K-1}{M^k}\leqslant \frac{12C^2}{\left( 1-q \right) ^2K}\left\| x^0 \right\| ^2+\frac{12nC^2}{\left( 1-q \right) ^2K}\left( 3b^2+\frac{4d}{n}\sum_{i=1}^n{\sigma _{i}^{2}} \right)+\frac{72\sqrt{n}C^2\left( f\left( \bar{x}^0 \right) -f^* \right)}{\left( 1-q \right) ^2K^{\frac{3}{2}}}.
\end{equation}
\\
\\
\\
\noindent
\textbf{Step 3:} \quad \textbf{Upper bounding} $\frac{1}{K}\sum_{k=0}^{K-1}{\frac{1}{n}\sum_{i=1}^n{\mathbb{E}\left[ \left\| \nabla f\left( z_{i}^{k} \right) \right\| ^2 \right]}}$
\\
\\
Using the upper bound in~\eqref{upper_bound_running_sum_nabla_f_k} and~\eqref{consensus_error}, we obtain
\begin{equation}
\label{hhhzzzjjj}
\begin{aligned}
&\frac{1}{K}\sum_{k=0}^{K-1}{\frac{1}{n}\sum_{i=1}^n{\mathbb{E}\left[ \left\| \nabla f\left( z_{i}^{k} \right) \right\| ^2 \right]}}
\\
=&\frac{1}{K}\sum_{k=0}^{K-1}{\frac{1}{n}\sum_{i=1}^n{\mathbb{E}\left[ \left\| \nabla f\left( z_{i}^{k} \right) -\nabla f\left( \bar{x}^k \right) +\nabla f\left( \bar{x}^k \right) \right\| ^2 \right]}}
\\
\leqslant & 2\frac{1}{K}\sum_{k=0}^{K-1}{\frac{1}{n}\sum_{i=1}^n{\mathbb{E}\left[ \left\| \nabla f\left( z_{i}^{k} \right) -\nabla f\left( \bar{x}^k \right) \right\| ^2 \right]}}+2\frac{1}{K}\sum_{k=0}^{K-1}{\frac{1}{n}\sum_{i=1}^n{\mathbb{E}\left[ \left\| \nabla f\left( \bar{x}^k \right) \right\| ^2 \right]}}
\\
\overset{\text{Assumption~\ref{assumption_smooth_saga}}}{\leqslant}&2L^2\frac{1}{K}\sum_{k=0}^{K-1}{\frac{1}{n}\sum_{i=1}^n{\mathbb{E}\left[ \left\| z_{i}^{k}-\bar{x}^k \right\| ^2 \right]}}+2\frac{1}{K}\sum_{k=0}^{K-1}{\frac{1}{n}\sum_{i=1}^n{\mathbb{E}\left[ \left\| \nabla f\left( \bar{x}^k \right) \right\| ^2 \right]}}
\\
\overset{\eqref{def_of_M_k_saga}}{=} & 2L^2\frac{1}{K}\sum_{k=0}^{K-1}{M^k}+2\frac{1}{K}\sum_{k=0}^{K-1}{\mathbb{E}\left[ \left\| \nabla f\left( \bar{x}^k \right) \right\| ^2 \right]}
\\
\overset{\eqref{upper_bound_running_sum_nabla_f_k}~\eqref{consensus_error}}{\leqslant} & \frac{24L^2C^2}{\left( 1-q \right) ^2K}\left\| x^0 \right\| ^2+\frac{24nL^2C^2}{\left( 1-q \right) ^2K}\left( 3b^2+\frac{4d}{n}\sum_{i=1}^n{\sigma _{i}^{2}} \right) +\frac{144\sqrt{n}L^2C^2\left( f\left( \bar{x}^0 \right) -f^* \right)}{\left( 1-q \right) ^2K^{\frac{3}{2}}}
\\
&+\frac{8\left( f\left( \bar{x}^0 \right) -f^* \right) +4L\left\| x^0 \right\| ^2+12Lb^2+16L\cdot \frac{d}{n}\sum_{i=1}^n{\sigma _{i}^{2}}}{\sqrt{nK}}
\\
\overset{\eqref{total_total_iteration_K}}{\leqslant}&\frac{13\left( f\left( \bar{x}^0 \right) -f^* \right) +6L\left\| x^0 \right\| ^2+18Lb^2+24L\cdot \frac{d}{n}\sum_{i=1}^n{\sigma _{i}^{2}}}{\sqrt{nK}}
.
\end{aligned}
\end{equation}

According to~\eqref{initialization} and the definition of $F^0$ in Theorem~\ref{Theorem_1_saga}, we have
\begin{equation}
f\left( \bar{x}^0 \right) -f^*=f\left( x^0 \right) -f^*=F^0,
\end{equation}
and thus~\eqref{hhhzzzjjj} becomes
\begin{equation}
\label{final_upper_bound}
\frac{1}{K}\sum_{k=0}^{K-1}{\frac{1}{n}\sum_{i=1}^n{\mathbb{E}\left[ \left\| \nabla f\left( z_{i}^{k} \right) \right\| ^2 \right]}}
\leqslant 
\frac{13 F^0 +6L\left\| x^0 \right\| ^2+18Lb^2+24L\cdot \frac{d}{n}\sum_{i=1}^n{\sigma _{i}^{2}}}{\sqrt{nK}}.
\end{equation}

By now, the proof of Theorem \ref{Theorem_1_saga} has been completed. 


\section{Proof of Theorem~\ref{Theorem_3}}
\label{proof_of_moments_accountant}
In this section, we provide the proof of Theorem~\ref{Theorem_3}.
\begin{proof}
According to the definition of $g_i^k$ in~\eqref{corrected_stochastic_gradient}, we have
\begin{equation}
\label{clipping_bound}
\begin{aligned}
\left\| g_{i}^{k} \right\| & =\left\| \nabla f_i(z_{i}^{k};\xi _{i}^{k})-\nabla f_i(\phi _{i,\xi _{i}^{k}};\xi _{i}^{k})+\frac{1}{J}\sum_{j=1}^J{\nabla}f_i(\phi _{i,j};j) \right\| 
\\
& \leqslant \left\| \nabla f_i(z_{i}^{k};\xi _{i}^{k}) \right\| +\left\| \nabla f_i(\phi _{i,\xi _{i}^{k}};\xi _{i}^{k}) \right\| +\left\| \frac{1}{J}\sum_{j=1}^J{\nabla}f_i(\phi _{i,j};j) \right\| 
\\
& \leqslant \left\| \nabla f_i(z_{i}^{k};\xi _{i}^{k}) \right\| +\left\| \nabla f_i(\phi _{i,\xi _{i}^{k}};\xi _{i}^{k}) \right\| +\frac{1}{J}\sum_{j=1}^J{\left\| \nabla f_i(\phi _{i,j};j) \right\|}
\\
& \leqslant G+G+\frac{1}{J}\sum_{j=1}^J{G}
\\
& =3G.
\end{aligned}
\end{equation}

Also, knowing that the sampling probability is $\frac{1}{J}$ for each node $i$, the proof of Theorem~\ref{Theorem_3} is straightforward by extending the Theorem 1 in~\cite{abadi2016deep}.
\end{proof}

\section{PrivSGP Algorithm}
\label{appendix_alg}
In this section, we supplement the pseudo-code of our generated differentially private algorithm PrivSGP we missed in the main text, and it is presented in Algorithm~\ref{PrivSGP}.

\begin{center}
\begin{minipage}{9cm}
\begin{algorithm}[H]
\caption{PrivSGP}
\label{PrivSGP}
\textbf{Initialization}: $x_{i}^{0}=z_{i}^{0}=x^0\in \mathbb{R}^d$, $w_i^0=1$ and privacy budget $(\epsilon_i,\delta_i)$ for all $i \in \{1,2,...,n\}$, step size $\gamma > 0$, and  total number of iterations $K$.
\begin{algorithmic}[1] 
\FOR{$k=0,1,2,...,K-1$, at node $i$,}
\STATE Randomly samples a local training data $\xi_i^k$ with the sampling probability $\frac{1}{J}$;
\STATE Computes stochastic gradient at $z_i^k$: $\nabla f_i(z_i^k;\xi_i^k)$
\STATE Adds noise $\nabla \tilde{f}_i(z_{i}^{k};\xi _{i}^{k})=\nabla f_i(z_{i}^{k};\xi _{i}^{k})+N_i^k$, where $N_i^k \in 
    \mathbb{R}^d \thicksim \mathcal{N}\left( 0,\sigma_i^2\mathbb{I}_d \right)$;
\STATE Generates intermediate parameter:  $x_i^{k+\frac{1}{2}}=x_i^k-\gamma \nabla \tilde{f}_i(z_{i}^{k};\xi _{i}^{k})$ ;
\STATE Sends $\left( x_i^{k+\frac{1}{2}}, w_i^k \right)$ to out-neighbors;
\STATE Receives $\left( x_j^{k+\frac{1}{2}}, w_j^k \right)$ from in-neighbors;
\STATE Updates $x_i^{k+1}$ by: \quad $x_{i}^{k+1}=\sum_{j=1}^n{P_{i,j}^{k}x_{j}^{k+\frac{1}{2}}}$;
\STATE Updates $w_i^{k+1}$ by: \quad $w_{i}^{k+1}=\sum_{j=1}^n{P_{i,j}^{k}w_{j}^{k}}$;
\STATE Updates $z_i^{k+1}$ by: \quad $z_{i}^{k+1}=x_{i}^{k+1}/w_i^{k+1}$;
\ENDFOR
\end{algorithmic}
\end{algorithm}
\end{minipage}
\end{center}

\section{Missing Definition of Time-varying Directed Exponential Graph}
\label{missing_definition_of_graph}

We supplement the definition of time-varying directed exponential graph~\cite{assran2019stochastic} we missed in the main text.
Specifically, $n$ nodes are ordered sequentially with their rank $0$,$1$...,$n-1$, and each node has out-neighbours that are $2^0$,$2^1$,...,$2^{\lfloor \log _2\left( n-1 \right) \rfloor}$ hops away.
Each node cycles through these out-neighbours, and only transmits messages to one of its out-neighbours at each iteration.
For example, each node sends message to its $2^0$-hop out-neighbour at iteration $k$, and to its $2^1$-hop out-neighbour at iteration $k+1$, and so on. The above procedure will be repeated within the list of out-neighbours. 
Note that each node only sends and receives a single message at each iteration.

\section{Experimental Results for Training 2-layer Neural Network on Mnist Dataset}
\label{extra_experiments}

\begin{figure}[!htpb]
  \centering
  \subfigure[Training loss]{
    \includegraphics[width=0.3\linewidth]{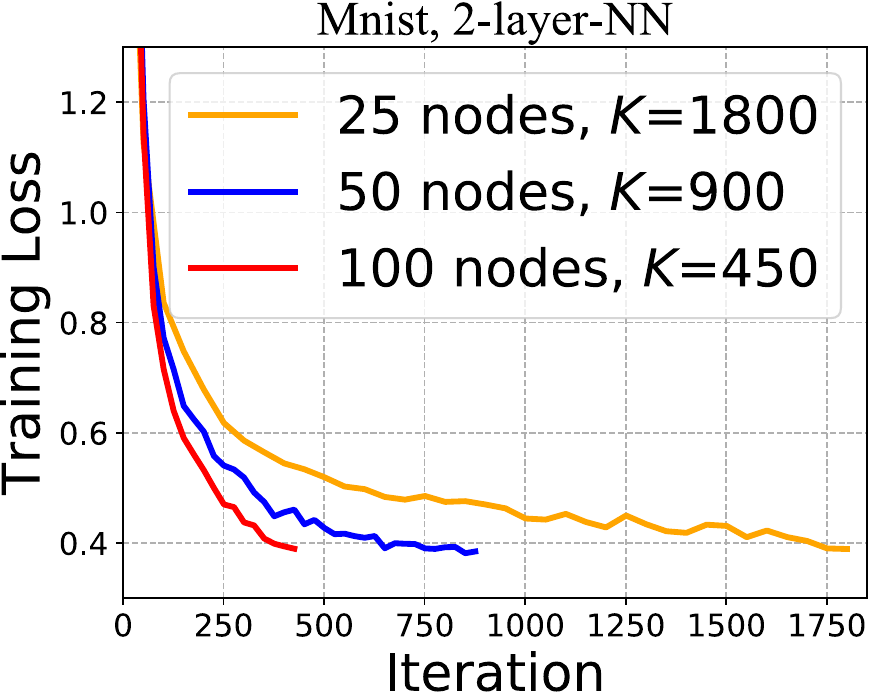}
    \label{Mnist_loss_speed_up}
  }
  \subfigure[Testing accuracy]{
    \includegraphics[width=0.3\linewidth]{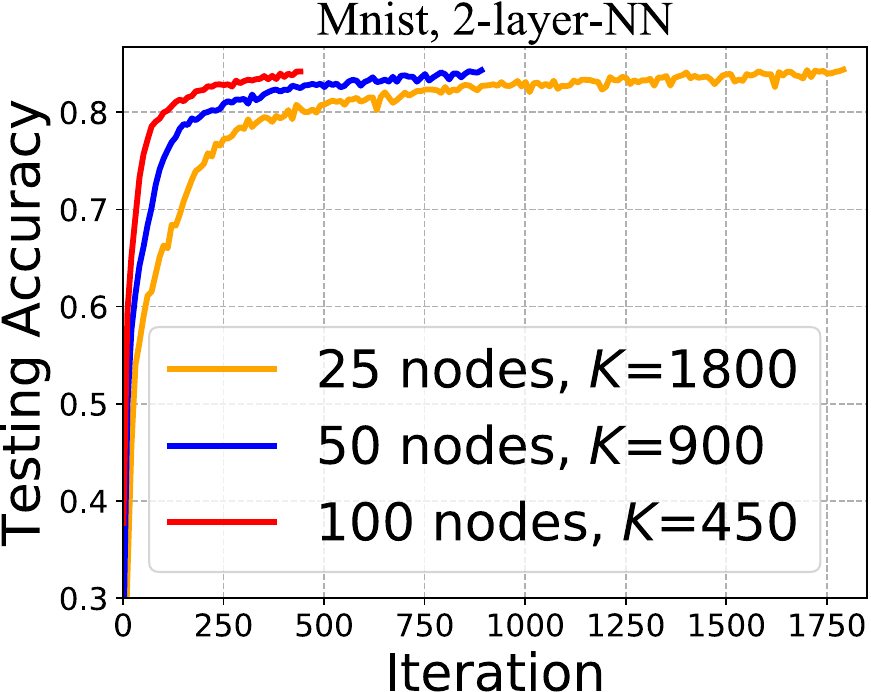}
    \label{Mnist_acc_speed_up}
  }
  \caption{Comparison of convergence performance for PrivSGP-VR over 25, 50 and 100 nodes under the same DP Gaussian noise variance, when training 2-layer neural network on Mnist dataset.}
  \label{Mnist_speed up}
\end{figure}

\paragraph{Linear speedup under constant DP Gaussian noise variance.}
We conduct experiments of training shallow 2-layer neural network on Mnist dataset for our PrivSGP-VR, under 3 distinct network configurations, comprising 25, 50 and 100 nodes, respectively. All configurations utilize the same DP Gaussian noise variance $\sigma_i^2=0.01$ for each node $i$. It can be observed from Figure~\ref{Mnist_speed up} that, by increasing the number of nodes by a factor of 2, we can achieve comparable final training loss and model testing accuracy by running only half the total number of iterations, which also illustrates the linear speedup property in terms of number of nodes $n$ exhibited by our PrivSGP-VR.

\begin{figure}[!htpb]
  \centering
  \subfigure[Training loss]{
    \includegraphics[width=0.3\linewidth]{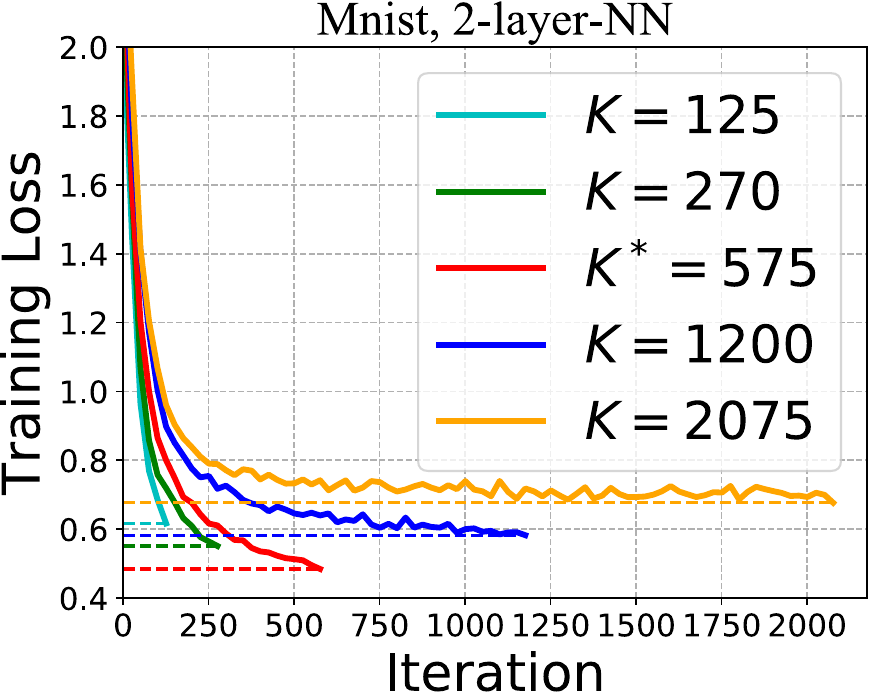}
    \label{Mnist_loss_different_K}
  }
  \subfigure[Testing accuracy]{
    \includegraphics[width=0.3\linewidth]{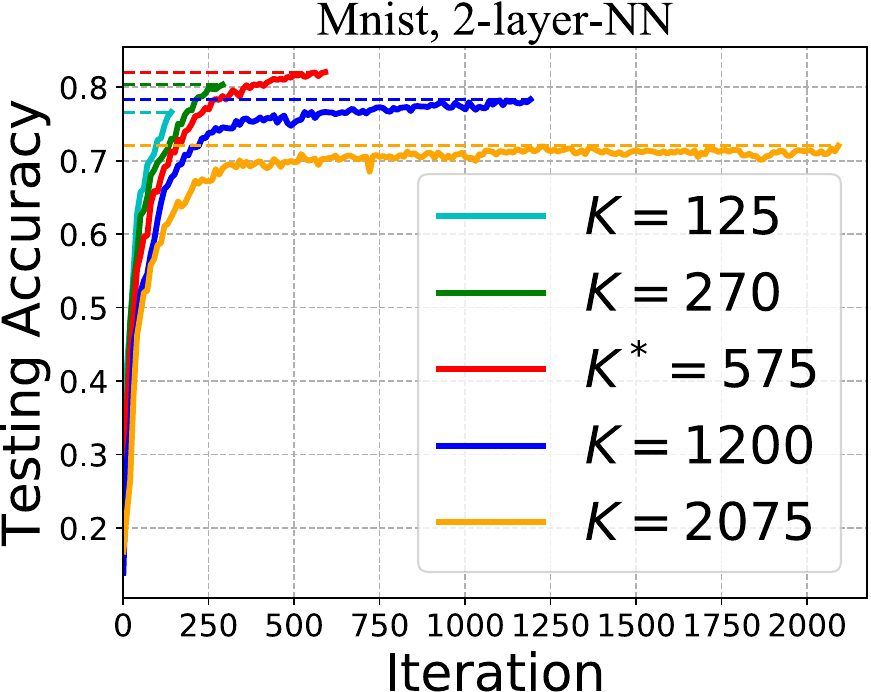}
    \label{Mnist_acc_different_K}
  }
  \caption{Comparison of convergence performance for PrivSGP-VR over 50 nodes by setting different total number of iterations $K$ under a certain privacy budget, when training 2-layer neural network on Mnist dataset.}
  \label{Mnist_different total iteration}
\end{figure}

\paragraph{Optimizing number of iterations under certain privacy budget.}
We conduct experiments of training shallow 2-layer neural network on Mnist dataset for our PrivSGP-VR, under a 50-node setup. We set privacy budget $\epsilon_i=1$ and $\delta_i=10^{-5}$ for each node $i$; and we determine the optimal value of $K$ to be approximately $575$, according to~\eqref{value_of_T_saga}. To investigate the impact of the total number of iterations $K$ on the overall performance of PrivSGP-VR, we consider other values of $K$ for comparison: $125$, $270$, $1200$ and $2075$. For each value of $K$, to guarantee the given privacy budget, we add DP Gaussian noise with variance $\sigma_i^2$ calculated according to~\eqref{noise_scale_saga}. It follows from Figure~\ref{Mnist_different total iteration} that the value of $K$ obtained by our proposed approach leads to the minimum loss and maximum accuracy, which also validates the importance of selecting an appropriate value for $K$ to ensure optimal performance of PrivSGP-VR under a certain privacy budget.

\begin{figure}[!htpb]
\centering
\includegraphics[width=0.45\columnwidth]{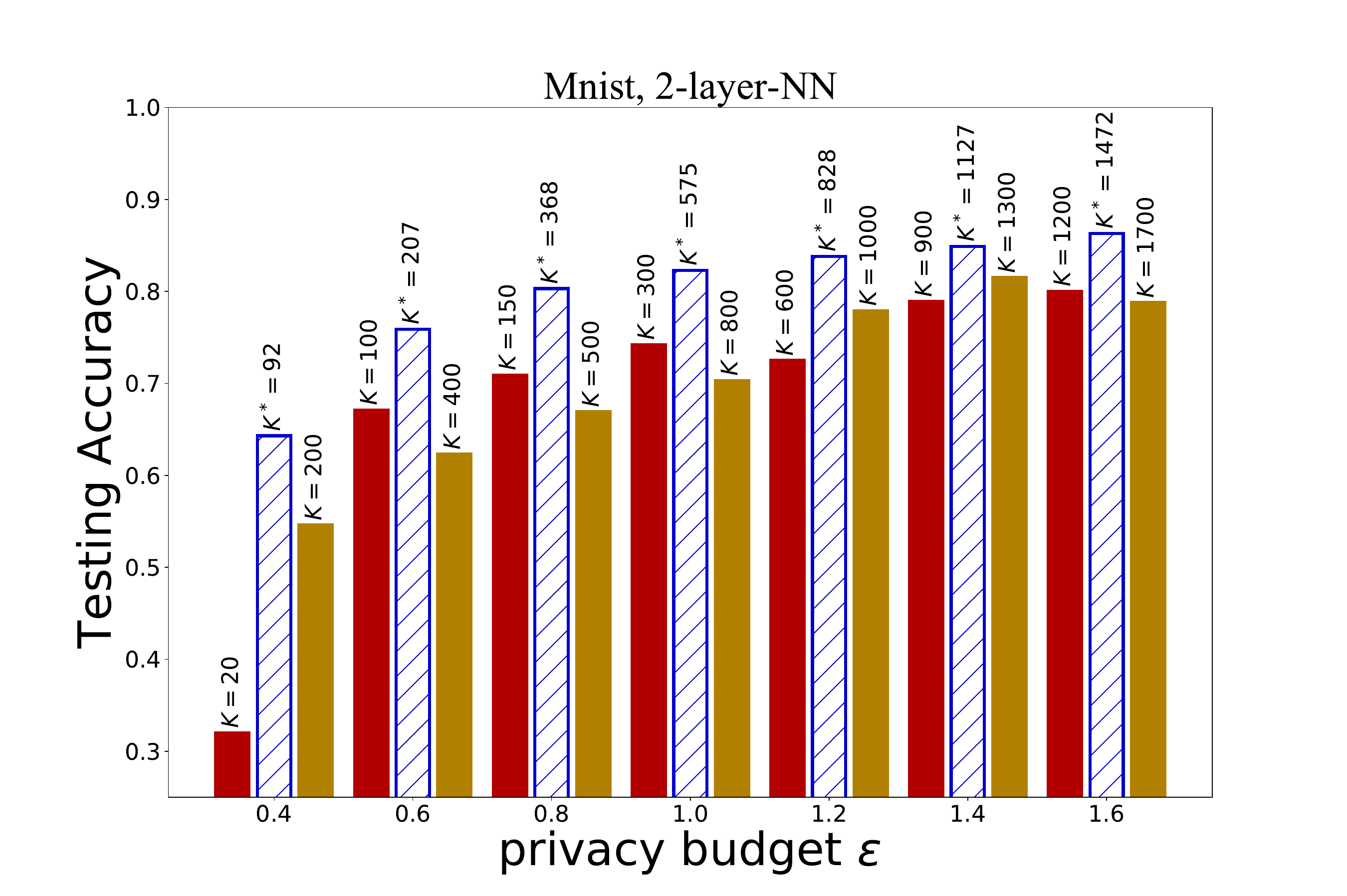} 
\caption{Performance of running PrivSGP-VR for $K^*$ ($K$) iterations under different certain privacy budgets $\epsilon$, when training 2-layer neural network on Mnist dataset.}
\label{Mnist_trade_off_utility_privacy}
\end{figure}

\paragraph{Trade off between the maximized model utility and privacy guarantee.}
We conduct experiments by deploying the PrivSGP-VR algorithm on a network consisting of 50 nodes with a fixed value of $\delta=10^{-5}$ for each node. The $\epsilon$ value for each node is varied from the set $\{0.4, 0.6, 0.8, 1, 1.2, 1.4, 1.6\}$.
For each value of $\epsilon$, we determine the optimal total number of iterations $K^*$ using equation~\eqref{value_of_T_saga}. Then, we execute PrivSGP-VR for $K^*$ iterations, along with two other $K$ values for comparative analysis. We incorporate the corresponding DP Gaussian noise with variance calculated according to equation~\eqref{noise_scale_saga}.
Figure~\ref{trade_off_utility_privacy} illustrates the trade-off between model utility (testing accuracy) and privacy under the optimized number of iterations. As the privacy budget $\epsilon$ diminishes (indicating a higher level of privacy protection), the maximized model utility deteriorates. This trade-off between privacy and maximized utility aligns with the theoretical insights outlined in Remark~\ref{remark_for_trade_off}.

\begin{figure}[!htpb]
  \centering
  \subfigure[Training loss]{
    \includegraphics[width=0.3\linewidth]{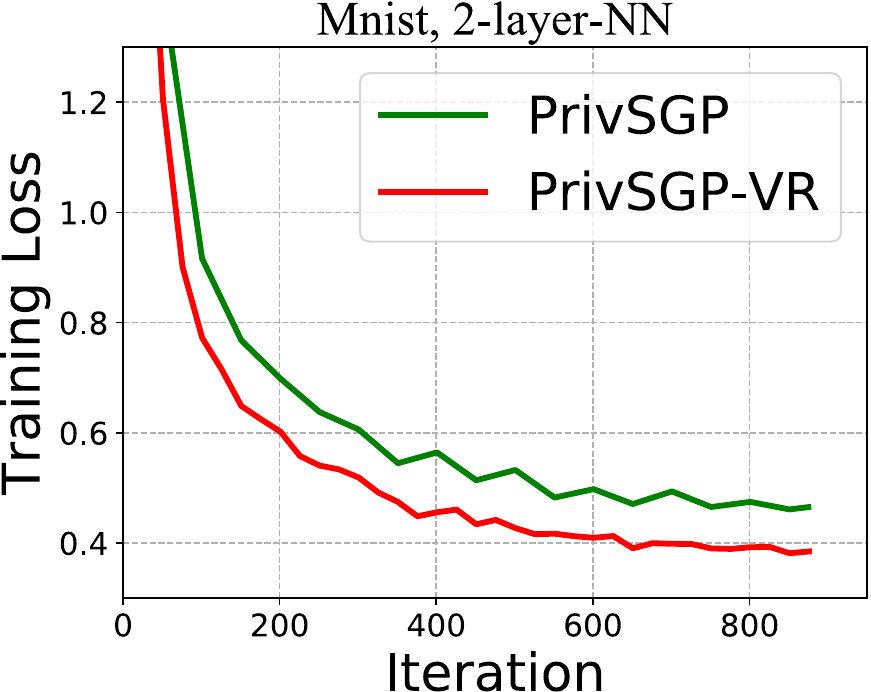}
    \label{Mnist_loss_VR}
  }
  \subfigure[Testing accuracy]{
    \includegraphics[width=0.3\linewidth]{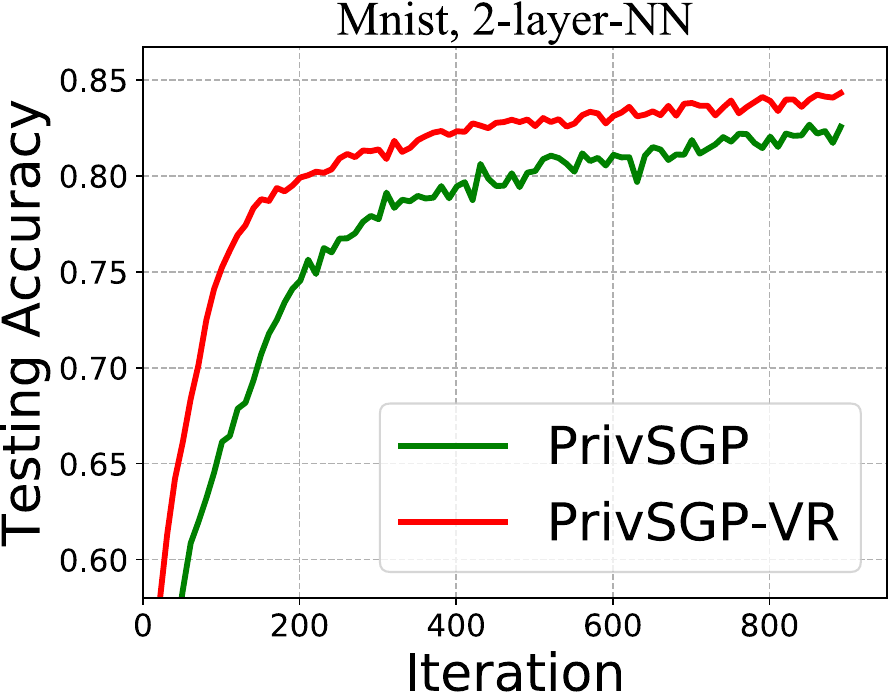}
    \label{Mnist_acc_VR}
  }
  \caption{Comparison of convergence performance for PrivSGP-VR with PrivSGP over 50 nodes under the same DP Gaussian noise variance, when training 2-layer neural network on Mnist dataset.}
  \label{Mnist_variance reduction}
\end{figure}

\paragraph{Verifying the effectiveness of variance reduction technique.}
We conduct experiments of training shallow 2-layer neural network on Mnist dataset to compare PrivSGP-VR with PrivSGP (Algorithm~\ref{PrivSGP}, without the variance reduction technique), under a 50-node setup. We apply DP Gaussian noise with an identical variance of $\sigma_i^2=0.01$ for both PrivSGP-VR and PrivSGP, and we execute both algorithms for same iterations. The results, as shown in Figure~\ref{Mnist_variance reduction}, illustrate that PrivSGP-VR outperforms PrivSGP in terms of both training loss and model testing accuracy, which also validates the effectiveness of the variance reduction technique integrated into PrivSGP-VR.

\begin{figure}[!htpb]
  \centering
  \subfigure[Training loss]{
    \includegraphics[width=0.3\linewidth]{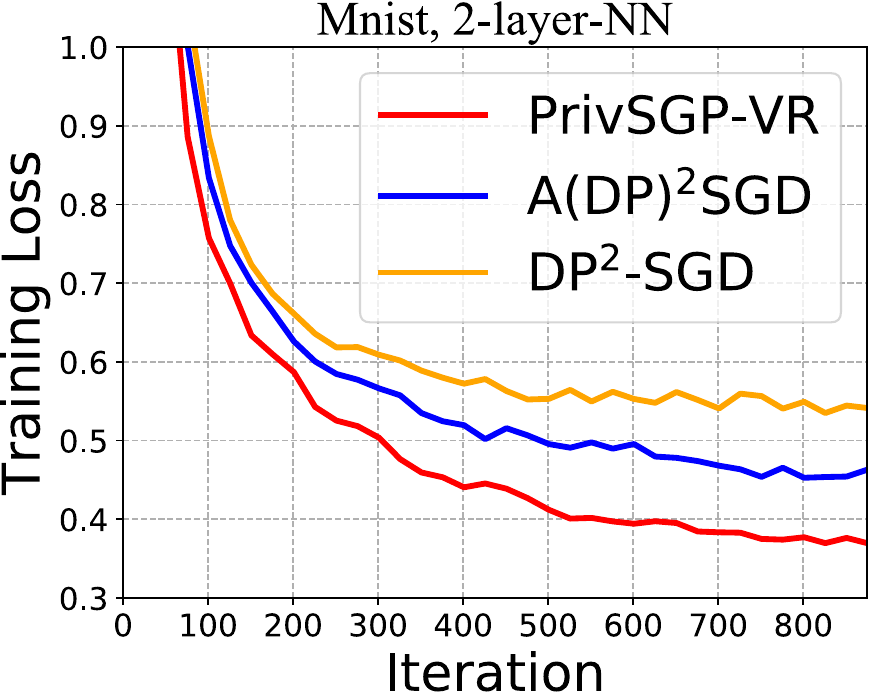}
    \label{Mnist_loss_compare_alg}
  }
  \subfigure[Testing accuracy]{
    \includegraphics[width=0.3\linewidth]{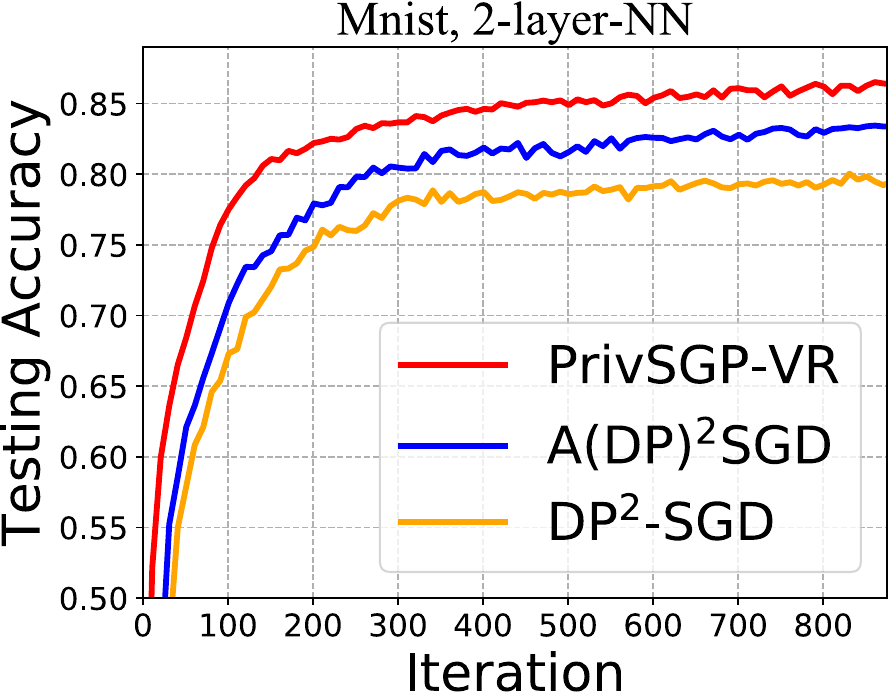}
    \label{Mnist_acc_compare_alg}
  }
  \caption{Comparison of convergence performance for PrivSGP-VR with DP$^2$-SGD and A(DP)$^2$SGD over 50 nodes with $(1,10^{-5})$-DP guarantee for each node, when training 2-layer neural network on Mnist dataset.}
  \label{Mnist_compare_methods}
\end{figure}

\paragraph{Comparison with existing decentralized counterparts.}
We present experiments of training shallow 2-layer neural network on Mnist dataset to compare the performance of PrivSGP-VR with other fully decentralized private stochastic algorithms DP$^2$-SGD and A(DP)$^2$SGD, under an undirected ring graph consisting of 50 nodes. The results shown in Figure~\ref{Mnist_compare_methods} also demonstrate that our PrivSGP-VR outperforms DP$^2$-SGD and A(DP)$^2$SGD under the same differential privacy guarantee.

\end{document}